\documentclass[11pt]{article}

\usepackage[margin=1in]{geometry}

\usepackage[utf8]{inputenc} 
\usepackage[T1]{fontenc}    
\usepackage{hyperref}       
\usepackage{url}            
\usepackage{booktabs}       
\usepackage{amsfonts}       
\usepackage{nicefrac}       
\usepackage{microtype}      
\usepackage{amsmath}
\usepackage{amsthm}
\usepackage{graphicx}
\usepackage{subcaption}
\usepackage{mathtools}
\usepackage{svg}
\usepackage{tabulary}
\usepackage{indentfirst}

\usepackage{mydef}

\title{\texttt{dotears}: Scalable, consistent DAG estimation using observational and interventional data}

\author{
  Albert Xue$^{1}$
  \quad
  Jingyou Rao$^{1}$
  \quad
  Sriram Sankararaman$^{1*}$
  \quad
  Harold Pimentel$^{1*}$\\
  $^1$University of California, Los Angeles\\
  \{asxue, roserao, hjp\}@ucla.edu\\
  sriram@cs.ucla.edu
}
\date{February 15, 2024}

\begin{document}




\maketitle
\def\thefootnote{*}\footnotetext{These authors contributed equally to this work.}
\begin{abstract}
New biological assays like Perturb-seq link highly parallel CRISPR interventions to a high-dimensional transcriptomic readout, providing insight into gene regulatory networks.
Causal gene regulatory networks can be represented by directed acyclic graph (DAGs), but learning DAGs from observational data is complicated by lack of identifiability and a combinatorial solution space. Score-based structure learning improves practical scalability of inferring DAGs. Previous score-based methods are sensitive to error variance structure; on the other hand, estimation of error variance is difficult without prior knowledge of structure.
Accordingly, we present $\texttt{dotears}$ [doo-tairs], a continuous optimization framework which leverages observational and interventional data to infer a single causal structure, assuming a linear Structural Equation Model (SEM).
$\texttt{dotears}$ exploits structural consequences of hard interventions to give a marginal estimate of exogenous error structure, bypassing the circular estimation problem.
We show that $\texttt{dotears}$ is a provably consistent estimator of the true DAG under mild assumptions.
\dotears{} outperforms other methods in varied simulations, and in real data infers edges that validate with higher precision and recall than state-of-the-art methods through differential expression tests and high-confidence protein-protein interactions.
\end{abstract}

\section{Introduction}
Understanding gene regulatory networks can identify mechanisms and pathways linking GWAS significant variants to phenotype. 
Recent efforts to map regulatory networks through \textit{trans}-eQTLs are partly limited by power;
for example, the GTEx project finds only 143 \textit{trans}-eQTLs in 838 individuals \cite{gtex2020gtex}. On the other hand, V{\~o}sa \etal{} detect almost 60,000 \textit{trans}-eQTLs in $\sim$31,000 individuals, across more than a third of trait-associated variants \cite{vosa2021large}. These results imply that \textit{trans}-regulatory relationships are pervasive, but our ability to detect small eQTL effects is often limited by small sample size regimes, observational data, and a high multiple testing burden. 
Importantly, since rare tissues are unlikely to be sampled at sufficient sample sizes, capturing gene regulation events across a wide array of cell types and tissues requires another experimental method.

High-throughput genomic technologies like Perturb-seq provide a natural alternative for learning gene regulatory networks. Perturb-seq links high-dimensional transcriptomic readouts to known, highly parallel CRISPR interventions, allows direct interrogation of causal regulatory relationships, and has scaled genome-wide \cite{dixit2016perturb, norman2019exploring, replogle2022mapping}. In particular, the effects of CRISPR gene interventions are large in comparison to QTL effects, faciliating inference of downstream regulatory relationships. Notably, analogous experiments have already mapped gene-gene networks in yeast \cite{boone2007exploring, costanzo2010genetic, costanzo2016global}. 

The inference of gene regulatory networks can be treated as a causal structure learning problem, which considers learning relationships between variables in the form of a Directed Acyclic Graph (DAG). Identifiability and scalability are the primary difficulties in learning DAGs from data. For identifiability, distinct DAGs may contain the same conditional independence relationships in \textit{observational} data, and DAGs are only identifiable up to Markov equivalence \cite{verma1990equivalence, hauser2012characterization, squires2022causal}. 
For scalability, DAGs with NO TEARS introduces a continuous, differentiable acyclicity constraint, which avoids combinatorial characterizations of DAGs and allows for continuous optimization over structure \cite{zheng2018dags}. However, NO TEARS and related methods infer DAGs whose topological order follows increasing marginal variance, and re-scaling data can change or reverse their inferences \cite{reisach2021beware}. 

Fundamentally, this is still an issue of identifiability. Because NO TEARS uses observational data, it must choose a single member of a class of Markov equivalent DAGs. The ``tiebreaker'' is then a function of the variance. However, we show that interventional data can correct for variance sensitivity in NO TEARS and related methods, and further that this correction is sufficient for consistent estimation of structure.

Explicitly, exogenous error in the linear SEM drives variance sensitivity. Let $X$ be a $p$-dimensional random vector (e.g. the distribution of gene expression across $p$ genes), $W_0\in \RR^{p\times p}$ the weighted adjacency matrix of the true DAG, and $\epsilon$ a $p$-dimensional random vector specifying the exogenous error. The linear SEM gives the autoregressive formulation
\begin{equation}
    X = XW_0 + \epsilon \nonumber,
\end{equation}
where $\Omega_0 \coloneqq \CV \left(\epsilon\right) = \text{diag}\left(\sigma_1^2, \dots, \sigma_p^2\right)$. For any node $j$ we then have 
\begin{equation}
    \VV\left(X_j\right)= \sum_{i = 1}^p w_{ij}^2 \VV\left(X_i\right) + \sigma_j^2.
    \label{eq:variance introduction}
\end{equation}
$\VV\left(X_i\right)$ is itself a linear function of $\sigma_1^2, \dots, \sigma_p^2$. Critically, parental error variances thus propagate to downstream nodes to provide a signal of structure. However, in observational data the exact relationship between $W_0$, $\Omega_0$, and $\VV (X)$ is unidentifiable. 

Consequently, given $\Omega_0$ then $W_0$ is recoverable even in observational data \cite{loh2014high}. However, estimation of $\Omega_0$ is difficult without \textit{a priori} knowledge of $W_0$. Previous methods either ignore exogenous variance structure or use the conditional estimate $\hat{\Omega}_0 \mid W$ \cite{zheng2018dags, ng2020role}. We show that both procedures are sensitive to exogenous variance structure even in the simplest two node DAG. 

\textit{Hard} interventions \cite{eberhardt2006interventions, pearl2009causality} remove upstream variance in the linear SEM to allow marginal estimation of $\Omega_0$ (Figure \ref{fig:main}). We show that naive incorporation of interventional data into the NO TEARS framework, without estimation of $\Omega_0$, is insufficient for structural recovery. Finally, we show correction by this marginal estimate is sufficient for structural recovery.

\begin{figure*}
    \centering
    \includegraphics[width=0.8\textwidth]{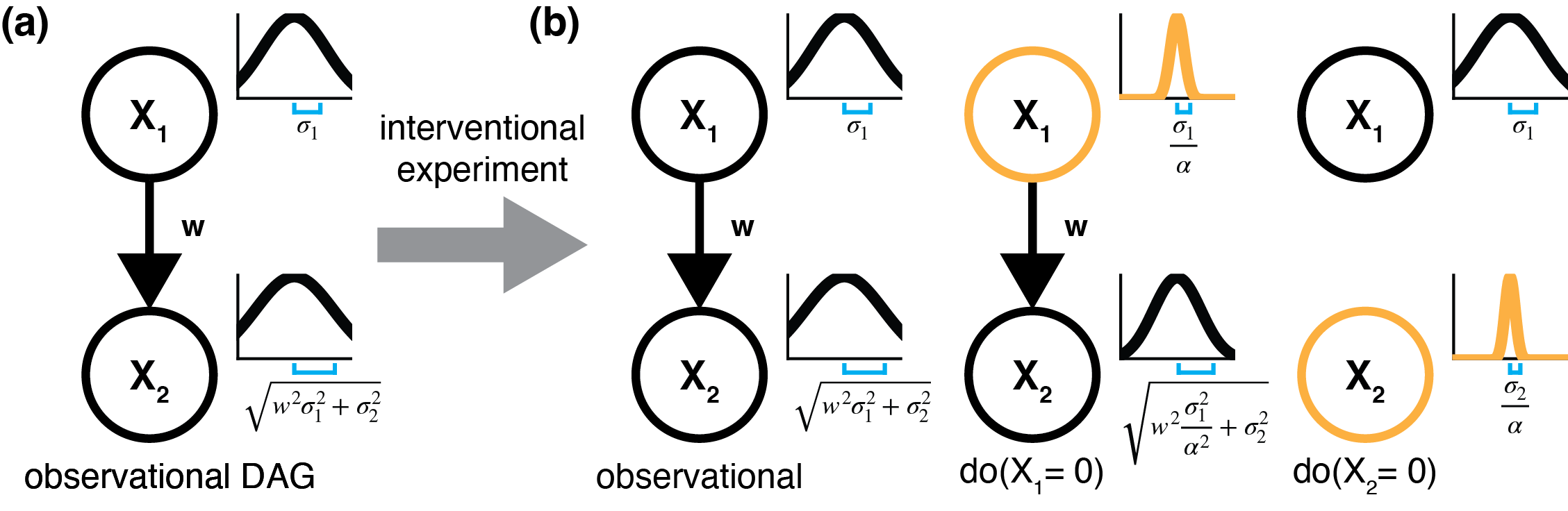}
    \caption{Hard interventions allow marginal estimation of the error variances $\Omega_0 \coloneqq \text{diag}\left(\sigma_1^2, \sigma_2^2\right)$. (a) The true observational DAG $X_1 \overset{w}{\to} X_2$, with corresponding distributions. (b) Hard interventions remove incoming edges to the target. The marginal variance shrinks due to removal of upstream variance and effects from the intervention.} 
    \label{fig:main}
\end{figure*}

Accordingly, we present \dotears{}, a novel optimization framework for structure learning. \dotears{} uses \textbf{1.)} a novel marginal estimation procedure for $\Omega_0$ using the structural consequences of interventions, and \textbf{2.)} joint estimation of the causal DAG from observational and interventional data, given the estimated $\hat{\Omega}_0$. \dotears{} provides a simple model that we show, by extending results from \cite{loh2014high}, is a provably consistent estimator of the true DAG under mild assumptions. In simulations \dotears{} corrects for exogenous variance structure and is robust to reasonable violations of its modeling assumptions. We also apply \dotears{} to the Perturb-seq experiment in Replogle \etal{} \cite{replogle2022mapping}. \dotears{} infers a sparse set of edges that validate with high precision in differential expression tests and in an orthogonal set of protein-protein interactions. In both simulations and real data, \dotears{} outperforms all other tested methods in all used metrics.

\section{Model}
\label{section:model}
\subsection{Formulation}
\label{subsection:formulation}
Let $G = \left([p],\, \mathcal{E}\right)$ a DAG on $p$ nodes with node set $[p] \coloneqq \{ 1\dots p\}$ and edge set $\mathcal{E}$. We represent $\mathcal{E}$ with the weighted adjacency matrix $W \in \mathcal{D} \subset \RR^{p\times p}$, where $\mathcal{D} \subset \RR^{p \times p}$ is the set of weighted adjacency matrices on $p$ nodes whose support is a DAG. We denote the parent set of node $i$ in the observational setting as $\Pa(i)$. For $w_{ij}$ the $i,j$ entry in $W$, $|w_{ij}| > 0$ indicates an edge $i\rightarrow j$ with weight $w_{ij}$, equivalently denoted $i\overset{w_{ij}}{\rightarrow} j$. $k=0,1,\dots, p$ indexes the intervention, where $k=0$ is reserved for the observational system, and $k \neq 0$ denotes intervention on node $k$. Similarly, we denote $\interv{\left(\cdot\right)}{0}$ for observational quantities, and $\interv{\left(\cdot\right)}{k}$ for quantities under intervention on node $k$. For brevity, a variable without a superscript is assumed to be observational; for example, $X \equiv \interv{X}{0}$. $\boldX$ (bolded) denotes $n_0$ samples drawn from the $p$-dimensional random vector $X$ (unbolded), and $\boldepsilon$ (bolded) denotes $n_0$ samples drawn from the $p$-dimensional random vector $\epsilon$ (unbolded). Similarly, if $\interv{X}{k}$ is a $p$-dimensional random vector, then $\interv{\boldX}{k}\in \RR^{n_k\times p}$ represents $n_k$ observations of $X$. We denote the total sample size $n \coloneqq \sum_{k=0}^p n_k$, and the true weighted adjacency matrix $W_0$.

The linear SEM is an autoregressive representation of $\interv{X}{k}$ and weighted adjacency matrix $\interv{W_0}{k}$,
\begin{equation}
    \interv{X}{k} = \interv{X}{k}\interv{W_0}{k} + \interv{\epsilon}{k}, \quad k=0,\dots, p.
    \label{eq:SEM}
\end{equation}
Here, $\interv{W_0}{k}$ is permutation-similar to a strictly upper triangular matrix, representing the constraint $\interv{W_0}{k} \in \mathcal{D}$. For each $k$, $\interv{\epsilon}{k}$ is a $p$-dimensional random vector such that $\EE\interv{\epsilon}{k} = 0_p$, and $\interv{\Omega_0}{k} \coloneqq \CV(\interv{\epsilon}{k})$. Denote $\epsilon_i$ as the $i$th element of $\epsilon$. Then $\interv{\epsilon_i}{k}$ is the exogenous error on node $i$, such that $\EE\interv{\epsilon_i}{k} = 0$ and $\interv{\epsilon_i}{k} \independent \interv{\epsilon_j}{k}$ for $i\neq j$. We further define
 \begin{equation}
    \begin{aligned}
     \Omega_0 &\coloneqq \CV(\interv{\epsilon}{0}) \\
     &= \text{diag}\left(\sigma_1^2, \sigma_2^2, \cdots, \sigma_p^2 \right). \\
    \end{aligned}
    \nonumber
 \end{equation} 



\subsection{Interventions}
Motivated by recent work on a genome-wide screen that performs known single interventions on all protein-coding genes \cite{replogle2022mapping}, we consider the linear SEM with known single interventions on all $p$ nodes. Accordingly, we obtain a system of $p+1$ structural equations
\begin{equation}
    \begin{aligned}
    \interv{X}{0} &= \interv{X}{0}\interv{W}{0} + \interv{\epsilon}{0}\\
    &\vdots \\
    \interv{X}{p} &= \interv{X}{p}\interv{W}{p} + \interv{\epsilon}{p}.
    \end{aligned}
    \nonumber 
\end{equation}
In this setting we have complicated our problem. Before, with the single data matrix $\boldX$, we inferred a single $W$; now, with the $p + 1$ data matrices $\interv{\boldX}{0}, \interv{\boldX}{1}, \dots, \interv{\boldX}{p}$, it seems that we must infer $p + 1$ adjacency matrices $\interv{W}{0}, \interv{W}{1}, \dots, \interv{W}{p}$.


Our model assumes hard interventions, i.e. that an intervention on node $k$ removes causal influences from observational parents of $k$ \cite{eberhardt2006interventions, pearl2009causality}. Hard interventions relax the do operation $\texttt{do}\left(\interv{X}{k}_k = 0\right)$ to allow for residual noise, modeling the limit of interventional efficacy combined with a noisy readout. 

\begin{assumption}
For an intervention $k \neq 0$, if in the observational setting 
\begin{equation}
    \interv{X_k}{0} = \sum_{i\in \text{Pa}(k)} W_{ik}\interv{X_i}{0} + \interv{\epsilon_k}{0}, \nonumber
\end{equation}
then upon intervention on $k$
\begin{equation}
    \interv{X_k}{k} = \interv{\epsilon_k}{k}. \nonumber
\end{equation}
\label{assumption:hard intervention assumption}
\end{assumption}

Under Assumption \ref{assumption:hard intervention assumption}, we can relate $\interv{W}{0}$ to $\interv{W}{k}$ by setting the $k$th column of $\interv{W}{0}$ to $\Vec{0}_p$, giving
\begin{equation}
     \interv{W}{k}_{ij} \coloneqq \begin{cases} 
        W_{ij} & j \neq k \\
        0 & j = k
        \end{cases}.
\label{eq:w k intervention form}
\end{equation} 
This gives $p+1$ data matrices $\interv{\boldX}{k}$ to jointly infer a single weighted adjacency matrix $W$. 

We now characterize the $p+1$ exogenous variance structures $\interv{\Omega_0}{k}$. We assume that the exogenous variance of non-targeted nodes $j \neq k$ is invariant.
\begin{assumption}
    For an intervention $k \neq 0$, $\VV \left(\interv{\epsilon_j}{k}\right) = \VV \left(\interv{\epsilon_j}{0}\right) = \sigma_j^2$ for $j \neq k$.
    \label{assumption:variance without intervention}
\end{assumption}

We allow interventions to have effects on the error variance of the target, but require the effect to be uniform across targets.

\begin{assumption}
    Let unknown $\alpha \in \RR$. Then $\forall \,\, k$, if  $\VV \left(\interv{\epsilon_k}{0}\right) = \sigma_k^2$ then $\VV\left(\interv{\epsilon_k}{k}\right) = \frac{\sigma_k^2}{\alpha^2}$.
    \label{assumption:global a assumption}
\end{assumption}
Here, $\alpha$ is shared across interventions. As $\alpha \to \infty$, this interventional model is equivalent to $\texttt{do}(X_k = 0)$ \cite{pearl2009causality}. Under these assumptions, the variance of the target is 
\begin{equation}
    \VV\left(\interv{X_k}{k}\right)= \frac{\sigma_k^2}{\alpha^2}.\nonumber
\end{equation}

Let $\widehat{\VV}$ denote the unbiased sample variance. We then obtain the estimator
\begin{equation}
\hat{\Omega}_0 = \text{diag}\left(\widehat{\VV}\left(X_1^{(1)}\right), \dots, \widehat{\VV}\left(X_p^{(p)}\right) \right).
\label{eq:omega hat estimator}
\end{equation}


\section{dotears}
DAGs with NO TEARS \cite{zheng2018dags} transforms the combinatorial constraint $W\in \mathcal{D}$ into the continuous constraint $h(W) = 0$, where $\circ$ denotes the Hadamard product and
\begin{equation}
    h(W) = \TR{\exp\left({W \circ W}\right)} - p \nonumber.
\end{equation}
Define $\|\cdot \|_1$ as the vector $\ell_1$ norm on $\text{vec}\left(W\right)$, i.e. $\|\text{vec}\left(W\right)\|_1$, and $\left|\left|{\cdot}\right|\right|_F$ the Frobenius norm. For some loss function $\mathcal{F}$, the differentiability of $h$ allows for the optimization framework
\begin{equation}
    \begin{aligned}
        \min_W &\,\, \mathcal{F}\left(W, \boldX\right) + \lambda \| W \|_1 \\
        \text{s.t.} & \,\, h(W) = 0.
    \end{aligned}
\end{equation}

We present \dotears{}, a consistent, intervention-aware joint estimation procedure for structure learning. Loh and Buhlmann (2014) showed that the Mahalanobis norm is a consistent estimator of $W_0$, and is uniquely minimized in expectation by $W_0$ given $\Omega_0$, but give no estimation procedure for $\Omega_0$ \cite{loh2014high, kaiser2022unsuitability}. Note the Mahalanobis norm's characterization as inverse-variance-weighted by $\Omega_0$,
\begin{equation}
\frobnorm{\left(\boldX-\boldX W\right)\Omega_0^{-\frac{1}{2}}} = \sum_{i=1}^p \frac{1}{\sigma_i^2} \frobnorm{\left(\boldX - \boldX W\right)_i}. 
\label{eq:mahalanobis norm as reweighting}
\end{equation} 
\dotears{} solves the following optimization problem:
\begin{equation}
\begin{aligned}
\min_W &\,\, \frac{1}{p} \sum_{k=0}^p \frac{1}{2n_k} \frobnorm{\left(\boldX^{(k)} - \boldX^{(k)}W^{(k)}\right)\hat{\Omega}_0^{-\frac{1}{2}}} + \lambda \|W\|_1 \\
\text{s.t. }& h(W)=0,  \\
\end{aligned}
\label{eq:dotears optimization framework}
\end{equation}
where 
\begin{equation}
    \interv{W_{ij}}{k} = \begin{cases}
    W_{ij} & j\neq k\\
    0 & j = k
\end{cases}.\nonumber
\end{equation} 
\dotears{} retains the continuous DAG constraint and $\ell_1$ regularization of $W$ from NO TEARS \cite{zheng2018dags}, but incorporates exogenous variance structure through $\hat{\Omega}_0$ as well as interventional data $(k=1,\dots, p)$. 

\subsection{\texttt{dotears} successfully corrects for exogenous variance structure}
\label{section:small p simulations}
We show that \texttt{dotears} is robust to exogenous variance structure, and motivate the necessity of marginal estimation of $\hat{\Omega}_0$ using the simplest non-trivial DAG $X_1 \overset{w}{\rightarrow} X_2$. Assume $X_1 \overset{w}{\to} X_2$ has true weighted adjacency matrix $W_0 \coloneqq \begin{pmatrix}
    0 & w \\
    0 & 0
\end{pmatrix}$ and SEM
\begin{equation}    
    \begin{aligned}
        X_1 &= \epsilon_1,  && \VV\left(\epsilon_1\right) = \sigma_1^2, \\
        X_2 &= w X_1 + \epsilon_2, && \VV\left(\epsilon_2\right) = \sigma_2^2. 
    \end{aligned}
    \label{eq:two node SEM}
\end{equation}

Let $\gamma \coloneqq \frac{\sigma_1^2}{\sigma_2^2}$, such that $\Omega_0 = \begin{pmatrix}
    \gamma & 0\\
    0 & 1
\end{pmatrix}$ and $W_0 = \begin{pmatrix}
    0 & w\\
    0 & 0
\end{pmatrix}$. The least squares loss used by NO TEARS is minimized in expectation if and only if 
\begin{equation}
    |w| \geq \sqrt{1 - \frac{1}{\gamma}},
    \label{eq:varsortability cutoff}
\end{equation}
which is true if and only if the topological ordering of the DAG follows increasing marginal variance, or equivalently a \textit{varsortable} DAG. For the full proof, see Supplementary Material \ref{suppsection:least squares and varsortability}, or \cite{reisach2021beware, kaiser2022unsuitability}.

In Figure \ref{fig:two node sims}, we examine the performance of four different strategies of correcting for $\Omega_0$ in simulations. NO TEARS (black) uses the least squares loss, which ignores $\Omega_0$, while GOLEM-NV (orange) uses a likelihood loss that estimates $\hat{\Omega}_0 \mid W$ \cite{zheng2018dags, ng2020role}. We also include the scenario when $\hat{\Omega}_0$ is set to $I_p$ in Eq. \ref{eq:dotears optimization framework}. 
Call this NO TEARS interventional (green), the simplest extension of NO TEARS to interventional data. NO TEARS interventional is aware of hard interventions through the structure $\interv{W}{k}$, but ignores  $\Omega_0$. NO TEARS interventional is thus an ablation study on $\hat{\Omega}_0$. 

For each set $(w, \gamma) \in \{0.1, \dots, 1.5\} \times \{1, 2, 4, 10, 100\}$, we draw 25 simulations of observational and interventional data, with sample size $n=(p + 1) * 1000 = 3000$. For observational data, this is $3000$ observations from the observational system; for interventional data, this is $1000$ observations from each system $k=0, 1, 2$. For observational data, we draw Gaussian data under the SEM in Eq. \ref{eq:two node SEM}. For interventional methods, we draw Gaussian data under the system of SEMs in Supplementary Material \ref{suppsection:full two node sims}. To isolate the behavior of the loss, we remove $\ell_1$ regularization. Full simulation details and results are given in Supplementary Material \ref{suppsection:full two node sims}. 

NO TEARS does not correct for $\Omega_0$ and uses only observational data. As a result, in Figure \ref{fig:two node sims} it estimates correctly only on varsortable pairs of $w, \gamma$. Note that the grey dashed line represents the theoretical varsortability cutoff $|w| \leq \sqrt{1 - \frac{1}{\gamma}}$ given in Eq. \ref{eq:varsortability cutoff}.

GOLEM-NV uses the maximum likelihood estimate $\hat{\Omega} \mid W$ under a Gaussian model. Subsequently, joint estimation is performed over $W, \hat{\Omega} \mid W$ using the Gaussian negative log likelihood and profile likelihood for $\Omega_0$, which simplifies to
\begin{equation}
    \frac{1}{2} \sum_{i=1}^p \log\left( \frobnorm{(X-XW)_i} \right).
    \label{eq:golem loss}
\end{equation}

This profile likelihood is insufficient to correct for exogenous variance structure, and only infers varsortable structures in Figure \ref{fig:two node sims}. This evidence qualitatively holds in simulations on three node topologies (Supplementary Material \ref{suppsection:full three node sims}), where the behavior of GOLEM-NV remains deterministic in $w, \gamma$. Joint estimation of $W, \hat{\Omega}_0 \mid W$ is thus still sensitive to exogenous variance structure. 

Through NO TEARS interventional, we also see that interventional data alone, without correction for $\Omega_0$, is insufficient to infer the two node structure. The NO TEARS interventional estimate is also deterministic in $w, \gamma$ in Figure \ref{fig:two node sims}, and behaves almost identically to NO TEARS and GOLEM. Note that since NO TEARS interventional is given interventional data for all nodes, it operates in a fully identifiable setting (where NO TEARS and GOLEM-NV do not).

Thus, neither interventional data nor correction by $\hat{\Omega}_0 \mid W$ are alone sufficient to infer structure. \dotears{} combines the two to give the marginal estimate of $\Omega_0$ in Eq. \ref{eq:omega hat estimator} and a robust estimate of $W$.

We do not imply that other methods using interventional data cannot infer the two node case under interventional data; in fact, many do successfully (see Supplementary Material \ref{suppsection:full two node sims}). Rather, we use \textbf{1.)} an observational procedure that ignores $\Omega_0$, \textbf{2.)} an observational procedure that corrects for an estimated $\hat{\Omega}_0$, \textbf{3.)} an interventional procedure that ignores $\Omega_0$, and \textbf{4.)} \dotears{}, an interventional procedure that corrects for an estimated $\hat{\Omega}_0$, to motivate \dotears{} as most parsimonious model robust to exogenous variance structure under this line of thought.


\begin{figure*}
     \centering
     \begin{subfigure}[t]{0.3\textwidth}
         \centering
         \includegraphics[width=\textwidth]{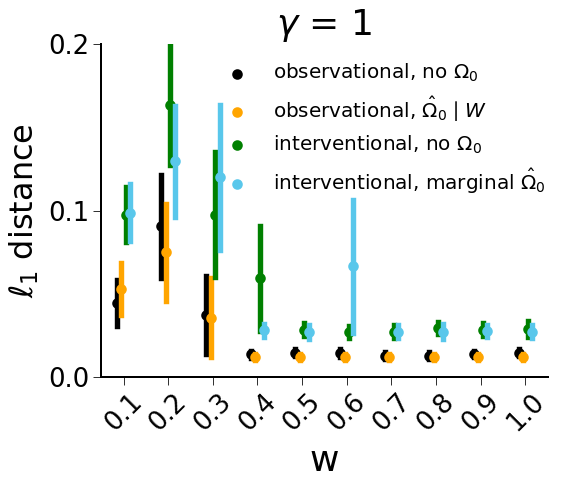}
         \caption{Under the equal variance assumption $\gamma = 1$, correction for $\Omega_0$ is unneeded. All methods are sufficient for structure recovery.}
         \label{fig:gamma1 two node sim}
     \end{subfigure}
     \hfill
     \begin{subfigure}[t]{0.3\textwidth}
         \centering
         \includegraphics[width=\textwidth]{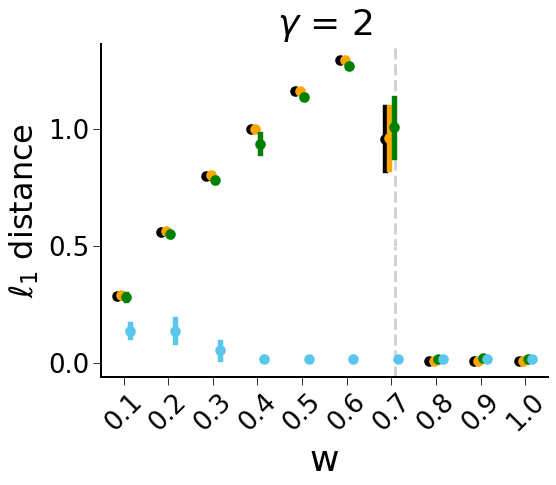}
         \caption{At $\gamma = 2$, NO TEARS (black), GOLEM-NV (orange), and NO TEARS interventional (green) infer correctly on varsortable $w$.}
         \label{fig:gamma2 two node sim}
     \end{subfigure}
     \hfill
     \begin{subfigure}[t]{0.3\textwidth}
         \centering
         \includegraphics[width=\textwidth]{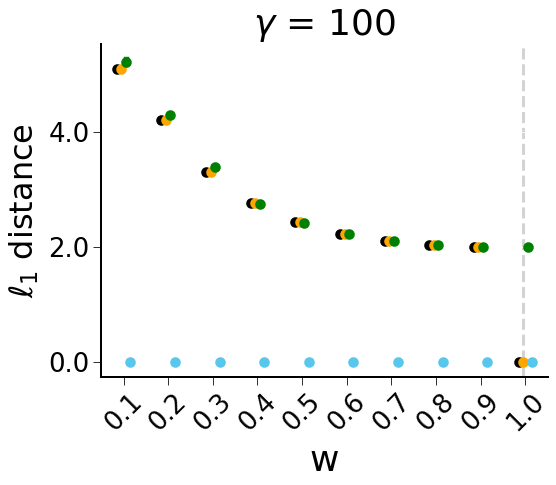}
         \caption{As $\gamma$ grows large, the varsortability bound approaches $|w| \geq 1$. Once again, only \dotears{} estimates correctly for all $w$.}
         \label{fig:gamma100 two node sim}
     \end{subfigure}
    \caption{Comparison of $\ell_1$ distance (lower is better) between true structure and estimates from NO TEARS (black), GOLEM-NV (orange), NO TEARS interventional (green), and \dotears{} (blue). Each method corrects differently for $\Omega_0$. For each $w = 0.1, 0.2, \dots, 1.0$ and $\gamma = 1, 2, 100$, we generate Gaussian data from the structure $X_1 \overset{w}{\to} X_2$ such that $\sigma_1^2 = \gamma \sigma_2^2$. For each pair $w, \gamma$, we draw 25 simulations at a sample size of $n = 3000$. The dashed grey line represents the varsortability bound $|w| \geq \sqrt{1 - \frac{1}{\gamma}}$. Bars represent standard errors; some standard errors are too small to see.}
    \label{fig:two node sims}
\end{figure*}

\subsection{Optimization and consistency}
\label{section:optimization and consistency}
We now wish to show that \dotears{} is a consistent estimator of the true DAG. However, two natural problems arise from our usage of $\hat{\Omega}_0$. First, we have provided no estimation procedure for $\alpha$, but $\EE \hat{\Omega}_0 \neq \Omega_0$ for $\alpha\neq 1$. However, $\EE \hat{\Omega}_0 \propto \Omega_0$ for all $\alpha$, and constant scalings of $\Omega$ are rescalings of the loss \cite{loh2014high}. $\hat{\Omega}_0$ is therefore well-specified for inference on observational data $k=0$.

However, if $\alpha \neq 1$ then $\hat{\Omega}_0$ is still misspecified for interventional data $k \neq 0$. Under Assumption \ref{assumption:global a assumption}
\begin{equation}
    \CV \left(\interv{\epsilon}{k}\right) = \text{diag}\left(\sigma_1^2, \sigma_2^2, \dots, \frac{\sigma_k^2}{\alpha^2}, \dots, \sigma_p^2\right), \nonumber
\end{equation}
and thus $\EE \hat{\Omega}_0 \not \propto \CV\left(\interv{\epsilon}{k}\right)$. A naive approach might estimate $\hat{\Omega}_0$ from interventional data only and then estimate $\hat{W}$ from observational data only. However, this approach ignores a majority of our data and performs substantially worse in simulations (Supplementary Material \ref{section:interventional vs observational only}). We show that $\hat{\Omega}_0$ is well-specified even for $k \neq 0$, and estimation of $\alpha$ is unnecessary (see Corollary \ref{corollary:argmin equivalence under true omega}).

Under a sub-Gaussian assumption, Loh and Buhlmann show consistency of the Mahalanobis norm on observational data given $\Omega_0$ \cite{loh2014high}. We extend these results, using $\hat{\Omega}_0 \overset{p}{\to} \Omega_0$, to show
\begin{equation}
    \argmin_W \frobnorm{\left(\interv{\boldX}{k} - \interv{\boldX}{k} \interv{W}{k}\right) \hat{\Omega}_0^{-\frac{1}{2}}} \nonumber
\end{equation}
is a consistent estimator of $\interv{W_0}{k}$ for each $k$, where $\overset{p}{\to}$ denotes convergence in probability. A full proof is given in Supplementary Material \ref{suppsection:consistency proof}. 

\section{Simulations}
\label{section:simulations}
We evaluate structure learning methods across a range of DAG topologies, effects distributions, and generative models. We benchmark methods that leverage interventional data (\dotears{}, GIES, IGSP, UT-IGSP, DCDI) and methods using only observational data (NO TEARS, sortnregress, GOLEM-EV, GOLEM-NV, DirectLiNGAM) \cite{hauser2012characterization, zheng2018dags, reisach2021beware, ng2020role, yang2018characterizing, wang2017permutation, squires2020permutation, brouillard2020differentiable, shimizu2011directlingam}. \dotears{} outperforms all tested methods in DAG estimation and is robust to reasonable violations of the model.

\begin{figure*}[h]
    \centering
    \includegraphics[width=\textwidth]{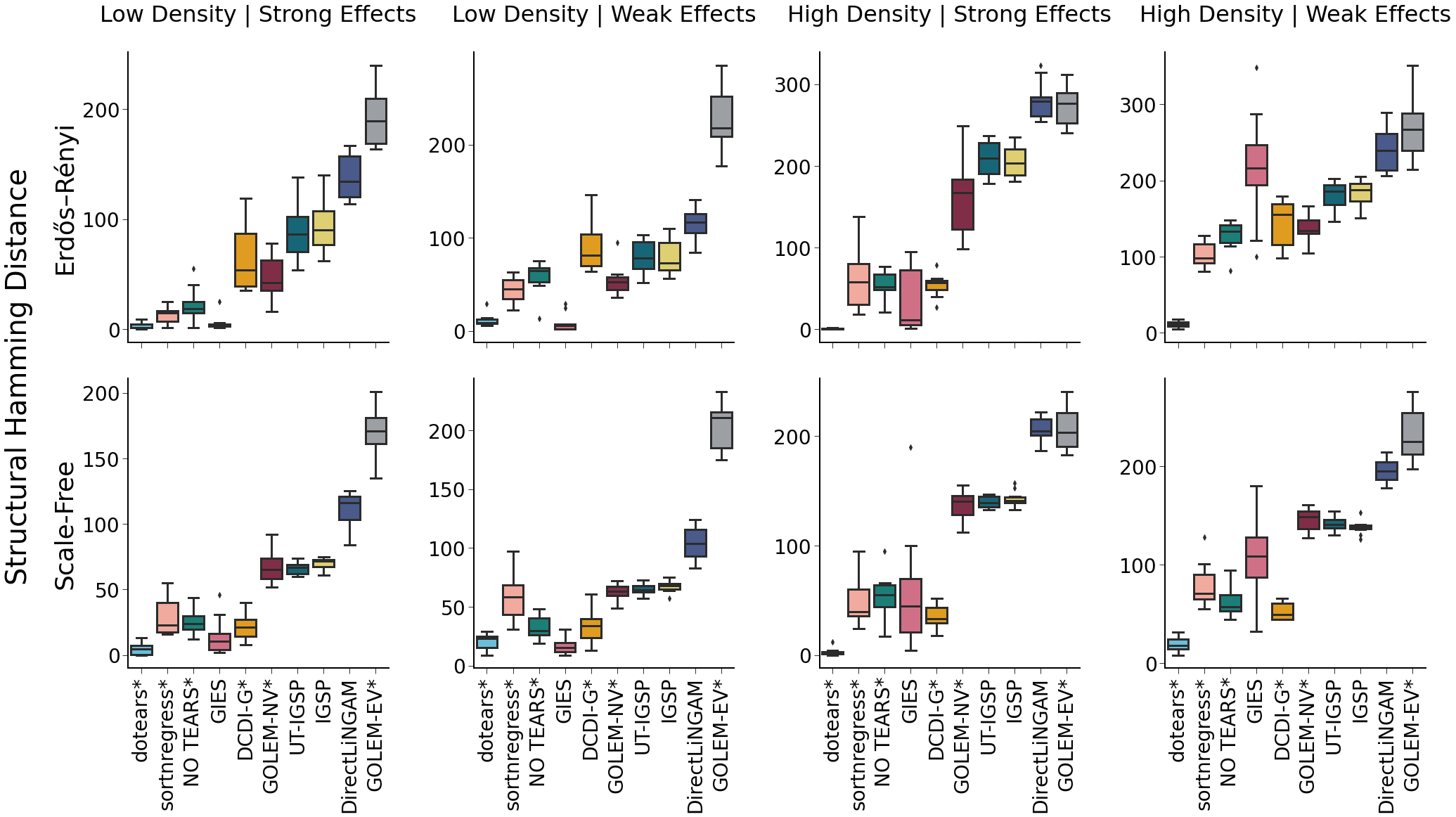}
    \caption{Method performance on large random graphs ($p$ = 40) using Structural Hamming Distance (lower is better). Rows index \erdosrenyi{} or Scale Free topologies. Columns index parameterizations of edge density and weight, ordered in increasing difficulty. For details see Supplementary Material \ref{suppsection:large p data generation}. 10 simulations were drawn for each parameterization with sample size $(p+1) * 100 = 4100$. * indicates cross-validated methods. Methods are sorted by average performance.} 
    \label{fig:large simulations SHD results}
\end{figure*}

Some methods come with important caveats for evaluation. sortnregress is not intended as a ``true'' structure learning method, but benchmarks the data's varsortability \cite{reisach2021beware}. UT-IGSP can infer structure with unknown interventional targets, but we constrain to known targets for fairness \cite{squires2020permutation}. Most simulations use Gaussian data, but \dotears{}, NO TEARS, sortnregress, GIES, IGSP, and UT-IGSP do not assume Gaussianity. The non-Gaussianity assumption is violated for Direct-LiNGAM, and the equal variance assumption for GOLEM-EV \cite{ng2020role, shimizu2011directlingam}. 

\subsection{Large random graph simulations}
\label{section:large random graph simulations}

We simulate synthetic data from large \erdosrenyi{}{} and Scale-Free DAGs \cite{erdHos1960evolution, barabasi1999emergence} ($p = 40$), with 10 replicates each. We simulate under four parameterizations: $\{\text{Low Density, High Density}\} \times \{\text{Weak Effects, Strong Effects}\}$. Observational and interventional data have matched sample size $n = (p + 1) \times 100 = 4100$. Methods using 5-fold cross-validation for hyperparameter tuning are denoted by *. For non-binary methods, edge weights are thresholded. Methods are robust to threshold choice (Supplementary Material \ref{suppsection:thresholding}) \cite{reisach2021beware}. For simulation, cross-validation, and thresholding details, see Supplementary Material \ref{suppsection:large p data generation}; for runtime and memory usage, see Supplementary Material \ref{suppsection:large p simulation benchmarking}. We note that the memory usage of DCDI-G was extreme even with no hidden layers.


On average, \dotears{} outperforms all tested methods in structural recovery (Structural Hamming Distance (SHD), Figure \ref{fig:large simulations SHD results}) and edge weight recovery ($\ell_1$ distance, Supplementary Figure \ref{suppfig:large simulation L1 results}). Furthermore, \dotears{} outperforms all other methods in most parameterizations. GIES matches \dotears{} in ``Low Density'' simulations, but performs substantially worse in ``High Density'' simulations. We hypothesize that the greedy nature of GIES hinders performance in more complex DAGs. 

The primary modeling assumptions of \dotears{} are \textbf{1.)} hard interventions (Assumption \ref{assumption:hard intervention assumption}), \textbf{2.)} shared $\alpha$ across interventions (Assumption \ref{assumption:global a assumption}), and \textbf{3.)} linearity of the SEM. In Supplementary Material \ref{section:linear} and \ref{section:nonlinear} we assess the sensitivity of \dotears{} to each assumption. In addition, in Supplementary Material \ref{section:linear} we also consider \textbf{4.)} simulations under different interventional models. 

We find that \dotears{} performance can change under violations of the hard intervention assumption, but is robust to violations of its interventional model and linearity. Surprisingly, \dotears{} is the second best performing method under a mean-shift intervention model. Moreover, \dotears{} outperforms the neural network method DCDI-DSF in nonlinear simulations, even with ``imperfect'' interventions. Under ``imperfect'' interventions the primary determining factor is not nonlinearity, but the denseness of the DAG.

\section{Results}
\begin{figure}
    \centering
    \includegraphics[width=\textwidth]{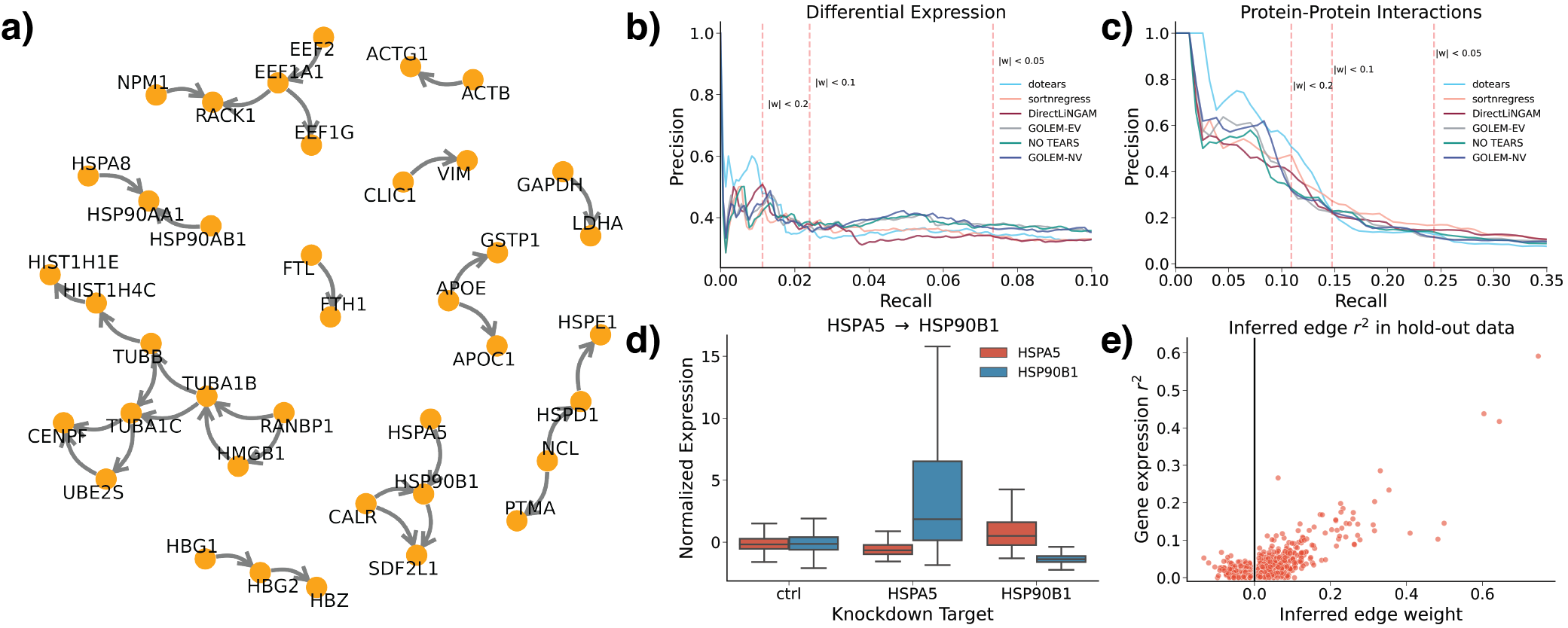}
    \caption{a) \dotears{}-inferred network. Edges with magnitude less than 0.2, and genes without inferred edges, were removed. b) Precision-recall curves across differential expression calls made by DESeq2. Dashed red lines indicate recall of \dotears{} at thresholds of $|w| < 0.2, 0.1,$ and $0.05$ respectively. c) Precision-recall curves across high confidence protein-protein interactions nominated by STRING. Dashed red lines indicate recall of \dotears{} at thresholds of $|w| < 0.2, 0.1,$ and $0.05$ respectively. d) \dotears{} infers  HSP90AB1 $\rightarrow$ HSP90AA1. HSP90AB1 knockdown increases expression of HSP90AA1, but HSP90AA1 knockdown does not change HSP90AB1 expression. d) \dotears{} inferred edges show correlated gene expression in hold-out observational data.}
    \label{fig:gwps}
\end{figure}

We apply all benchmarked methods in Section \ref{section:simulations} to a genome-wide Perturb-seq experiment from Replogle \etal{} \cite{replogle2022mapping}. We validate inferred edges through \textbf{1.)} differential expression tests in the training data using DESeq2 \cite{love2014moderated, ahlmann2020glmgampoi} and \textbf{2.)} an orthogonal set of high-confidence protein-protein interactions from the STRING database \cite{szklarczyk2023string}. We also examine gene-gene correlations in held-out observational expression data. High-confidence edges inferred by \dotears{} show differential expression and/or protein-protein interactions more frequently than those found by other methods, and \dotears{} outperforms all other methods in precision and recall under reasonable thresholding.

Replogle \etal{} provide both normalized and raw data. We select the top 100 most variable genes in the raw observational data. We then benchmark all methods on the normalized, feature-selected data. Cross-validation is not performed due to low sample sizes in some knockdowns; instead, the $\ell_1$ penalty is arbitrarily set to 0.1 for all methods where appropriate. Figure \ref{fig:gwps}a shows the network inferred by \dotears{}, thresholded at $|w| < 0.2$. For full details, see Supplementary Material \ref{suppsection:gwps}.

We use DESeq2 differential expression calls and high-confidence protein-protein interactions from the STRING database to validate inferred edges \cite{love2014moderated, ahlmann2020glmgampoi, szklarczyk2023string}. For differential expression, we call an edge $i\rightarrow j$ a true positive if either gene shows differential expression under knockdown of the other. This is because all methods struggle equally with predicting directionality (see Supplementary Table \ref{supptab:directionality}). For protein-protein interactions, we take high-confidence physical interactions as true positives, where here ``high-confidence'' is defined by STRING as having a confidence level of over 70\% \cite{szklarczyk2023string}. Figures \ref{fig:gwps}b and \ref{fig:gwps}c show precision and recall across thresholds for differential expression calls and the protein-protein interactions, respectively. Vertical lines indicate different thresholding regimes. 

\dotears{} shows much higher precision than all other methods at equivalent recall. Over 65\% of inferred edges validated by either differential expression or high-confidence protein-protein interactions. GIES, IGSP, UT-IGSP, and DCDI are excluded because they infer binary edges. These methods inferred 3038, 3064, 3075, and 2039 out of a possible 4950 edges, respectively; no other method predicted more than 700 even at a weight threshold of .05. Accordingly, they had almost random precision - see Supplementary Material \ref{suppsection:gwps tables} for detailed results for all methods at multiple thresholds. These results reinforce concerns about the scalability of GIES to more complex scenarios. For DCDI, we report intermediate results, since convergence was not obtained under 24 hours on GPU training. CPU training attempts ran out of memory even after allocation of 180GB; GPU training attempts also repeatedly ran out of memory.

We re-run \dotears{} on the same data excluding the observational data, and use the held-out observational data to validate inferred edges in the knockdown data. For each pair of genes $i, j$, Figure \ref{fig:gwps}e shows the inferred edge weight against the $r^2$ of $i, j$ in the observational data. Note that we take the largest magnitude weight between $w_{ij}$ and $w_{ji}$. Figure \ref{fig:gwps}e shows a clear relationship between inferred edge weight in knockdown data and the $r^2$ in observational expression data.

\section{Discussion}
We present \dotears{}, a structure learning framework that uses interventional data to estimate exogenous variance structure and subsequently leverages observational and interventional data to learn the causal graph.
We showed that \dotears{} is appropriate for Peturb-seq data analysis and can recover high-confidence gene regulation events.
Additionally, we verify that \dotears{} is robust to exogenous variance structure, prove that its loss function provides consistent DAG estimation, and show that \dotears{} outperforms all tested state-of-the-art methods in simulations. 

In simulations simple methods generally outperform complex methods in structure recovery, even under complex generating mechanisms. In particular, \dotears{} and sortnregress regularly outperform more complex methods, including neural network methods, even under modeling assumption violations and nonlinear data. Their strong performance also points to the efficacy of using patterns in variance to infer structure.

In real data, \dotears{} infers DAGs with edges supported by knockdown expression in training data, orthogonal high-confidence protein-protein interactions, and gene expression correlations in held out observational data. \dotears{}-inferred edges validate with higher precision than any other method without sacrificing power. In particular, edges with high inferred edge weights - in other words, the edges \dotears{} infers with highest ``confidence'' - are validated more frequently than other methods' inferences, regardless of validation scheme. We value performance in these ``high-confidence'' edges most highly, since these edges are the top candidates for follow-up experimental validation.

\clearpage
\bibliographystyle{unsrt}
\bibliography{main}

\clearpage

\section{Supplementary Material}
\subsection{Least squares infers varsortable structures in expectation}
\label{suppsection:least squares and varsortability}
We first derive an explicit cutoff for varsortability, matching the cutoff given in Eq. \ref{eq:varsortability cutoff}.

\begin{lemma}
    Let $\gamma \coloneqq \frac{\sigma_1^2}{\sigma_2^2}$. The system $X_1 \overset{w}{\rightarrow} X_2$ is varsortable if and only if $|w| \geq \sqrt{1 - \frac{1}{\gamma}}$.
    \label{lemma:varsortability cutoff}
\end{lemma}

For the proof, see Supplementary Material \ref{suppsection:varsortability cutoff proof}. We verify that the least squares loss infers varsortable structures by comparing the SEM in Eq. \ref{eq:two node SEM} against the false but Markov equivalent model $X_1 \overset{\delta}{\leftarrow} X_2$, with weighted adjacency matrix $ W_\delta \coloneqq \begin{pmatrix}
        0 & 0 \\
        \delta & 0
    \end{pmatrix}$ and SEM
\begin{equation}
    \begin{aligned}
        X_1 &= \delta X_2 + \epsilon_1, \\
        X_2 &= \epsilon_2.
    \end{aligned}
\end{equation}
Let $\left(\cdot\right)_i$ represent the $i$th column vector, and denote the least squares loss as
\begin{equation}
        \ell \left(W, \boldX \right) \coloneqq  \frac{1}{2n_0} \sum_{i = 1}^p \frobnorm{\left(\boldX - \boldX W\right)_i}.
    \label{eq:component-wise definition of loss}
\end{equation}

\begin{theorem}
    Let $\gamma \coloneqq \frac{\sigma_1^2}{\sigma_2^2}$. $\EE \ell\left(W_0, \boldX\right) \leq \EE\ell\left(W_\delta, \boldX\right)$ for all $\delta$ if and only if $|w| \geq \sqrt{1 - \frac{1}{\gamma}}$. 
    \label{thm:notears observational loss delta}
\end{theorem}

\begin{corollary}
    In the system $X_1 \overset{w}{\rightarrow} X_2$, the least squares loss is uniquely minimized in expectation by the true DAG if and only if the system is varsortable.
    \label{cor:notears is varsortability}
\end{corollary}

For the proof of Theorem \ref{thm:notears observational loss delta}, see Supplementary Material \ref{suppsection:notears observational loss delta proof}.

\subsubsection{Proof of Lemma \ref{lemma:varsortability cutoff}}
\label{suppsection:varsortability cutoff proof}
\begin{proof}
    Under the generative SEM in Eq. \ref{eq:two node SEM}, we have
    \begin{equation}
        \begin{aligned}
            \VV \left(X_1\right) &= \sigma_1^2 \\
            \VV \left(X_2\right) &= w^2\sigma_1^2 + \sigma_2^2
        \end{aligned}.
    \end{equation}

    The system is therefore varsortable if and only if $w^2\sigma_1^2 + \sigma_2^2 \geq \sigma_1^2$. Using the substitution $\sigma_1^2 = \gamma \sigma_2^2$, we obtain
    \begin{equation}
        \begin{aligned}
            w^2\sigma_1^2 + \sigma_2^2 &\geq \sigma_1^2 \\
            w^2 \gamma \sigma_2^2 + \sigma_2^2 &\geq \gamma \sigma_2^2 \\
            w^2 \gamma + 1 &\geq \gamma \\
            |w| &\geq \sqrt{1 - \frac{1}{\gamma}}
        \end{aligned}
    \end{equation}
\end{proof}
\subsubsection{Proof of Theorem \ref{thm:notears observational loss delta}}
\label{suppsection:notears observational loss delta proof}
\begin{proof}
From Eq. \ref{eq:component-wise definition of loss}, we can calculate both $\EE\ell\left(W_0, \boldX\right)$ and $\EE \ell\left(W_\delta, \boldX\right)$ component-wise. Under $W_0$ we obtain component-wise
\begin{equation}
    \begin{aligned}
        \left({X} - {X}W_0 \right)_1 &= {X_1}\\
        &= {\epsilon_1} & \text{by Eq. \ref{eq:two node SEM}} \\
        \left({X} - {X}W_0 \right)_2 &= {X_2} - w{X_1}\\
        &= {\epsilon_2} & \text{by Eq. \ref{eq:two node SEM}} \\
    \end{aligned}
    \label{eq:two node observational component-wise true model}
\end{equation}
As a result, 
\begin{equation}
    \begin{aligned}
        \EE\ell\left(W_0, \boldX\right) 
        &= \frac{1}{2n_0}\sum_{i=1}^p \EE \frobnorm{\left({\boldX} - {\boldX}W_0\right)_i} \\
        &= \frac{1}{2n_0} \left(\EE \left(\frobnorm{{\boldepsilon_1}}\right) + \EE \left(\frobnorm{\boldepsilon_2}\right) \right)\\
        &= \frac{1}{2} \left(\sigma_1^2 + \sigma_2^2\right) \nonumber
    \end{aligned}
\end{equation}
Here we note that the expected loss under the generative DAG is $\EE\ell\left(W_0, \boldX\right) = \frac{1}{2}\sum_{i=1}^p \VV\left({\epsilon_i}\right)$.

Similarly, under $W_\delta$ we have
\begin{equation}
    \begin{aligned}
        \left(X - XW_\delta \right)_1 
        &= X_1 - \delta X_2\\
        &= \epsilon_1 - \delta\left(w \epsilon_1 + \epsilon_2 \right) 
        & \text{by Eq. \ref{eq:two node SEM}} \\
        &= (1 - \delta w) \epsilon_1 - \delta \epsilon_2 \\
        \left(X - XW_\delta \right)_2 
        &= X_2 \\
        &= w \epsilon_1 + \epsilon_2
        & \text{by Eq. \ref{eq:two node SEM}} \\
    \end{aligned}
    \label{eq:two node observational component-wise false model}
\end{equation}
and therefore $\EE \ell \left(W_\delta, \boldX\right)$ is
\begin{equation}
    \frac{1}{2}\left(\sigma_1^2 \left(w^2 + \left(1 - \delta w\right)^2\right) + \sigma_2^2 \left(1 + \delta^2\right)\right). \nonumber
\end{equation}

We can now ask when $\EE \ell\left(W_\delta, \boldX\right) \leq \EE \ell \left(W_0, \boldX\right)$, or equivalently,
\begin{equation}
    \begin{aligned}
        \sigma_1^2 \left(w^2 + \left(1 - \delta w\right)^2\right) + \sigma_2^2 \left(1 + \delta^2\right) 
        &\leq \sigma_1^2 + \sigma_2^2 \\
        \gamma \sigma_2^2 \left(w^2 + \left(1 - \delta w\right)^2\right) + \sigma_2^2 \left(1 + \delta^2\right) 
        &\leq \gamma\sigma_2^2 + \sigma_2^2 \\
        \gamma \left(w^2 + 1 - 2\delta w + \delta^2w^2\right) + 1 + \delta^2 
        &\leq \gamma + 1 \\
        \gamma \left(w^2 - 2\delta w + \delta^2w^2\right) + \delta^2 
        &\leq 0 \\
        \delta^2 \left(1 + \gamma w^2\right) - \delta \left(2\gamma w\right) + \gamma w^2 &\leq 0\\
    \end{aligned}
    \label{eq:quadratic in delta}
\end{equation}
The inequality becomes a quadratic in $\delta$, which when solved gives us roots at 
\begin{equation}
    \begin{aligned}
        \delta &= \frac{2\gamma w \pm \sqrt{(-2\gamma w)^2 - 4\left(1 + \gamma w^2\right)\left(\gamma w^2\right)}}{2\left(1 + \gamma w^2\right)} \\
        &= \frac{2\gamma w \pm 2\gamma w\sqrt{1 - \frac{1}{\gamma} \left(1 + \gamma w^2\right)}}{2\left(1 + \gamma w^2\right)} \\
        &= \frac{\gamma w \left(1 \pm \sqrt{1 - w^2 - \frac{1}{\gamma}} \right)}{1 + \gamma w^2} .
    \end{aligned}
    \label{eq:delta quadratic solution}
\end{equation}
Note that for the quadratic in Eq. \ref{eq:quadratic in delta}, $\gamma w^2 \geq 0$. Then $\exists \delta$ such that $\EE \ell\left(W_\delta, \boldX\right) \leq \EE \ell \left(W_0, \boldX\right)$ when the term in the root exists in Eq. \ref{eq:delta quadratic solution}. Equivalently, we can say that $\EE \ell\left(W_0, \boldX\right) \leq \EE \ell\left(W_\delta, \boldX\right)$ is guaranteed for all $\delta$ if and only if 
\begin{equation}
    |w| \geq \sqrt{1 - \frac{1}{\gamma}}.
    \label{eq:beta bound notears}
\end{equation}
\end{proof}

\subsection{Proof of Consistency}
\label{suppsection:consistency proof}
For a diagonal matrix $\Omega$, define
 \begin{equation}
 \begin{aligned}
        \mathcal{L}_{\Omega}\left(\interv{W}{k}, \interv{\boldX}{k}\right) &\coloneqq 
        \frac{1}{n_k}\frobnorm{\left(\interv{\boldX}{k} - \interv{\boldX}{k} \interv{W}{k}\right)\Omega^{-\frac{1}{2}}} \\
        &= \frac{1}{n_k}\sum_{j=1}^p \frac{1}{\Omega_{jj}} \frobnorm{\left(\interv{\boldX}{k} - \interv{\boldX}{k} \interv{W}{k}\right)_j},
 \end{aligned}
        \label{suppeq:dotears loss definition}
\end{equation}
where $\Omega_{jj}$ is the $j, j$th entry of $\Omega$. Further, recall the definition
\begin{equation}
        \interv{\Omega_0}{k} \coloneqq \CV \left(\interv{\epsilon}{k}\right). \nonumber
\end{equation}

Denote convergence in probability by $\overset{p}{\rightarrow}.$ We show that under a sub-Gaussian assumption, the \dotears{} loss is a consistent estimator of the true DAG.

\begin{assumption}
    For any $k$, we assume $\interv{X}{k}$ is a sub-Gaussian random vector with parameter $\sigma^2$.
\end{assumption}

\begin{lemma}
    For any intervention $k$, 
    \begin{equation}
    \mathcal{L}_{{\Omega}_0}\left(\interv{W}{k}, \interv{\boldX}{k}\right) = \mathcal{L}_{\interv{\Omega_0}{k}}\left(\interv{W}{k}, \interv{\boldX}{k}\right) + \mathcal{O}_W(1) \nonumber
    \end{equation}
    \label{lemma:equivalence of loss functions}
\end{lemma}

For the proof, see Supplementary Material \ref{suppsection:lemma equivalence of loss functions proof}.

\begin{corollary}
\begin{equation}
        \arg \min_W \mathcal{L}_{\Omega_0}\left(\interv{W}{k}, \interv{\boldX}{k}\right) 
        = \arg \min_W \mathcal{L}_{\interv{\Omega_0}{k}}\left(\interv{W}{k}, \interv{\boldX}{k}\right) \nonumber
\end{equation}
\label{corollary:argmin equivalence under true omega}
\end{corollary}

For any intervention $k$ and any $\alpha$, Lemma \ref{lemma:equivalence of loss functions} and Corollary \ref{corollary:argmin equivalence under true omega} show that true exogenous variance structure $\Omega_0 = \alpha^2 \EE \hat{\Omega}_0$ is sufficient for structure recovery. Results from Loh and Buhlmann (2014) are then sufficient to establish consistency of the estimator $\argmin_W \mathcal{L}_{\Omega_0} \left(\interv{W}{k}, \interv{X}{k}\right)$. We wish to show consistency of $\argmin_W \mathcal{L}_{\hat{\Omega}_0} \left(\interv{W}{k}, \interv{\boldX}{k} \right)$, where we have an estimated $\hat{\Omega}_0$ rather than $\Omega_0$.

For simplicity, in the following we assume $\alpha = 1$ without loss of generality, which is justified by Corollary \ref{corollary:argmin equivalence under true omega} and Lemma \ref{lemma:equivalence of loss functions}. Note then that $\Omega_0 = \interv{\Omega_0}{k}$, which allows us to drop the $^{(k)}$ notation. 

\begin{remark}
    For consistency under an estimated $\hat{\Omega}_0$, we have two cases: the observational case $k = 0$ and the interventional case $k \neq 0$. The observational case is simplest - in particular, when $k=0$, $\hat{\Omega}_0$ is independent of $\interv{\boldX}{0}$, and we may therefore freely estimate $\interv{W_0}{0}$ from $\mathcal{L}_{\hat{\Omega}_0} \left(\interv{W}{0}, \interv{\boldX}{0} \right)$. When $k \neq 0$, we can assume independence of $\hat{\Omega}_0$ and $\interv{\boldX}{k}$ under a data splitting framework, where some fraction of our $n_k$ samples are reserved for estimation of $\hat{\sigma}_k^2$, and the remaining samples are given in $\interv{\boldX}{k}$.
\end{remark}

\begin{assumption}
    Let 
    \begin{equation}
        \hat{\sigma}_k^2 \coloneqq \left(\hat{\Omega}_0\right)_{kk}. \nonumber 
    \end{equation}
    Then
    \begin{equation}
        \hat{\sigma}_k^2 \independent \interv{\boldX}{k}_k. \nonumber
    \end{equation}
\end{assumption}
However, in the presented applications we use the full $n_k$ samples of $\interv{\boldX_k}{k}$ for both estimates.

Note that we abuse notation to write $\interv{\boldX}{k}$ as the set of samples used to estimate $\interv{\hat{W}_0}{k}$ given $\hat{\Omega}_0$. Further, for convenience we abuse notation to write $n_k$ as the sample size of $\interv{\boldX}{k}$ AND the sample size of $\hat{\sigma}_k^2$, since the respective sample sizes are equal to $n_k$ up to a multiplicative constant which can be ignored as $n_k \to \infty$.

We further define
\begin{equation}
    \EE \mathcal{L}_{\Omega}\left(\interv{W}{k}, \interv{\boldX}{k}\right) \\
    \coloneqq \frac{1}{n_k}\sum_{j=1}^p \frac{1}{\Omega_{jj}} \EE \frobnorm{\left(\interv{\boldX}{k} - \interv{\boldX}{k} \interv{W}{k}\right)_j}. \nonumber
\end{equation}
Here, we abuse notation to remove $\Omega_{jj}$ from the expectation. Moreover, for the random vector $X$ we write
\begin{equation}
    \EE \mathcal{L}_{\Omega}\left(\interv{W}{k}, \interv{X}{k}\right) \\
    \coloneqq \sum_{j=1}^p \frac{1}{\Omega_{jj}} \EE \left[\left(\interv{X}{k} - \interv{X}{k}\interv{W}{k} \right)_j^2\right], \nonumber
\end{equation}
 and note that the two definitions are equivalent under expectation. Note further that 
\begin{equation}
    \score_\Omega(W) \equiv \EE \mathcal{L}_\Omega (W, X), \nonumber
\end{equation}
where $\score$ is defined in \cite{loh2014high}.

\begin{remark}
    Loh and Buhlmann differentiate between the \textbf{weighted adjacency matrix} $W$ and the \textbf{binary} DAG $G$. In particular, they define
    \begin{equation}
        \score_\Omega \left(G, X \right) \coloneqq \min_{W \in \mathcal{U}_G} \left\{\score_\Omega \left(W, X \right) \right\}, \nonumber
    \end{equation}
    where $\mathcal{U}_G$ is the set of weighted adjacency matrices with support in $G$. However, for the true binary DAG $G_0$,
    \begin{equation}
        \score_\Omega \left(G, X \right) = \score_\Omega \left(G_0, X \right) \quad \forall G \supseteq G_0, \nonumber
    \end{equation}
    which leads to identifiability issues (see Lemma 6 and discussion of Lemma 19 \cite{loh2014high}). 
    
    We define everything directly on the weighted adjacency matrix $W$, and ignore the binary DAG $G$. This discrepancy is essentially due to methodological differences. Loh and Buhlmann first find a superset of the binary structure $G_0$, and subsequently rely on sparse regression to recover edge weights \cite{loh2014high}. However, \dotears{} searches directly in the space of weighted adjacency matrices to find $W_0$. While the end result is similar, we use this remark to explain discrepancies in notation and proof structure for readers following with \cite{loh2014high}. 
\end{remark}

We start by restating results from \cite{loh2014high}. Let $\Omega_1 \in \RR ^{p\times p}$ an arbitrary diagonal weight matrix. Define
\begin{equation}
    a_\text{max} \coloneqq \lambda_\text{max} \left(\Omega_0 \Omega_1^{-1}\right), \,\, a_\text{min} \coloneqq \lambda_\text{min} \left(\Omega_0 \Omega_1^{-1}\right) \nonumber
\end{equation}
to be the maximum and minimum ratios between the corresponding diagonal entries of $\Omega_0$ and $\Omega_1$, and further define
\begin{equation}
\begin{aligned}
    \interv{\xi}{k}_{\Omega} \coloneqq \min_{\substack{
        W \in \mathcal{D} \\
        W \neq W_0
    }} \EE \mathcal{L}_\Omega \left(\interv{W}{k}, \interv{X}{k} \right) - \EE \mathcal{L}_\Omega\left(\interv{W_0}{k}, \interv{X}{k}\right)      
\end{aligned}
\label{eq:definition of gap}
\end{equation}
to be the additive gap between the expected loss of the true weighted adjacency matrix $W_0$ and the expected loss of the next-best weighted adjacency matrix under an arbitrary diagonal matrix $\Omega$. Note that $\xi_{\Omega_0} > 0$ (see Theorem 7 in \cite{loh2014high}). Then we have the following theorem, whose proof is given in \cite{loh2014high}:

\begin{theorem}
    Suppose 
    \begin{equation}
        \frac{a_{\text{max}}}{a_{\text{min}}} \leq 1 + \frac{\interv{\xi_{\Omega_0}}{k}}{p}.
        \label{eq:amax amin}
    \end{equation}
    Then $\interv{W_0}{k} \in \argmin_{W \in \mathcal{D}} \left\{\EE \mathcal{L}_{\Omega_1} \left(W, \interv{X}{k}\right)\right\}$. If Inequality \ref{eq:amax amin} is strict, then $\interv{W_0}{k}$ is the unique minimizer of $\EE \mathcal{L}_{\Omega_1}\left(W, \interv{X}{k}\right)$.
    \label{thm:loh misspecification of variances}
\end{theorem}

\begin{theorem}
    For any intervention $k$, as $n_j \to \infty$ for all $j=1,\dots, p$, with high probability $\interv{W_0}{k}$ is the unique minimizer of $\EE \mathcal{L}_{\hat{\Omega}_0}\left(\interv{W}{k}, \interv{X}{k} \right)$, and thus $\interv{\xi_{\hat{\Omega}_0}}{k} > 0$.
    \label{thm:gap nonzero and unique minimization}
\end{theorem}

For the proof, see Supplementary Material \ref{suppsection:thm gap nonzero and unique minimization proof}.

\begin{lemma}
    For any intervention $k$, suppose $\interv{X}{k}$ is a sub-Gaussian random vector with parameter $\sigma^2$. Then $\interv{X}{k}_k$ is a sub-Gaussian random variable with parameter $\sigma^2$, and for all $t \geq 0$, we have
    \begin{equation}
        \left| \frac{1}{\hat{\sigma}_k^2} - \frac{1}{\sigma_k^2} \right| \leq \frac{2\sigma^2 }{\sigma_k^2} \cdot\max \{\delta, \delta^2 \} \nonumber
    \end{equation}
    with probability at least $1 - 2\exp (-c n_k t^2)$, where $\delta = c'\sqrt{\frac{1}{n_k}} + c'' t$ and 
    \begin{equation}
        \frac{\sigma^2}{\sigma_k^2} \cdot \max \{\delta, \delta^2\} \leq \frac{1}{2}.
    \end{equation}
    \label{lemma:sub gaussian bound on precision convergence}
\end{lemma}

\begin{proof}
    Lemma \ref{lemma:sub gaussian bound on precision convergence} is a re-statement of Lemma 29 in \cite{loh2014high} (Appendix E) for the degenerate case $p = 1$.
\end{proof}

Following \cite{loh2014high}, we introduce the new notation $\interv{f}{k}_{\sigma_j}$. For a node $j$ and a set $S \subseteq \{1\dots p\} \backslash \{j\}$, define 
\begin{equation}
    \interv{f}{k}_{\sigma_j} (S) \coloneqq \frac{1}{\sigma_j^2} \EE \left[\left(\interv{X}{k}_j - \interv{X}{k}_S w_j \right)^2\right],   \nonumber
\end{equation}
where $\interv{X}{k}_j$ is the $j$th column vector of $\interv{X}{k}$, and $\interv{X}{k}_S w_j$ is the best linear predictor for $\interv{X}{k}_j$ regressed upon $\interv{X_S}{k}$. Similarly,
\begin{equation}
    \interv{\hat{f}_{\sigma_j}}{k}(S) \coloneqq \frac{1}{\sigma_j^2} \frac{1}{n_k} \frobnorm{\interv{\boldX}{k}_j - \interv{\boldX}{k}_S \hat{w}_j}, \nonumber
\end{equation}
where $\hat{w}_j$ is the ordinary least squares solution for linear regression of $\interv{\boldX}{k}_j$ upon $\interv{\boldX}{k}_S$, i.e.
\begin{equation}
\begin{aligned}
    \hat{w}_j &\coloneqq \left(\left(\left(\interv{\boldX}{k}_S\right)\T\interv{\boldX}{k}_S\right)^{-1} \left(\interv{\boldX}{k}_S\right)\T \interv{\boldX}{k}_j \right) \\
    & = \left(\left(\left(\interv{\boldX}{k}_S\right)\T\interv{\boldX}{k}_S\right)^{-1} \left(\interv{\boldX}{k}_S\right)\T \left(\interv{\boldX}{k}_Sw_j + \interv{\bolde}{k}_j\right) \right) \\
    &= w_j + \left(\left(\interv{\boldX}{k}_S\right)\T\interv{\boldX}{k}_S\right)^{-1}\left(\interv{\boldX}{k}_S\right)\T\interv{\bolde}{k}_j,
\end{aligned}
\label{eq:least squares estimate form}
\end{equation}
where note that 
\begin{equation}
    \interv{\boldX_j}{k} = \interv{\boldX_S}{k}w_j + \interv{\bolde}{k}_j. \nonumber
\end{equation}
Lastly, denote the vector $\ell_2$ norm as $\|\|\cdot \|\|_2$.

Loh and Buhlmann achieve a high-dimensional consistency result by first conditioning on the support of a sparse precision matrix $\interv{\Theta}{k}$, since the support of $\interv{\Theta}{k} \coloneqq \CV \left(\interv{X}{k}\right)^{-1}$ defines the \textit{moralized} graph, an undirected graph obtained from a DAG by adding edges between all parents with a shared child node \cite{loh2014high, wang2018direct}. Since the edge set of the moralized graph is a superset of the edge set of the true DAG, for a given node $j$ one may condition on the maximum size of the putative neighbor set, $N_{\interv{\Theta}{k}} (j)$. Indeed, in \cite{loh2014high} Loh and Buhlmann restrict $|N_{\interv{\Theta}{k}} (j)| \leq d$ for all $j$, with the only restriction on $d$ being that $d \leq n$.

\dotears{} does not condition on $\interv{\Theta}{k}$, and therefore does not condition on the moralized graph or $N_\Theta(j)$. To maintain the validity of our consistency proof, we let $d = p - 1 \leq n$. Under this restriction, we no longer have a high-dimensional result, but maintain consistency of the loss function of \dotears{} for low dimensionality of $p$.

\begin{assumption}
    For all $k=0,\dots, p$, let $d = p - 1 \leq n_k$. 
    \label{assumption:p, d, n}
\end{assumption}

\begin{lemma}
    For any intervention $k$, suppose $\interv{X}{k}$ is sub-Gaussian with parameter $\sigma^2$. Then $\forall j \text{ and } S \subseteq \{1\dots p\} \backslash \{j\}$,
    \begin{equation}
        \sigma_j^2 \left| \interv{\hat{f}}{k}_{\hat{\sigma}_j} (S) - \interv{\hat{f}}{k}_{\sigma_j}(S) \right| \leq \sigma^2 \cdot \max \{\delta, \delta^2\} \cdot \frac{C}{n_k} \left|\left|\interv{\bolde}{k}_j\right|\right|_2^2
        \label{eq:fhatsigmahat - fhatsigma}
    \end{equation}
    with probability $1 - c_1 \exp (-c_2 \log p)$, where $\delta = c'\sqrt{\frac{1}{n_j}} + c'' \sqrt{\frac{\log p}{n_j}}$, if 
    \begin{equation}
        \frac{\sigma^2}{\sigma_j^2} \cdot \delta \leq \frac{1}{2}.
    \end{equation}
    \label{lemma:fhatsigmahat - fhatsigma}
\end{lemma}

For the proof, see Supplementary Material \ref{suppsection:lemma fhatsigmahat - fhatsigma proof}.

\begin{lemma}
    For any intervention $k$, suppose $\interv{X}{k}$ is sub-Gaussian with parameter $\sigma^2$. Then as $n_i \to \infty$ for all $i=1\dots p$, $\forall j \text{ and } S \subseteq \{1\dots p\} \backslash \{j\}$ it is true that 
    \begin{equation}
        \left| \interv{\hat{f}}{k}_{\hat{\sigma}_j} (S) - \interv{f}{k}_{\sigma_j} (S) \right| \leq c_0 \sigma^4 \sqrt{\frac{\log p}{n_j}} + c_1 \sigma^2 \sqrt{\frac{\log p}{n_k}} + c_2 \frac{p}{n_k}
        \label{eq:f final bound}
    \end{equation}
    with probability at least $1 - c_1 \exp (-c_2 \log p)$. 
    \label{lemma:f bound lemma}
\end{lemma}

For the proof, see Supplementary Material \ref{suppsection:f bound lemma proof}.

\begin{theorem}
    For any intervention $k$, suppose Inequality \ref{eq:f final bound} holds, and suppose 
    \begin{equation}
        \left(c_0 \sigma^4 \sqrt{\frac{\log p}{n_j}} + c_1 \sigma^2 \sqrt{\frac{\log p}{n_k}} + c_2 \frac{p}{n_k} \right) \cdot \sum_{j=1}^p \frac{1}{\sigma_j^2} < \frac{\interv{\xi_{{\Omega}_0}}{k}}{2}.
        \label{eq:sum f final bound}
    \end{equation}
    Then
    \begin{equation}
        \mathcal{L}_{\hat{\Omega}_0}\left(\interv{W_0}{k}, \interv{\boldX}{k} \right) < \mathcal{L}_{\hat{\Omega}_0}\left(\interv{W}{k}, \interv{\boldX}{k} \right)\nonumber
    \end{equation}
    $\forall W \in \mathcal{D},$ and the estimator
    \begin{equation}
        \interv{\hat{W}}{k} \coloneqq \argmin_W \mathcal{L}_{\hat{\Omega}_0}\left(W, \interv{\boldX}{k}\right) 
    \end{equation} 
    is therefore consistent as $n_j \to \infty$ for all $j=1\dots p$. 
    \label{thm:estimator consistency}
\end{theorem}

For the proof, see Supplementary Material \ref{suppsection:theorem estimator consistency proof}.
\subsubsection{Proof of Lemma \ref{lemma:equivalence of loss functions}}
\label{suppsection:lemma equivalence of loss functions proof}
\begin{proof}
Note that
\begin{equation}
\begin{aligned}
         \mathcal{L}_{\interv{\Omega_0}{k}}\left(\interv{W}{k}, \interv{\boldX}{k}\right) 
         \coloneqq & \frac{1}{n_k} \frobnorm{\left(\interv{\boldX}{k} - \interv{\boldX}{k}\interv{W}{k}\right)\left(\interv{\Omega_0}{k}\right)^{-\frac{1}{2}}} \\
         = & \frac{1}{n_k} \sum_{i=1}^p \frac{1}{\VV \left(\interv{\epsilon_i}{k}\right)}  \frobnorm{\left(\interv{\boldX}{k} - \interv{\boldX}{k}\interv{W}{k}\right)_i} \\
         =& \frac{1}{n_k} \frac{\alpha^2}{\sigma_k^2} \frobnorm{\left(\interv{\boldX}{k} - \interv{\boldX}{k}\interv{W}{k}\right)_k} 
         + \frac{1}{n_k} \sum_{\substack{i = 1 \\ i \neq k}}^p \frac{1}{\sigma_i^2}  \frobnorm{\left(\interv{\boldX}{k} - \interv{\boldX}{k}\interv{W}{k}\right)_i} \\
         \mathcal{L}_{\Omega_0} \left(\interv{W}{k}, \interv{\boldX}{k}\right) 
         \coloneqq & \frac{1}{n_k} \frobnorm{\left(\interv{\boldX}{k} - \interv{\boldX}{k}\interv{W}{k}\right)\Omega_0^{-\frac{1}{2}}} \\
         = &  \frac{1}{n_k}\sum_{i=1}^p \frac{1}{\sigma_i^2} \frobnorm{\left(\interv{\boldX}{k} - \interv{\boldX}{k}\interv{W}{k}\right)_i}. \nonumber
    \end{aligned}
\end{equation} 

Let $I_0^k = \text{diag}(1, 1, \dots, 0, \dots, 1)$ a modifier on the $p\times p$ identity matrix, where the diagonal entries are all $1$ except for the $k, k$ entry, which is $0$. 

Then in the interventional system,
\begin{equation}
        \interv{W}{k} = WI_0^k  \nonumber
\end{equation}
and thus
\begin{equation}
        \mathcal{L}_{\Omega_0}\left(\interv{W}{k}, \interv{\boldX}{k}\right) = \frac{1}{n_k}\frobnorm{\left(\interv{\boldX}{k} - \interv{\boldX}{k}WI_0^k\right)\Omega_0^{-\frac{1}{2}}} \nonumber
\end{equation} 

We expand $\mathcal{L}_{\Omega_0}$ to obtain
\begin{equation}
\begin{aligned}
        \mathcal{L}_{\Omega_0}\left(\interv{W}{k}, \interv{\boldX}{k}\right) 
        =& \frac{1}{n_k}\sum_{i=1}^p \frac{1}{\sigma_i^2}\frobnorm{\left(\interv{\boldX}{k} - \interv{\boldX}{k}\interv{W}{k}\right)_i} \\
        =& \frac{1}{n_k}\frac{1}{\sigma_k^2}\frobnorm{\left(\interv{\boldX}{k} - \interv{\boldX}{k}\interv{W}{k}\right)_k} + \frac{1}{n_k}\sum_{\substack{i=1 \\ i \neq k}}^p \frac{1}{\sigma_i^2}\frobnorm{\left(\interv{\boldX}{k} - \interv{\boldX}{k}\interv{W}{k}\right)_i}\\    \nonumber 
\end{aligned}
\end{equation}
and note that 
\begin{equation}
\begin{aligned}
    \left(\interv{\boldX}{k}\interv{W}{k}\right)_k &= \interv{\boldX}{k}\left( W I_0^k\right)_k\\
    &= \interv{\boldX}{k} \Vec{0}_p \\
    &= \Vec{0}_{n_k}. \nonumber
\end{aligned}
\end{equation}
As a result, 
\begin{equation}
    \frobnorm{\left(\interv{\boldX}{k} - \interv{\boldX}{k}\interv{W}{k}\right)_k} = \frobnorm{\interv{\boldX_k}{k}} = \mathcal{O}_W(1) \nonumber
\end{equation}
is constant in $W$, and
\begin{equation}
    \begin{aligned}
        \mathcal{L}_{\Omega_0}\left(\interv{W}{k}, \interv{\boldX}{k}\right) 
        =& \frac{1}{n_k}\frac{1}{\sigma_k^2}\frobnorm{\left(\interv{\boldX}{k} - \interv{\boldX}{k}\interv{W}{k}\right)_k} 
        + \frac{1}{n_k}\sum_{\substack{i = 1 \\ i \neq k}}^p \frac{1}{\sigma_i^2} \frobnorm{\left(\interv{\boldX}{k} - \interv{\boldX}{k}\interv{W}{k}\right)_i} \\
        =& \frac{1}{n_k}\frac{1}{\sigma_k^2}\frobnorm{\interv{\boldX_k}{k}} 
        + \frac{1}{n_k}\sum_{\substack{i = 1 \\ i \neq k}}^p \frac{1}{\sigma_i^2} \frobnorm{\left(\interv{\boldX}{k} - \interv{\boldX}{k}\interv{W}{k}\right)_i} \\
        =& \mathcal{L}_{\interv{\Omega_0}{k}}\left(\interv{W}{k}, \interv{\boldX}{k}\right) + \mathcal{O}_W(1) \nonumber
    \end{aligned}
\end{equation}
\end{proof}

\subsubsection{Proof of Theorem \ref{thm:gap nonzero and unique minimization}}
\label{suppsection:thm gap nonzero and unique minimization proof}
\begin{proof}
    Let $a_\text{max} \coloneqq \lambda_\text{max} \left(\Omega_0\hat{\Omega}_0^{-1}\right)$, and $a_\text{min}\coloneqq \lambda_\text{min} \left(\Omega_0\hat{\Omega}_0^{-1}\right)$ similarly. We first prove that $a_\text{max} \overset{p}{\rightarrow} \alpha^2$ without loss of generality. For all $k$, note first that 
    \begin{equation}
        \frac{\sigma_k^2}{\hat{\sigma}_k^2} \overset{p}{\rightarrow} \alpha^2. \nonumber
    \end{equation}
    Then by Continuous Mapping Theorem we have 

    \begin{equation}
        a_\text{max} \overset{p}{\rightarrow} {\alpha^2}. \nonumber
    \end{equation}
    Similarly, $a_\text{min} \overset{p}{\rightarrow} {\alpha^2}$, and therefore
    \begin{equation}
        \frac{a_\text{max}}{a_\text{min}} \overset{p}{\rightarrow} 1. \nonumber
    \end{equation}

    We note that $\frac{a_\text{max}}{a_\text{min}} = 1$ is almost impossible empirically, but can be controlled with high probability to $1$ with arbitrary precision. Then by Theorem \ref{thm:loh misspecification of variances}, with high probability $\interv{W_0}{k}$ is the unique minimizer of $\EE \mathcal{L}_{\hat{\Omega}_0}\left(\interv{W}{k}, \interv{X}{k} \right)$ as $n_j \to \infty$ for all $j=1\dots p$. Further, $\interv{\xi_{\hat{\Omega}_0}}{k} > 0$. 
\end{proof}

\subsubsection{Proof of Lemma \ref{lemma:fhatsigmahat - fhatsigma}}
\label{suppsection:lemma fhatsigmahat - fhatsigma proof}
\begin{proof}
    Let the projection matrix of $\interv{\boldX_S}{k}$ be defined as $P_{\interv{\boldX_S}{k}} \coloneqq  \interv{\boldX_S}{k} \left(\left(\interv{\boldX_S}{k}\right)\T \interv{\boldX_S}{k} \right)^{-1} \left(\interv{\boldX_S}{k}\right)\T$.
    
    Given the least squares estimate in Eq. \ref{eq:least squares estimate form}, we may write
    \begin{equation}
        \begin{aligned}
            \sigma_j^2 \cdot \interv{\hat{f}}{k}_{\sigma_j}(S) &= \frac{1}{n_k} \frobnorm{\interv{\boldX}{k}_j - \interv{\boldX_S}{k} \hat{w}_j} \\
            &= \frac{1}{n_k} \frobnorm{\interv{\boldX_j}{k} \left(w_j - \hat{w}_j\right) + \interv{\bolde}{k}_j} \\
            &= \frac{1}{n_k}\frobnorm{\left(I - P_{\interv{\boldX_S}{k}}  \right)\interv{\bolde_j}{k}}.
        \end{aligned}
        \label{eq:fhat expanded form}
    \end{equation}

    Note by Triangle Inequality, we have for the $\ell_2$ norm $\left|\left| \cdot \right| \right|_2$ and the spectral norm $\left|\left|\left| \cdot \right|\right| \right|_2$
    \begin{equation}
        \begin{aligned}
             \left| \left|\left|\left(I - P_{\interv{\boldX_S}{k}}  \right)\interv{\bolde_j}{k} \right|\right|_F - \left|\left| \interv{\bolde_j}{k} \right|\right|_2 \right| 
            \leq & \left| \left| P_{\interv{\boldX_S}{k}} \interv{\bolde_j}{k}  \right| \right|_F \\
            \leq &\left|\left|\left| P_{\interv{\boldX_S}{k}} \right|\right|\right|_2 \left| \left| \interv{\bolde_j}{k}\right| \right|_2 \\
            \leq &\left| \left| \interv{\bolde_j}{k} \right| \right|_2, \nonumber
        \end{aligned}
    \end{equation}
    where the spectral norm of the projection matrix $\left|\left|\left| P_{\interv{\boldX_S}{k}} \right|\right|\right|_2$ is 1. Then
    \begin{equation}
        \sigma_j^2 \cdot \interv{\hat{f}}{k}_{\sigma_j}(S) = \frac{1}{n_k} \frobnorm{\interv{\boldX}{k}_j - \interv{\boldX_S}{k} \hat{w}_j} \leq \frac{2}{n_k} \left|\left|\interv{\bolde_j}{k}\right|\right|_2^2. \nonumber
    \end{equation}

    Expanding the left hand side of Eq. \ref{eq:fhatsigmahat - fhatsigma}, we obtain
    \begin{equation}
        \begin{aligned}
            \sigma_j^2 \left| \interv{\hat{f}}{k}_{\hat{\sigma}_j} (S) - \interv{\hat{f}}{k}_{\sigma_j}(S) \right|
            =& \sigma_j^2 \left| \frac{1}{\hat{\sigma}_j^2} - \frac{1}{\sigma_j^2} \right| \frac{1}{n_k} \frobnorm{\interv{\boldX_j}{k} - \interv{\boldX_S}{k} \hat{w}_j } \\  
            \leq & \sigma_j^2 \cdot 2 \frac{\sigma^2}{\sigma_j^2} \cdot \max \{ \delta, \delta^2\} \cdot \frac{1}{n_k}\frobnorm{\interv{\boldX_j}{k} - \interv{\boldX_S}{k} \hat{w}_j } \\
            \leq & 2\sigma^2 \cdot \max \{\delta, \delta^2\} \cdot \frac{2}{n_k} \left|\left|\interv{\bolde}{k}_j\right|\right|_2^2 \nonumber
        \end{aligned}
    \end{equation}
    with probability $\min \{1 - c_1 \exp (-c_2 \log p ), 1 - 2 \exp (-c n_j t^2) \}$, where $\delta = c' \sqrt{\frac{1}{n_j}} + c'' t$, if $\frac{\sigma^2}{\sigma_j^2}\cdot \max \{\delta, \delta^2 \} \leq \frac{1}{2}$. The first inequality is by Lemma \ref{lemma:sub gaussian bound on precision convergence}. Set $t = \sqrt{\frac{\log p}{n_j}}$. Then 
    \begin{equation}
        \begin{aligned}
            \delta &= c'\sqrt{\frac{1}{n_j}} + c''\sqrt{\frac{\log p}{n_j}}, \\
            \delta^2 &= (c')^2 \frac{1}{n_j} + (c'')^2 \frac{\log p}{n_j} + 2c'c'' \frac{\sqrt{\log p}}{n_j}. \nonumber
        \end{aligned}
    \end{equation}
    By Assumption \ref{assumption:p, d, n}, $\log p < p - 1 < n_j$. Then $\max\{\delta, \delta^2\} = \delta$, and $1 - 2 \exp (-c n_j t^2) = 1 - 2\exp (-c \log p)$.
\end{proof}

\subsubsection{Proof of Lemma \ref{lemma:f bound lemma}}
\label{suppsection:f bound lemma proof}
\begin{proof}
Give the singular value decomposition on the $n_k$ by $p$ matrix $\interv{\boldX}{k}_S$ as 
\begin{equation}
    \interv{\boldX}{k}_S = U\Sigma V\T, \nonumber
\end{equation}
where $U$ is an $n_k \times p$ matrix, $\Sigma$ is a diagonal $p \times p$ matrix, $V$ is a $p \times p$ matrix, and $U\T U = V \T V = I$. Then 
\begin{equation}
    \begin{aligned}
        \interv{\boldX_S}{k} \left(\left(\interv{\boldX_S}{k}\right)\T \interv{\boldX_S}{k} \right)^{-1} \left(\interv{\boldX_S}{k}\right)\T 
        = & U\Sigma V\T \left(V \Sigma^2 V \T \right)^{-1} V \Sigma U\T \\
        = & U U\T. \nonumber
    \end{aligned}
\end{equation}

We substitute into the expansion in Equation \ref{eq:fhat expanded form} to obtain
\begin{equation}
    \begin{aligned}
        \sigma_j^2 \cdot \interv{\hat{f}}{k}_{\sigma_j}(S) 
        =& \frac{1}{n_k}\frobnorm{\left(I - \interv{\boldX_S}{k} \left(\left(\interv{\boldX_S}{k}\right)\T \interv{\boldX_S}{k} \right)^{-1} \left(\interv{\boldX_S}{k}\right)\T  \right)\interv{\bolde_j}{k}} \\
        =&  \frac{1}{n_k}\frobnorm{\left(I - U U\T\right)\interv{\bolde_j}{k}} \\
        =& \frac{1}{n_k} \TR{\left(\interv{\boldepsilon_j}{k}\right)\T \left(I - UU\T\right)\left(I - U U\T\right) \left(\interv{\bolde_j}{k}\right)} \\
        =& \frac{1}{n_k} \TR{\left(\interv{\bolde_j}{k}\right)\T \left(I - U U\T\right) \left(\interv{\bolde_j}{k}\right)} \\
        =& \frac{1}{n_k} \TR{\left(\interv{\bolde_j}{k}\right)\T \left(\interv{\bolde_j}{k}\right) - \left(U \T \left(\interv{\bolde_j}{k}\right)\right) \T U \T\left(\interv{\bolde_j}{k}\right)}.  \nonumber
    \end{aligned}
\end{equation}
Let 
\begin{equation}
    Y \coloneqq U\T \interv{\bolde_j}{k} \in \RR^p. \nonumber
\end{equation}
Then 
\begin{equation}
\begin{aligned}
    \sigma_j^2 \cdot \interv{\hat{f}}{k}_{\sigma_j}(S) &= \frac{1}{n_k} \TR{\left(\interv{\bolde_j}{k}\right)\T \left(\interv{\bolde_j}{k}\right) - Y\T Y} \\
    &= \frac{1}{n_k}\left[\sum_{i=1}^{n_k} \left(\interv{\bolde_{i,j}}{k}\right)^2 - \sum_{l=1}^p Y_l^2 \right]. \nonumber
\end{aligned}
\end{equation}
Here, $Y$ is a random vector with expectation 0 and covariance
\begin{equation}
    \begin{aligned}
        \CV(Y) &= \EE (YY\T) \\
        &= \EE\left[U\T \left(\interv{\bolde_j}{k}\right)\left(\interv{\bolde_j}{k}\right)\T U\right] \\
        &= U\T \EE \left[\left(\interv{\bolde_j}{k}\right)\left(\interv{\bolde_j}{k}\right)\T\right] U \\
        &= U\T \left(\sigma_j^2 I_{n_k}\right) U \\
        &= \sigma_j^2 U\T U \\
        &= \sigma_j^2 I_p. \nonumber
    \end{aligned}
\end{equation}
As a result, for fixed $p$, $\sum_{l=1}^p Y_l^2 = \mathcal{O}_p(1)$, and as $n_k \to \infty$ we have $\frac{1}{n_k}\sum_{l=1}^p Y_l^2 = \mathcal{O}_p(1) \overset{p}{\to} 0$. Further, since  $\interv{e_{i, j}}{k}$ are i.i.d sub-Gaussian with parameter at most $c\sigma^2$, for $t \geq 0$ we can apply the sub-Gaussian tail bound
\begin{equation}
    \PP \left( \frac{1}{n_k}\left| \sum_{i=1}^{n_k} \left(\interv{\bolde_{i,j}}{k}\right)^2 - \EE \left[\left(\interv{e_{i,j}}{k}\right)^2 \right]\right| \geq c\sigma^2 t\right) 
    \leq \, c_1 \exp\left(-c_2 n_k t^2\right). \nonumber
\end{equation}
Note that 
\begin{equation}
    \frac{1}{n_k} \EE \left[\left(\interv{e_{i,j}}{k}\right)^2\right] = \sigma_j^2 \cdot \interv{f}{k}_{\sigma_j}(S) \nonumber
\end{equation}
By setting $t = \sqrt{\frac{\log p}{n_k}}$, we therefore obtain the bound
\begin{equation}
    \begin{aligned}
    \sigma_j^2\left|\interv{\hat{f}}{k}_{\sigma_j}(S) - \interv{f}{k}_{\sigma_j}(S)\right| &\leq c\sigma^2 \sqrt{\frac{\log p}{n_k}} + \left| \frac{1}{n_k}\sum_{j=1}^p Y_j^2\right| \nonumber 
    \end{aligned}
\end{equation}
with probability at least $c_1 \exp\left(-c_2 \log p\right)$.

    We now use Lemma \ref{lemma:fhatsigmahat - fhatsigma}. We set $t'=\sqrt{\frac{\log p}{n_j}}$, which for $\delta = c' \sqrt{\frac{1}{n_j}} + c'' t'$ gives 
    \begin{equation}
        \begin{aligned}
         \sigma_j^2 \left| \interv{\hat{f}}{k}_{\hat{\sigma}_j} (S) - \interv{f}{k}_{\sigma_j} (S) \right| 
            & \leq 
            \sigma_j^2 \left(\left| \interv{\hat{f}}{k}_{\hat{\sigma}_j} (S) - \interv{\hat{f}}{k}_{\sigma_j} (S) \right| + \left| \interv{\hat{f}}{k}_{\sigma_j} (S) - \interv{f}{k}_{\sigma_j} (S) \right|\right)\\
            &\leq \sigma^2 \cdot \max \{\delta, \delta^2 \} \cdot \frac{C}{n_k} \EE \left[ \left|\left|\interv{\bolde}{k}_j\right|\right|_2^2\right] \\
            &+ c_0 \sigma^2 \sqrt{\frac{\log p}{n_k}} + \left| \frac{1}{n_k}\sum_{j=1}^p Y_j^2\right| \\
            & \leq c_0' \sigma^4 \sqrt{\frac{\log p}{n_j}} + c_1' \sigma^2 \sqrt{\frac{\log p}{n_k}} + c_2 ' \frac{p}{n_k} \\ \nonumber
        \end{aligned}
    \end{equation}
    if
    \begin{equation}
        \frac{\sigma^2}{\sigma_j^2} \cdot \delta \leq \frac{1}{2}. \nonumber
    \end{equation}
    with probability at least $1 - c_1 \exp (- c_2 \log p)$.
\end{proof}

\subsubsection{Proof of Theorem \ref{thm:estimator consistency}}
\label{suppsection:theorem estimator consistency proof}
\begin{proof}
    Suppose the gap $\interv{\xi_{\hat{\Omega}_0}}{k}$ is nonzero, which we guarantee with high probability by Theorem \ref{thm:gap nonzero and unique minimization}. Then the following inequality is valid, by Eq. \ref{eq:f final bound} and Eq. \ref{eq:sum f final bound}:
    \begin{equation}
        \begin{aligned}
            \left|\mathcal{L}_{\hat{\Omega}_0} \left(\interv{W}{k}, \interv{\boldX}{k}\right) - \EE \mathcal{L}_{\Omega_0} \left(\interv{W}{k}, \interv{X}{k}\right)\right| 
            & \leq  \sum_{j=1}^p \left| \interv{\hat{f}_{\hat{\sigma}_j}}{k} \left(\{1\dots p\} \backslash \{j\}\right) - f_{\sigma_j} \left(\{1\dots p \} \backslash \{ j\}\right) \right| \\
            & < \frac{\interv{\xi}{k}_{\hat{\Omega}_0}}{2}
        \end{aligned}
        \label{eq:loss score bound}
    \end{equation}
    for all $W \in \mathcal{D}$. Then for all $W \in \mathcal{D}$, $W \neq W_0$,
    \begin{equation}
        \begin{aligned}
            \mathcal{L}_{\hat{\Omega}_0} \left(\interv{W_0}{k}, \interv{\boldX}{k}\right) 
            &< \EE \mathcal{L}_{\Omega_0}\left(\interv{W_0}{k}, \interv{X}{k}\right) + \frac{\interv{\xi_{\hat{\Omega}_0}}{k}}{2} \\ 
            &\leq \left(\EE \mathcal{L}_{\Omega_0}\left(\interv{W}{k}, \interv{X}{k}\right) - \interv{\xi_{\Omega_0}}{k}\right) + \frac{\interv{\xi_{\hat{\Omega}_0}}{k}}{2} \\
            &< \mathcal{L}_{\hat{\Omega}_0} \left(\interv{W}{k}, \interv{\boldX}{k}\right), \nonumber
        \end{aligned}
    \end{equation}
    where the first and third inequalities come from Eq. \ref{eq:loss score bound} and the second inequality comes from the definition of the gap $\interv{\xi}{k}_{\Omega_0}$.
\end{proof}

\subsection{Two node system - \dotears{}}
We return to the two-node system described by Eq. \ref{eq:two node SEM}, and re-examine the system under the \dotears{} loss $\mathcal{L}_{\Omega_0}$, defined in Eq. \ref{suppeq:dotears loss definition}.  As before, we define the ground-truth weighted adjacency matrix $W_0 \coloneqq \begin{pmatrix}
    0 & w \\
    0 & 0
\end{pmatrix}$ and false weighted adjacency matrix $W_\delta \coloneqq \begin{pmatrix}
    0 & 0\\
    \delta & 0
\end{pmatrix}$. For simplicity, we assume we are given $\Omega_\alpha = \EE \hat{\Omega}_0 = {\alpha^2}\Omega_0$, the expected value of the estimator $\hat{\Omega}_0$. We show that $\EE \mathcal{L}_{\Omega_\alpha}\left(\interv{W_0}{k}, \interv{\boldX}{k}\right) < \EE \mathcal{L}_{\Omega_\alpha}\left(\interv{W_\delta}{k}, \interv{\boldX}{k}\right)$ for all $k=0,1,2$. 

\subsubsection{Observational system}
In the observational system, we retain the generative SEM in Eq. \ref{eq:two node SEM}:
\begin{equation}    
    \begin{aligned}
        \interv{X_1}{0} &= \interv{\epsilon_1}{0} \\
        \interv{X_2}{0} &= w \interv{X_1}{0} + \interv{\epsilon_2}{0} \\
                        &= w \interv{\epsilon_1}{0} + \interv{\epsilon_2}{0}.
    \end{aligned}
    \label{eq:two node SEM (supp)}
\end{equation}
To calculate $\EE\mathcal{L}_{\Omega_\alpha}\left(\interv{W_0}{0}, \interv{\boldX}{0}\right)$, we decompose the loss $\mathcal{L}$ component-wise by noting that 
\begin{equation}
        \EE\mathcal{L}_{\Omega_\alpha}\left(\interv{W}{k}, \interv{\boldX}{k}\right) 
        = \frac{1}{n_k}\sum_{i=0}^p \frac{\alpha^2}{\sigma_i^2} \EE\frobnorm{\left(\interv{\boldX}{k} - \interv{\boldX}{k}\interv{W}{k}\right)_i}, \nonumber
\end{equation}
to obtain
\begin{equation}
    \begin{aligned}
        \left(\interv{\boldX}{0} - \interv{\boldX}{0}W_0\right)_1 
        &= \interv{\boldX_1}{0} \\
        &= \interv{\boldepsilon_1}{0} \\
        \left(\interv{\boldX}{0} - \interv{\boldX}{0}W_0\right)_2 
        &= \interv{\boldX_2}{0} - w \interv{\boldX_1}{0} \\
        &= w \interv{\boldepsilon_1}{0} + \interv{\boldepsilon_2}{0} - w\interv{\boldepsilon_1}{0} \\
        &= \interv{\boldepsilon_2}{0} \nonumber
    \end{aligned}
\end{equation}
and therefore
\begin{equation}
    \begin{aligned}
        \frac{1}{n_0} \frac{\alpha^2}{\sigma_1^2} \EE \frobnorm{\left(\interv{\boldX}{0} - \interv{\boldX}{0}W_0\right)_1} 
        &= \frac{\alpha^2}{\sigma_1^2}\left(\frac{1}{n_0} \EE \frobnorm{\interv{\boldepsilon_1}{0}} \right)\\
        &= \left(\frac{\alpha^2}{\sigma_1^2}\right)\sigma_1^2  \\
        &=\alpha^2 \\
        \frac{1}{n_0} \frac{\alpha^2}{\sigma_2^2} \EE \frobnorm{\left(\interv{\boldX}{0} - \interv{\boldX}{0}W_0\right)_2} 
        &= \frac{\alpha^2}{\sigma_2^2}\left(\frac{1}{n_0} \EE \frobnorm{\interv{\boldepsilon_2}{0}} \right) \\
        &= \left(\frac{\alpha^2}{\sigma_2^2}\right)\sigma_2^2  \\
        &=\alpha^2 \nonumber
    \end{aligned}
\end{equation}
As a result,
\begin{equation}
    \EE \mathcal{L}_{\Omega_\alpha}\left(\interv{W_0}{0}, \interv{\boldX}{0}\right) =  2a^2. \nonumber
\end{equation}

Similarly, we calculate $\EE\mathcal{L}_{\Omega_\alpha}\left(\interv{W_\delta}{0}, \interv{\boldX}{0}\right)$ component-wise:
\begin{equation}
    \begin{aligned}
        \left(\interv{\boldX}{0} - \interv{\boldX}{0}W_\delta\right)_1 
        &= \interv{\boldX_1}{0} -\delta \interv{\boldX_2}{0} \\
        &= \interv{\boldepsilon_1}{0} - \delta\left(w \interv{\boldepsilon_1}{0} + \interv{\boldepsilon_2}{0}\right) \\
        &= \left(1 - \delta w\right) \interv{\boldepsilon_1}{0} - \delta \interv{\boldepsilon_2}{0} \\
        \left(\interv{\boldX}{0} - \interv{\boldX}{0}W_\delta\right)_2 
        &= \interv{\boldX_2}{0} \\ 
        &= w \interv{\boldepsilon_1}{0} + \interv{\boldepsilon_2}{0} \nonumber
    \end{aligned}
\end{equation}
Let $\gamma \in \RR^+$, such that $\sigma_1^2 = \gamma \sigma_2^2$. Then we calculate expected loss component-wise as
\begin{equation}
    \begin{aligned}
        \frac{1}{n_0} \frac{\alpha^2}{\sigma_1^2} \EE \frobnorm{\left(\interv{\boldX}{0} - \interv{\boldX}{0}W_\delta\right)_1} 
        &= \frac{\alpha^2}{\sigma_1^2}\left((1 - \delta w)^2 \sigma_1^2 + \delta^2\sigma_2^2  \right) \\
        &=\alpha^2 \left( (1 - \delta w)^2 + \frac{\delta^2}{\gamma}\right) \\
        \frac{1}{n_0} \frac{\alpha^2}{\sigma_2^2} \EE \frobnorm{\left(\interv{\boldX}{0} - \interv{\boldX}{0}W_\delta\right)_2} 
        &= \frac{\alpha^2}{\sigma_2^2}\left(w^2\sigma_1^2 + \sigma_2^2\right) \\
        &=\alpha^2 \left(1 + w^2 \gamma\right) \nonumber
    \end{aligned}
\end{equation}
As a result,
\begin{equation}
    \EE \mathcal{L}_{\Omega_\alpha}\left(\interv{W_\delta}{0}, \interv{\boldX}{0}\right) = \alpha^2 \left( (1 - \delta w)^2 + \frac{\delta^2}{\gamma} + 1 + w^2 \gamma\right). \nonumber
\end{equation}

We can now ask whether $\EE \mathcal{L}_{\Omega_\alpha}\left(\interv{W_\delta}{0}, \interv{\boldX}{0}\right) \geq \EE \mathcal{L}_{\Omega_\alpha}\left(\interv{W_0}{0}, \interv{\boldX}{0}\right)$ for all $w, \delta, \gamma$.
\begin{equation}
    \begin{aligned}
        \EE \mathcal{L}_{\Omega_\alpha}\left( \interv{W_\delta}{0}, \interv{\boldX}{0}\right) 
        &\overset{?}{\geq} \EE \mathcal{L}_{\Omega_\alpha}\left(\interv{W_0}{0}, \interv{\boldX}{0}\right) \\
       \alpha^2 \left( (1 - \delta w)^2 + \frac{\delta^2}{\gamma} + 1 + w^2 \gamma\right) 
        &\overset{?}{\geq} 2\alpha^2\\
        (1 - \delta w)^2 + \frac{\delta^2}{\gamma} + 1 + w^2 \gamma 
        &\overset{?}{\geq} 2 \\
        1 - 2\delta w + \delta^2 w^2 + \frac{\delta^2}{\gamma} + w^2 
        &\overset{?}{\geq} 1 \\ \nonumber
    \end{aligned}
\end{equation}
In the end, we obtain the inequality
\begin{equation}
w^2 \left(\delta^2 + \gamma\right) - w (2\delta) + \frac{\delta^2}{\gamma} \overset{?}{\geq} 0
\label{suppeq:two node final equality dotears}
\end{equation}
which is also a quadratic in $w$. We can solve for the roots of $w$:
\begin{equation}
    \begin{aligned}
        w
        &= \frac{2\delta\pm \sqrt{
        4\delta^2 - 4\left(\frac{\delta^2}{\gamma}\right)\left(\delta^2 + \gamma\right)} 
        }
        {
        2\left(\delta^2 + \gamma\right)
        } \\
        &= \frac{2\delta\pm 2\sqrt{
        \delta^2 - \frac{\delta^4}{\gamma} - \delta^2} 
        }
        {
        2\left(\delta^2 + \gamma\right)
        } \\
        &=\frac{2\delta\pm 2\sqrt{
        - \frac{\delta^4}{\gamma}} 
        }
        {
        2\left(\delta^2 + \gamma\right)
        }
        \nonumber
    \end{aligned}
\end{equation}
$\gamma \in \RR^+$ shows that $w$ has no solution in $\RR$. Moreover, the intercept term $\frac{\delta^2}{\gamma} > 0$ in Eq. \ref{suppeq:two node final equality dotears}, proving the result for all $w, \delta, \gamma$.

\subsubsection{Intervention on node 1}
We proceed similarly for the interventional cases. Upon intervention on node 1, the generative structure does not change, i.e. $\interv{W_0}{1} = \begin{pmatrix}
    0 & w \\
    0 & 0
\end{pmatrix}$, with SEM
\begin{equation}
    \begin{aligned}
        \interv{X_1}{1} &= \interv{\epsilon_1}{1} \\
        \interv{X_2}{1} &= w \interv{X_1}{1} + \interv{\epsilon_2}{1} \\
                        &= w \interv{\epsilon_1}{1} + \interv{\epsilon_2}{1}
        \nonumber
    \end{aligned}.
\end{equation}
Note that $\VV\left(\interv{\epsilon_i}{1}\right) = \frac{\sigma_1^2}{\alpha^2} = \frac{1}{\alpha^2}\VV\left(\interv{\epsilon_i}{0}\right)$, in accordance with Assumption \ref{assumption:global a assumption}. Component-wise under $\interv{W_0}{1}$ we have
\begin{equation}
    \begin{aligned}
        \left(\interv{\boldX}{1} - \interv{\boldX}{1}\interv{W_0}{1}\right)_1 
        &= \interv{\boldX_1}{1} \\
        &= \interv{\boldepsilon_1}{1} \\
        \left(\interv{\boldX}{1} - \interv{\boldX}{1}\interv{W_0}{1}\right)_2 
        &= \interv{\boldX_2}{1} - w \interv{\boldX_1}{1} \\
        &= \interv{\boldX_2}{1} - w \interv{\boldX_1}{1} \\
        &= w\interv{\boldepsilon_1}{1} + \interv{\boldepsilon_2}{1} - w\interv{\boldepsilon_1}{1}\\
        &= \interv{\boldepsilon_2}{1}
        \nonumber
    \end{aligned}
\end{equation}
which gives us the expected component-wise losses
\begin{equation}
    \begin{aligned}
        \frac{1}{n_1} \frac{\alpha^2}{\sigma_1^2} \EE \frobnorm{\left(\interv{\boldX}{1} - \interv{\boldX}{1}\interv{W_0}{1}\right)_1} 
        &= \frac{\alpha^2}{\sigma_1^2}\frac{\sigma_1^2}{\alpha^2} \\
        &= 1\\
        \frac{1}{n_1} \frac{\alpha^2}{\sigma_2^2} \EE \frobnorm{\left(\interv{\boldX}{1} - \interv{\boldX}{1}\interv{W_0}{1}\right)_2} 
        &= \frac{\alpha^2}{\sigma_2^2}\sigma_2^2 \\
        &=\alpha^2
        \nonumber
    \end{aligned}
\end{equation}
and thus
\begin{equation}
    \EE\mathcal{L}_{\Omega_\alpha}\left(\interv{W_0}{1}, \interv{\boldX}{1}\right) = 1 +\alpha^2.
        \nonumber
\end{equation}

Note that under intervention on $X_1$, $\interv{W_\delta}{1} = \begin{pmatrix}
    0 & 0 \\
    0 & 0
\end{pmatrix}$. Component-wise, we then have
\begin{equation}
    \begin{aligned}
        \left(\interv{\boldX}{1} - \interv{\boldX}{1}\interv{W_\delta}{1}\right)_1 
        &= \interv{\boldX_1}{1} \\
        &= \interv{\boldepsilon_1}{1} \\
        \left(\interv{\boldX}{1} - \interv{\boldX}{1}\interv{W_\delta}{1}\right)_2 
        &= \interv{\boldX_2}{1} \\
        &= w \interv{\boldepsilon_1}{1} + \interv{\boldepsilon_2}{1}
        \nonumber
    \end{aligned}
\end{equation}
and the expected component-wise loss
\begin{equation}
    \begin{aligned}
        \frac{1}{n_1} \frac{\alpha^2}{\sigma_1^2} \EE \frobnorm{\left(\interv{\boldX}{1} - \interv{\boldX}{1}\interv{W_\delta}{1}\right)_1} 
        &= \frac{\alpha^2}{\sigma_1^2}\frac{\sigma_1^2}{\alpha^2} \\
        &= 1\\
        \frac{1}{n_1} \frac{\alpha^2}{\sigma_2^2} \EE \frobnorm{\left(\interv{\boldX}{1} - \interv{\boldX}{1}\interv{W_\delta}{1}\right)_2}  
        &= \frac{\alpha^2}{\sigma_2^2}\left(w^2 \frac{\sigma_1^2}{\alpha^2} + \sigma_2^2 \right) \\
        &=\alpha^2\left(w^2 \frac{\sigma_1^2}{\sigma_2^2} + 1\right).
        \nonumber
    \end{aligned}
\end{equation}

Thus,
\begin{equation}
    \EE\mathcal{L}_{\Omega_\alpha}\left(\interv{W_\delta}{1}, \interv{\boldX}{1}\right) = 1 +\alpha^2 +\alpha^2 w^2 \frac{\sigma_1^2}{\sigma_2^2} > \EE\mathcal{L}_{\Omega_\alpha}\left(\interv{W_0}{1}, \interv{\boldX}{1}\right),
        \nonumber
\end{equation}
which proves the result.

\subsubsection{Intervention on node 2}
Under intervention on $X_2$, $\interv{W_0}{2} = \begin{pmatrix}
    0 & 0 \\
    0 & 0
\end{pmatrix}$, with corresponding SEM
\begin{equation}
    \begin{aligned}
        \interv{X_1}{2} &= \interv{\epsilon_1}{2} \\
        \interv{X_2}{2} &= \interv{\epsilon_2}{2} \\
        \nonumber
    \end{aligned}.
\end{equation}
Component-wise, we obtain the terms
\begin{equation}
    \begin{aligned}
        \left(\interv{\boldX}{2} - \interv{\boldX}{2}\interv{W_0}{2}\right)_1 
        &= \interv{\boldX_1}{2} \\
        &= \interv{\boldepsilon_1}{2} \\
        \left(\interv{\boldX}{2} - \interv{\boldX}{2}\interv{W_0}{2}\right)_2
        &= \interv{\boldX_2}{2} \\
        &= \interv{\boldepsilon_2}{2} \\ 
        \nonumber
    \end{aligned}
\end{equation}

Note that $\VV\left(\interv{\epsilon_2}{2}\right) = \frac{\sigma_2}{\alpha^2}$ in accordance with Assumption \ref{assumption:global a assumption}. Then the expected loss component-wise is
\begin{equation}
    \begin{aligned}
        \frac{1}{n_2}\frac{\alpha^2}{\sigma_1^2} \EE \frobnorm{\left(\interv{\boldX}{2} - \interv{\boldX}{2}\interv{W_0}{2}\right)_1} 
        &= \frac{\alpha^2}{\sigma_1^2} \sigma_1^2 \\
        &= \alpha^2 \\
        \frac{1}{n_2}\frac{\alpha^2}{\sigma_2^2} \EE \frobnorm{\left(\interv{\boldX}{2} - \interv{\boldX}{2}\interv{W_0}{2}\right)_2} 
        &= \frac{\alpha^2}{\sigma_2^2} \frac{\sigma_2^2}{\alpha^2} \\
        &= 1,
        \nonumber
    \end{aligned}
\end{equation}
giving 
\begin{equation}
    \EE \mathcal{L}_{\Omega_\alpha}\left(\interv{W_0}{2},\interv{\boldX}{2}\right) = 1 +\alpha^2.
        \nonumber
\end{equation}

Under $\interv{W_\delta}{2} = \begin{pmatrix}
    0 & 0\\
    \delta & 0
\end{pmatrix}$, we obtain the terms component-wise
\begin{equation}
    \begin{aligned}
        \left(\interv{\boldX}{2} - \interv{\boldX}{2}\interv{W_\delta}{2}\right)_1 
        &= \interv{\boldX_1}{2} - \delta \interv{\boldX_2}{2} \\
        &= \interv{\boldepsilon_1}{2} - \delta \interv{\boldepsilon_2}{2}\\
        \left(\interv{\boldX}{2} - \interv{\boldX}{2}\interv{W_\delta}{2}\right)_2
        &= \interv{\boldX_2}{2} \\
        &= \interv{\boldepsilon_2}{2} \\
        \nonumber
    \end{aligned}
\end{equation}
which in expectation gives the component-wise losses
\begin{equation}
    \begin{aligned}
        \frac{1}{n_2} \frac{\alpha^2}{\sigma_1^2} \EE \frobnorm{\left(\interv{\boldX}{2} - \interv{\boldX}{2}\interv{W_\delta}{2}\right)_1} 
        &= \frac{\alpha^2}{\sigma_1^2}\left( \sigma_1^2 + \delta^2 \frac{\sigma_2^2}{\alpha^2} \right)\\
        &= \alpha^2 + \delta^2 \frac{\sigma_2^2}{\sigma_1^2}\\
        \frac{1}{n_2} \frac{\alpha^2}{\sigma_2^2} \EE \frobnorm{\left(\interv{\boldX}{2} - \interv{\boldX}{2}\interv{W_\delta}{2}\right)_2}  
        &= \frac{\alpha^2}{\sigma_2^2} \frac{\sigma_2^2}{\alpha^2},
        \nonumber
    \end{aligned}
\end{equation}
which gives
\begin{equation}
    \EE \mathcal{L}_{\Omega_\alpha}\left(\interv{W_\delta}{2}, \interv{\boldX}{2}\right) = 1 +\alpha^2 + \delta^2 \frac{\sigma_2^2}{\sigma_1^2}
        \nonumber
\end{equation}
which is trivially greater than $\EE \mathcal{L}_{\Omega_\alpha}\left(\interv{W_\delta}{2}, \interv{\boldX}{2}\right)$ for all $\alpha, \delta, \sigma_1^2, \sigma_2^2$.

\subsection{Full two node simulation details}
\label{suppsection:full two node sims}
 In Section \ref{section:small p simulations}, we simulate from the system of SEMs
\begin{equation}
    \begin{aligned}
        \interv{X}{0} &= \interv{X}{0}W + \interv{\epsilon}{0} \\
        \interv{X}{1} &= \interv{X}{1}\interv{W}{1} + \interv{\epsilon}{1} \\
        \interv{X}{2} &= \interv{X}{2}\interv{W}{2} + \interv{\epsilon}{2}, \\
        \nonumber
    \end{aligned}
\end{equation}
such that in the observational system
\begin{equation}
    \begin{aligned}
        \interv{X_1}{0} &= \interv{\epsilon_1}{1} && \interv{\epsilon_1}{0} \sim \Nor\left(0, \gamma\right) \\
        \interv{X_2}{0} &= w \interv{X_1}{0} + \interv{\epsilon_2}{0} && \interv{\epsilon_2}{0} \sim \Nor\left(0, 1\right), \\
    \end{aligned}
    \label{eq:small sim generative model obs}
\end{equation}
under intervention on node 1 we have
\begin{equation}
    \begin{aligned}
        \interv{X_1}{1} &= \interv{\epsilon_1}{1} && \interv{\epsilon_1}{1} \sim \Nor\left(0, \frac{\gamma}{\alpha^2} \right) \\
        \interv{X_2}{1} &= w \interv{X_1}{1} + \interv{\epsilon_2}{1} && \interv{\epsilon_2}{1} \sim \Nor\left(0, 1\right), \\
    \end{aligned}
    \label{eq:small sim generative model ko 1}
\end{equation}
and under intervention on node 2
\begin{equation}
    \begin{aligned}
        \interv{X_1}{2} &= \interv{\epsilon_1}{2} && \interv{\epsilon_1}{2} \sim \Nor\left(0, \gamma\right) \\
        \interv{X_2}{2} &= \interv{\epsilon_2}{2} && \interv{\epsilon_2}{2} \sim \Nor\left(0, \frac{1}{\alpha^2}\right). \\
    \end{aligned}
    \label{eq:small sim generative model ko 2}
\end{equation}
For observational data, we simulate only from the SEM given in Eq. \ref{eq:small sim generative model obs}, with a sample size $n=3000$. For interventional data, we draw a sample size of $n_k = 1000$ from each intervention $k$. This gives a total sample size of $n=(p + 1) \times 3000$ for interventional data, which matches the sample size of observational data. We set $\alpha = 4$. In Section \ref{section:small p simulations}, we presented empirical results on two node simulations for $(w, \gamma) \in \{0.3, 0.5, 0.7, 0.8, 1.0\} \times \{1, 2, 100\}$. Figure \ref{suppfig:full two node sims} shows the $\ell_1$ distance between the ground truth DAG $W_0$ and the inferred DAG for each method, for all $(w, \gamma) \in \{0.1, 0.2, \dots ,1.5\} \times \{1, 2, 4, 10, 100\}$. For each parameter combination of $(w, \gamma)$, we draw 25 instances of simulated data for both the interventional and strictly observational case, at the described sample size. 

For \dotears{}, sortnregress, and NO TEARS, we set the regularization parameter $\lambda$ to 0, to isolate the performance of the loss function.  In GOLEM-EV and GOLEM-NV we set the $\ell_1$ regularization parameter $\lambda_1$ to 0, but let $\lambda_2 = 5$ to enforce DAG-ness, as recommended by the authors. 

See Figure \ref{suppfig:full two node sims} for the full set of simulated $w, \gamma$. 

\begin{figure*}
    \centering
    \includegraphics[width=\textwidth]{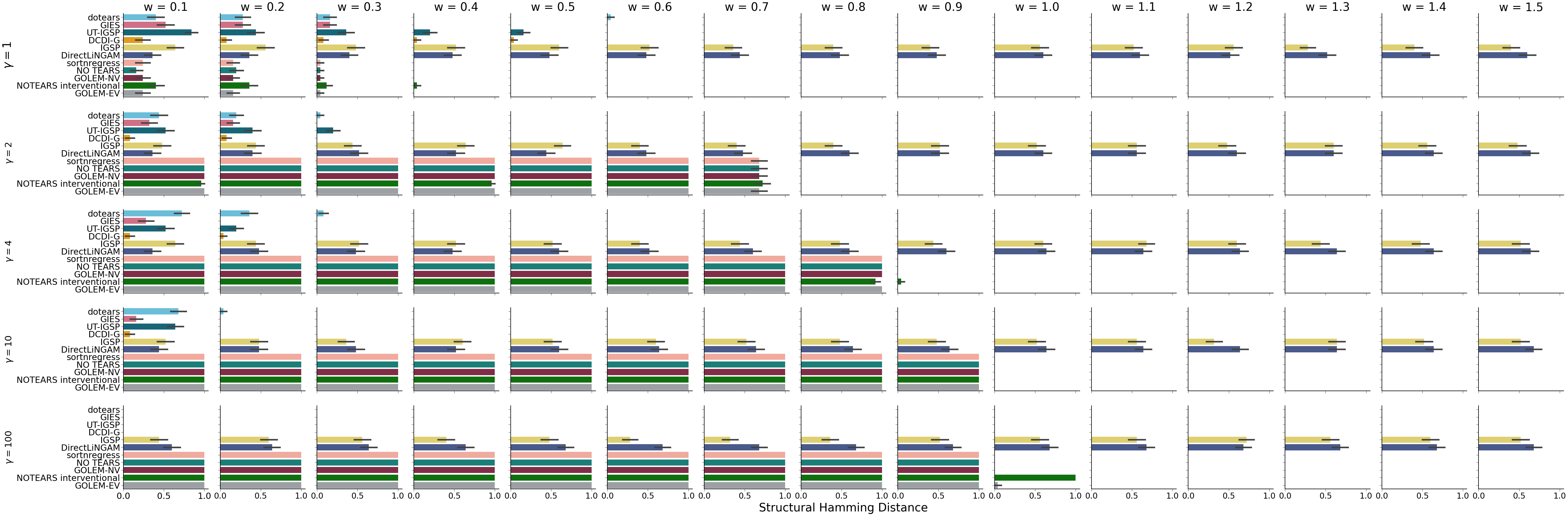}
    \caption{Full simulations for the two-node DAG for all $(w, \gamma) \in \{0.1, 0.2, \dots ,1.5\} \times \{1, 2, 4, 10, 100\}$, where $X_1 \overset{w}{\rightarrow} X_2$ and $\sigma_1^2 = \gamma\sigma_2^2$. Methods are compared using SHD (lower is better).}
    \label{suppfig:full two node sims}
\end{figure*}

\subsection{Full three node simulations}
\label{suppsection:full three node sims}
Simulations in the two-node DAG showed that NO TEARS and GOLEM-NV are deterministic functions of the true structure $W$ and the true exogenous variance structure $\Omega_0$ (Sections \ref{section:small p simulations}, \ref{suppsection:full two node sims}). Simulations in more complex three-node topologies verify determinism of NO TEARS and GOLEM-NV in $W$ and $\Omega_0$, but show they are distinct deterministic functions. 

We simulate observational and interventional data under the three node chain $X_1 \overset{w}{\rightarrow} X_2 \overset{w}{\rightarrow} X_3$, the three node collider $X_1 \overset{w}{\rightarrow} X_3 \overset{w}{\leftarrow} X_2$, and the three node fork $X_2 \overset{w}{\leftarrow} X_1 \overset{w}{\rightarrow} X_3$. In each topology, we let $\Omega_0 \coloneqq \begin{pmatrix}
    \gamma & 0 & 0 \\
    0 & 1 & 0 \\
    0 & 0 & 1
\end{pmatrix}$, so that $\gamma$ is the exogenous variance of a source node. For simplicity, we constrain the edge weights to be equal, and simulate data under Gaussian exogenous variance for $(w, \gamma) \in \{0.1, \dots, 1.5 \} \times \{1, 2, 4, 10, 100\}$. For interventional data with interventions $i\in \{0, 1, 2, 3\}$, we draw $n_i = 1000$ observations in all simulations, giving us a total sample size of $n = 4000$ that is matched in observational data. As in section \ref{section:small p simulations}, we set $\alpha = 4$, and remove $\ell_1$ regularization where appropriate. We benchmark \dotears{}, NO TEARS, sortnregress, GOLEM-EV, GOLEM-NV, DirectLingam, GIES, IGSP, UT-IGSP, and DCDI-G \cite{hauser2012characterization, zheng2018dags, reisach2021beware, ng2020role, yang2018characterizing, wang2017permutation, squires2020permutation, brouillard2020differentiable, shimizu2011directlingam}. 

We evaluate each method by the SHD between the ground truth DAG $W$ and the method inferred DAG. In the two node case, the least squares loss performed identically to the likelihood loss of GOLEM-NV; here, their performances diverge, but remain essentially deterministic functions in $w$ and $\gamma$.

\subsubsection{Chain}
In the chain, we have the true structure $W_0 = \begin{pmatrix}
    0 & w & 0 \\
    0 & 0 & w \\
    0 & 0 & 0
\end{pmatrix}$. We simulate under the system of SEMs
\begin{equation}
    \begin{aligned}
        \interv{X}{0} &= \interv{X}{0}W_0 + \interv{\epsilon}{0} \\
        \interv{X}{1} &= \interv{X}{1}\interv{W_0}{1} + \interv{\epsilon}{1} \\
        \interv{X}{2} &= \interv{X}{2}\interv{W_0}{2} + \interv{\epsilon}{2}, \\
        \interv{X}{3} &= \interv{X}{3}\interv{W_0}{3} + \interv{\epsilon}{3},
        \nonumber
    \end{aligned}
\end{equation}
such that in the observational system
\begin{equation}
    \begin{aligned}
        \interv{X_1}{0} &= \interv{\epsilon_1}{1} && \interv{\epsilon_1}{0} \sim \Nor\left(0, \gamma\right) \\
        \interv{X_2}{0} &= w \interv{X_1}{0} + \interv{\epsilon_2}{0} && \interv{\epsilon_2}{0} \sim \Nor\left(0, 1\right), \\
        \interv{X_3}{0} &= w \interv{X_2}{0} + \interv{\epsilon_3}{0} && \interv{\epsilon_3}{0} \sim \Nor\left(0, 1\right),
        \nonumber
    \end{aligned}
\end{equation}
under intervention on node 1 we have
\begin{equation}
    \begin{aligned}
        \interv{X_1}{1} &= \interv{\epsilon_1}{1} && \interv{\epsilon_1}{1} \sim \Nor\left(0, \frac{\gamma}{\alpha^2} \right) \\
        \interv{X_2}{1} &= w \interv{X_1}{1} + \interv{\epsilon_2}{1} && \interv{\epsilon_2}{1} \sim \Nor\left(0, 1\right), \\
        \interv{X_3}{1} &= w \interv{X_2}{1} + \interv{\epsilon_3}{1} && \interv{\epsilon_3}{1} \sim \Nor\left(0, 1\right), \\
        \nonumber
    \end{aligned}
\end{equation}
under intervention on node 2
\begin{equation}
    \begin{aligned}
        \interv{X_1}{2} &= \interv{\epsilon_1}{2} && \interv{\epsilon_1}{2} \sim \Nor\left(0, \gamma\right) \\
        \interv{X_2}{2} &= \interv{\epsilon_2}{2} && \interv{\epsilon_2}{2} \sim \Nor\left(0, \frac{1}{\alpha^2}\right), \\
        \interv{X_3}{2} &= w \interv{X_2}{2} + \interv{\epsilon_3}{2} && \interv{\epsilon_3}{2} \sim \Nor\left(0, 1\right), \\
        \nonumber
    \end{aligned}
\end{equation}
and under intervention on node 3
\begin{equation}
    \begin{aligned}
        \interv{X_1}{3} &= \interv{\epsilon_1}{3} && \interv{\epsilon_1}{3} \sim \Nor\left(0, \gamma\right) \\
        \interv{X_2}{3} &= w \interv{X_1}{3} + \interv{\epsilon_2}{3} && \interv{\epsilon_2}{3} \sim \Nor\left(0, 1 \right), \\
        \interv{X_3}{1} &= \interv{\epsilon_3}{3} && \interv{\epsilon_3}{3} \sim \Nor\left(0, \frac{1}{\alpha^2} \right). \\
        \nonumber
    \end{aligned}
\end{equation}

Figure \ref{suppfig:full three node chain sims} shows the full results. 

\begin{figure*}
    \centering
    \begin{subfigure}[b]{0.9\textwidth}
       \includegraphics[width=1\linewidth]{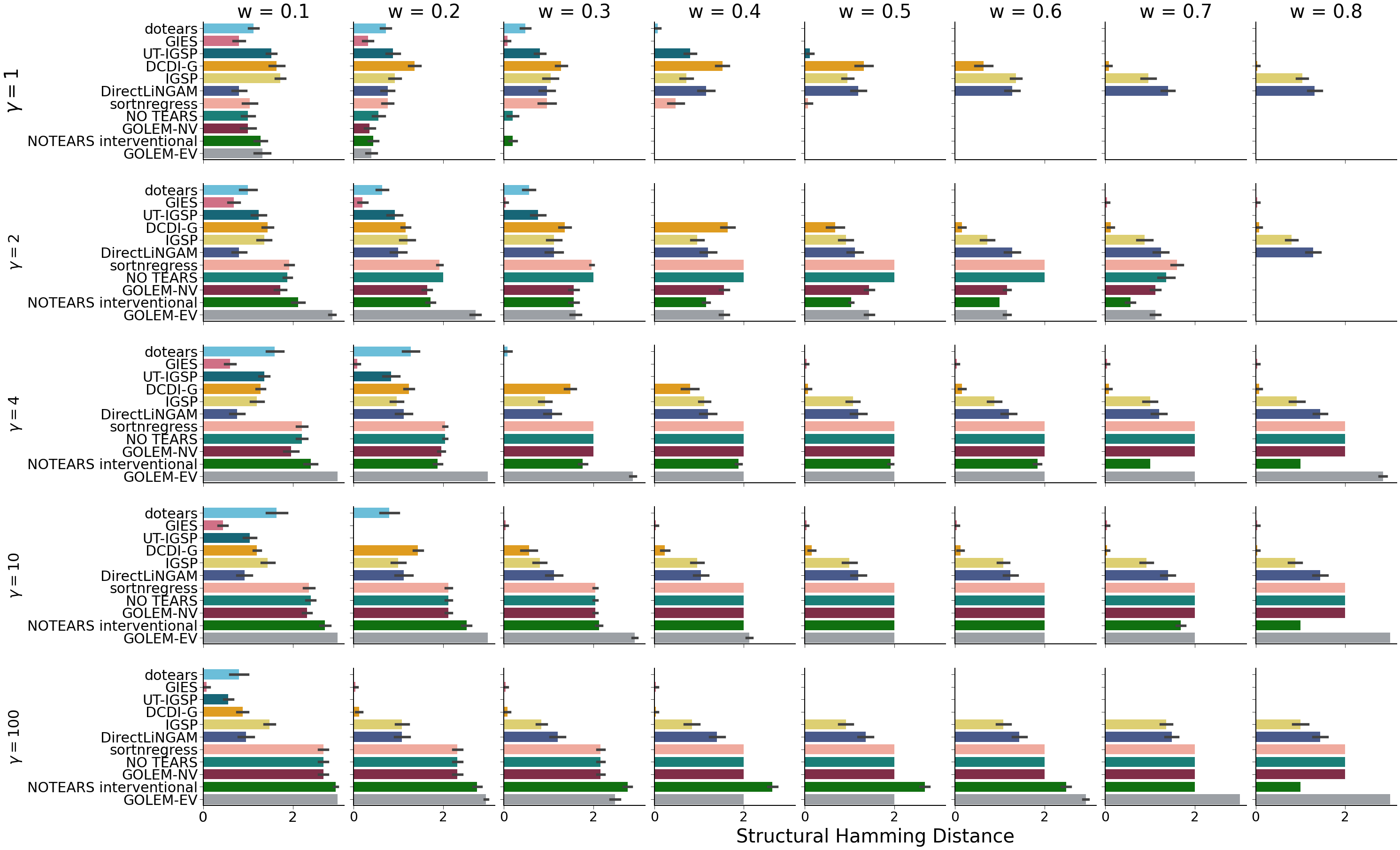}
       \caption{Results for $w$ in $0.1, \dots, 0.8$, inclusive.}
    \end{subfigure}
    
    \begin{subfigure}[b]{0.9\textwidth}
       \includegraphics[width=1\linewidth]{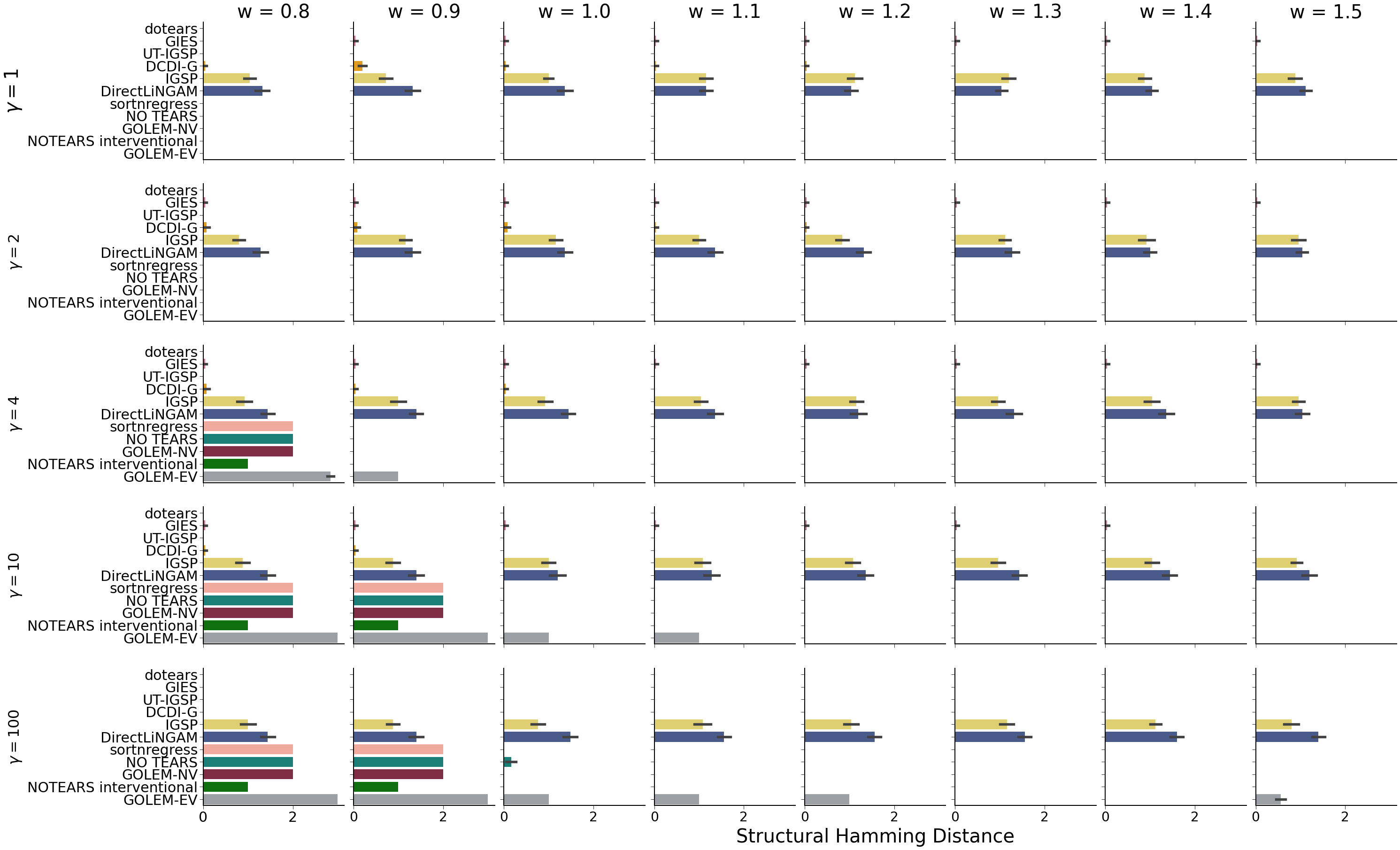}
       \caption{Results for $w$ in $0.8, \dots, 1.5$, inclusive.}
    \end{subfigure}
    
    \caption{Full simulations for the chain topology $X_1 \overset{w}{\rightarrow} X_2 \overset{w}{\rightarrow} X_3$, where $\Omega_0 = \begin{pmatrix}
        \gamma & 0 & 0 \\
        0 & 1 & 0 \\
        0 & 0 & 1 \\
    \end{pmatrix}$. We iterate across all $(w, \gamma) \in \{0.1, 0.2, \dots ,1.5\} \times \{1, 2, 4, 10, 100\}$. Methods are compared using SHD (lower is better). (a) shows results for $w$ in $0.1, \dots, 0.8$, inclusive, while (b) shows results for $w$ in $0.8, \dots, 1.5$, inclusive. $w = 0.8$ is repeated across subpanels.}
    \label{suppfig:full three node chain sims}
\end{figure*}

\subsubsection{Collider}
In the collider, we have the true structure $W_0 = \begin{pmatrix}
    0 & 0 & w \\
    0 & 0 & w \\
    0 & 0 & 0
\end{pmatrix}$. We simulate under the system of SEMs
\begin{equation}
    \begin{aligned}
        \interv{X}{0} &= \interv{X}{0}W_0 + \interv{\epsilon}{0} \\
        \interv{X}{1} &= \interv{X}{1}\interv{W_0}{1} + \interv{\epsilon}{1} \\
        \interv{X}{2} &= \interv{X}{2}\interv{W_0}{2} + \interv{\epsilon}{2}, \\
        \interv{X}{3} &= \interv{X}{3}\interv{W_0}{3} + \interv{\epsilon}{3},
        \nonumber
    \end{aligned}
\end{equation}
such that in the observational system
\begin{equation}
    \begin{aligned}
        \interv{X_1}{0} &= \interv{\epsilon_1}{1} && \interv{\epsilon_1}{0} \sim \Nor\left(0, \gamma\right) \\
        \interv{X_2}{0} &= \interv{\epsilon_2}{0} && \interv{\epsilon_2}{0} \sim \Nor\left(0, 1\right), \\
        \interv{X_3}{0} &= w\interv{X_1}{0} + w\interv{X_2}{0} + \interv{\epsilon_3}{0} && \interv{\epsilon_3}{0} \sim \Nor\left(0, 1\right),
        \nonumber
    \end{aligned}
\end{equation}
under intervention on node 1 we have
\begin{equation}
    \begin{aligned}
        \interv{X_1}{1} &= \interv{\epsilon_1}{1} && \interv{\epsilon_1}{1} \sim \Nor\left(0, \frac{\gamma}{\alpha^2} \right) \\
        \interv{X_2}{1} &= \interv{\epsilon_2}{1} && \interv{\epsilon_2}{1} \sim \Nor\left(0, 1\right), \\
        \interv{X_3}{1} &= w\interv{X_1}{1} + w\interv{X_2}{1} + \interv{\epsilon_3}{1} && \interv{\epsilon_3}{1} \sim \Nor\left(0, 1\right), \\
        \nonumber
    \end{aligned}
\end{equation}
under intervention on node 2
\begin{equation}
    \begin{aligned}
        \interv{X_1}{2} &= \interv{\epsilon_1}{2} && \interv{\epsilon_1}{2} \sim \Nor\left(0, \gamma\right) \\
        \interv{X_2}{2} &= \interv{\epsilon_2}{2} && \interv{\epsilon_2}{2} \sim \Nor\left(0, \frac{1}{\alpha^2}\right), \\
        \interv{X_3}{2} &= w\interv{X_1}{2} + w \interv{X_2}{2} + \interv{\epsilon_3}{2} && \interv{\epsilon_3}{2} \sim \Nor\left(0, 1\right), \\
        \nonumber
    \end{aligned}
\end{equation}
and under intervention on node 3
\begin{equation}
    \begin{aligned}
        \interv{X_1}{3} &= \interv{\epsilon_1}{3} && \interv{\epsilon_1}{3} \sim \Nor\left(0, \gamma\right) \\
        \interv{X_2}{3} &= \interv{\epsilon_2}{3} && \interv{\epsilon_2}{3} \sim \Nor\left(0, 1 \right), \\
        \interv{X_3}{1} &= \interv{\epsilon_3}{3} && \interv{\epsilon_3}{3} \sim \Nor\left(0, \frac{1}{\alpha^2} \right). \\
        \nonumber
    \end{aligned}
\end{equation}
\begin{figure*}
    \centering
    \begin{subfigure}[b]{0.9\textwidth}
       \includegraphics[width=1\linewidth]{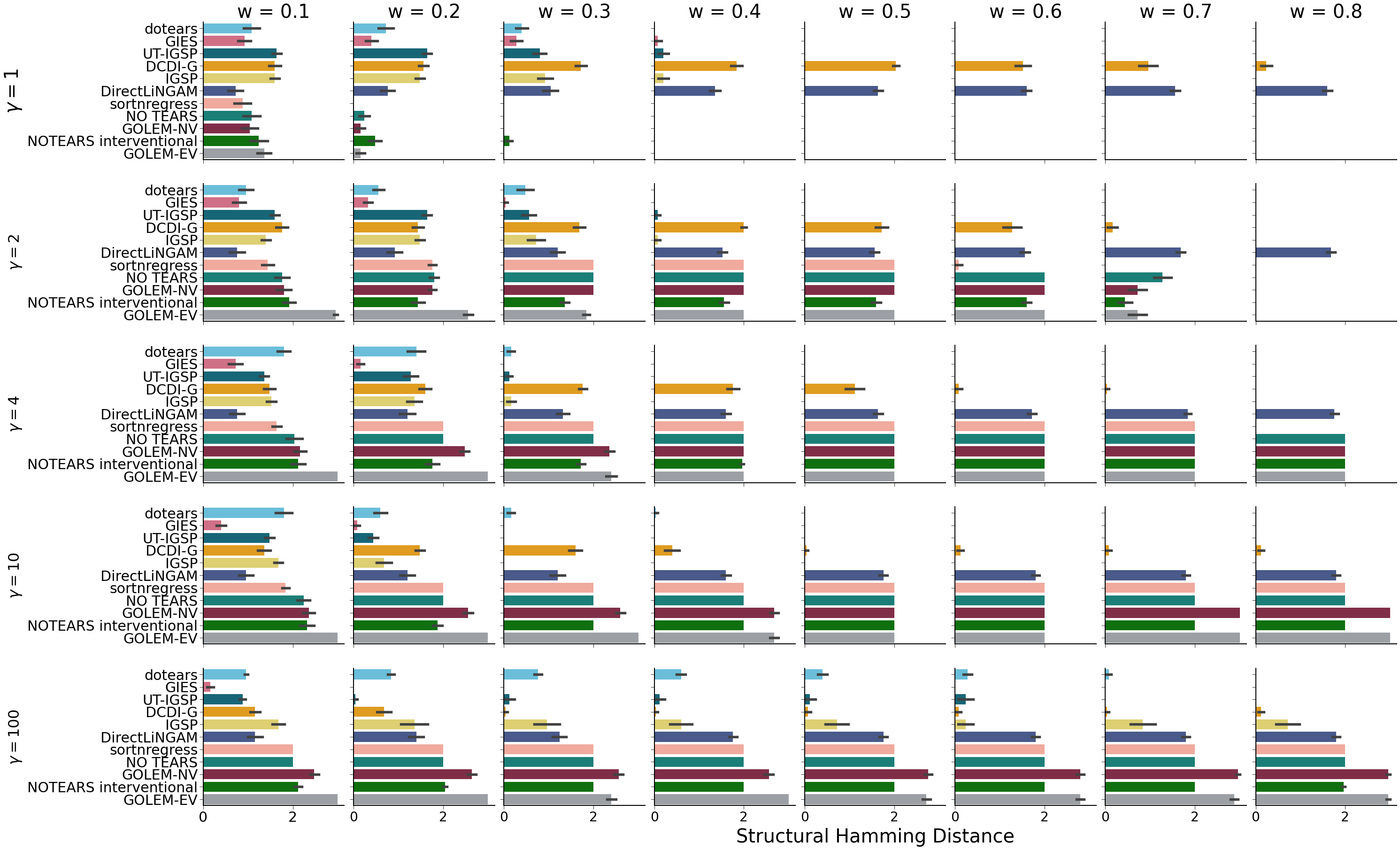}
       \caption{Results for $w$ in $0.1, \dots, 0.8$, inclusive.}
    \end{subfigure}
    
    \begin{subfigure}[b]{0.9\textwidth}
       \includegraphics[width=1\linewidth]{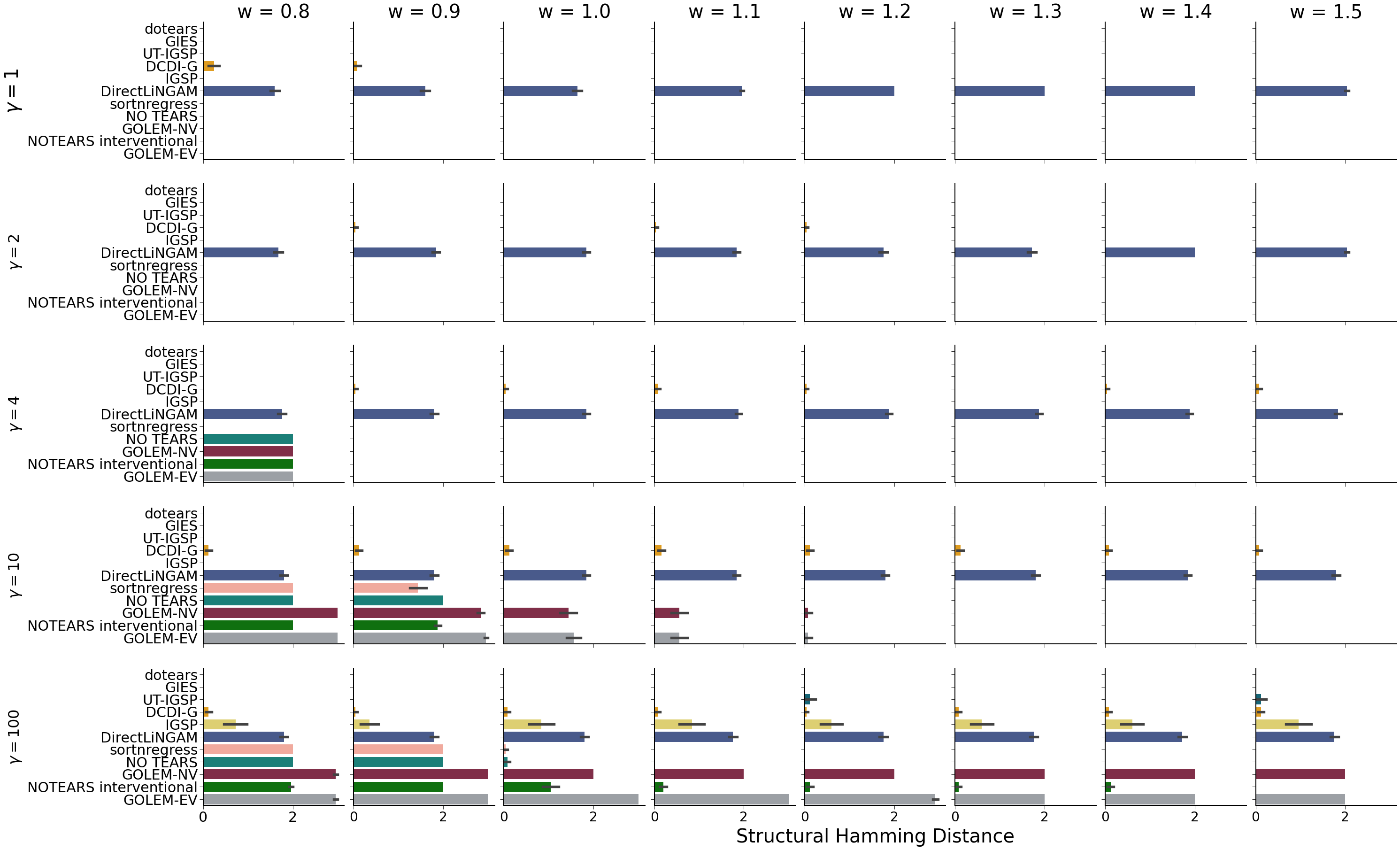}
       \caption{Results for $w$ in $0.8, \dots, 1.5$, inclusive.}
    \end{subfigure}
    \caption{Full simulations for the collider topology $X_1 \overset{w}{\rightarrow} X_3 \overset{w}{\leftarrow} X_2$, where $\Omega_0 = \begin{pmatrix}
        \gamma & 0 & 0 \\
        0 & 1 & 0 \\
        0 & 0 & 1 \\
    \end{pmatrix}$. We iterate across all $(w, \gamma) \in \{0.1, 0.2, \dots ,1.5\} \times \{1, 2, 4, 10, 100\}$. Methods are compared using SHD (lower is better). (a) shows results for $w$ in $0.1, \dots, 0.8$, inclusive, while (b) shows results for $w$ in $0.8, \dots, 1.5$, inclusive. $w = 0.8$ is repeated across subpanels.}
    \label{suppfig:full three node collider sims}
\end{figure*}

\subsubsection{Fork}
In the collider, we have the true structure $W_0 = \begin{pmatrix}
    0 & w & w \\
    0 & 0 & 0 \\
    0 & 0 & 0
\end{pmatrix}$. We simulate under the system of SEMs
\begin{equation}
    \begin{aligned}
        \interv{X}{0} &= \interv{X}{0}W_0 + \interv{\epsilon}{0} \\
        \interv{X}{1} &= \interv{X}{1}\interv{W_0}{1} + \interv{\epsilon}{1} \\
        \interv{X}{2} &= \interv{X}{2}\interv{W_0}{2} + \interv{\epsilon}{2}, \\
        \interv{X}{3} &= \interv{X}{3}\interv{W_0}{3} + \interv{\epsilon}{3},
        \nonumber
    \end{aligned}
\end{equation}
such that in the observational system
\begin{equation}
    \begin{aligned}
        \interv{X_1}{0} &= \interv{\epsilon_1}{1} && \interv{\epsilon_1}{0} \sim \Nor\left(0, \gamma\right) \\
        \interv{X_2}{0} &= w\interv{X_1}{0} + \interv{\epsilon_2}{0} && \interv{\epsilon_2}{0} \sim \Nor\left(0, 1\right), \\
        \interv{X_3}{0} &= w\interv{X_1}{0} + \interv{\epsilon_3}{0} && \interv{\epsilon_3}{0} \sim \Nor\left(0, 1\right),
        \nonumber
    \end{aligned}
\end{equation}
under intervention on node 1 we have
\begin{equation}
    \begin{aligned}
        \interv{X_1}{1} &= \interv{\epsilon_1}{1} && \interv{\epsilon_1}{1} \sim \Nor\left(0, \frac{\gamma}{\alpha^2} \right) \\
        \interv{X_2}{1} &= w\interv{X_1}{0} + \interv{\epsilon_2}{1} && \interv{\epsilon_2}{1} \sim \Nor\left(0, 1\right), \\
        \interv{X_3}{1} &= w\interv{X_1}{1} + \interv{\epsilon_3}{1} && \interv{\epsilon_3}{1} \sim \Nor\left(0, 1\right), \\
        \nonumber
    \end{aligned}
\end{equation}
under intervention on node 2
\begin{equation}
    \begin{aligned}
        \interv{X_1}{2} &= \interv{\epsilon_1}{2} && \interv{\epsilon_1}{2} \sim \Nor\left(0, \gamma\right) \\
        \interv{X_2}{2} &= \interv{\epsilon_2}{2} && \interv{\epsilon_2}{2} \sim \Nor\left(0, \frac{1}{\alpha^2}\right), \\
        \interv{X_3}{2} &= w\interv{X_1}{2} + \interv{\epsilon_3}{2} && \interv{\epsilon_3}{2} \sim \Nor\left(0, 1\right), \\
        \nonumber
    \end{aligned}
\end{equation}
and under intervention on node 3
\begin{equation}
    \begin{aligned}
        \interv{X_1}{3} &= \interv{\epsilon_1}{3} && \interv{\epsilon_1}{3} \sim \Nor\left(0, \gamma\right) \\
        \interv{X_2}{3} &= w\interv{X_1}{3} + \interv{\epsilon_2}{3} && \interv{\epsilon_2}{3} \sim \Nor\left(0, 1 \right), \\
        \interv{X_3}{1} &= \interv{\epsilon_3}{3} && \interv{\epsilon_3}{3} \sim \Nor\left(0, \frac{1}{\alpha^2} \right). \\
        \nonumber
    \end{aligned}
\end{equation}

\begin{figure*}
    \centering
    \begin{subfigure}[b]{0.9\textwidth}
       \includegraphics[width=1\linewidth]{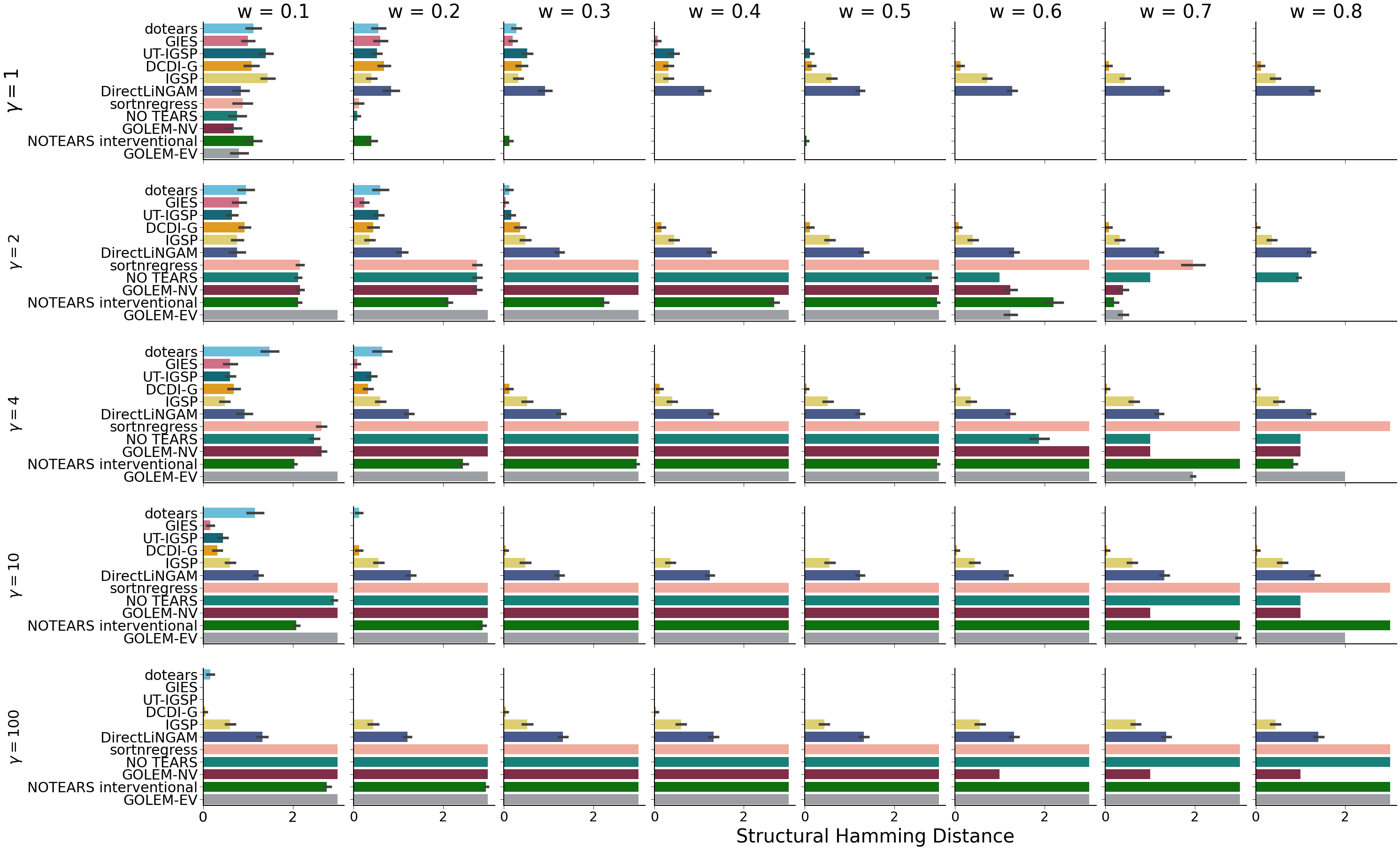}
       \caption{Results for $w$ in $0.1, \dots, 0.8$, inclusive.}
    \end{subfigure}
    
    \begin{subfigure}[b]{0.9\textwidth}
       \includegraphics[width=1\linewidth]{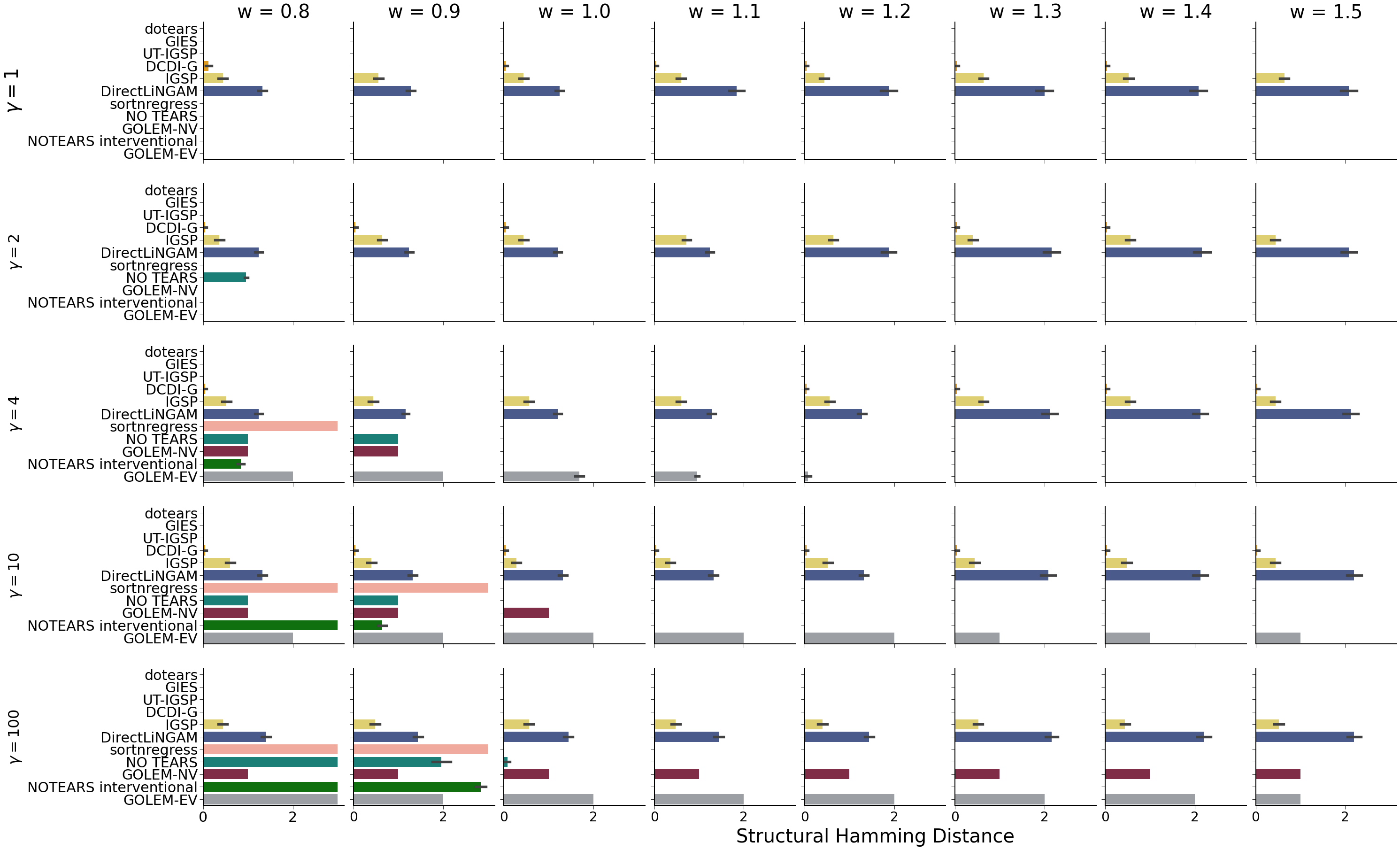}
       \caption{Results for $w$ in $0.8, \dots, 1.5$, inclusive.}
    \end{subfigure}
    \caption{Full simulations for the fork topology $X_2 \overset{w}{\leftarrow} X_1 \overset{w}{\rightarrow} X_3$, where $\Omega_0 = \begin{pmatrix}
        \gamma & 0 & 0 \\
        0 & 1 & 0 \\
        0 & 0 & 1 \\
    \end{pmatrix}$. We iterate across all $(w, \gamma) \in \{0.1, 0.2, \dots ,1.5\} \times \{1, 2, 4, 10, 100\}$. Methods are compared using SHD (lower is better). (a) shows results for $w$ in $0.1, \dots, 0.8$, inclusive, while (b) shows results for $w$ in $0.8, \dots, 1.5$, inclusive. $w = 0.8$ is repeated across subpanels.}
    \label{suppfig:full three node fork sims}
\end{figure*}

\subsection{Large random graph simulations}
\label{suppsection:large random graph simulations}

\subsubsection{Data generation}
\label{suppsection:large p data generation}
For evaluation, data is generated for DAGs with $p = 40$ nodes. Structures are drawn from both \erdosrenyi{}{} (ER) and Scale-Free (SF) DAGs \cite{erdHos1960evolution, barabasi1999emergence}. We parameterize ER graphs as ER-$r$ and SF graphs as SF-$z$, where $r\in [0, 1]$ represents the probability of assignment of each individual edge, whereas $z\in \mathbb{Z}$ is the integer number of edges assigned per node.

We simulate under two parameterizations $\{\text{Low Density, High Density}\}$ of the edge densities $(r, z)$. In the ``Low Density'' parameterization, we give $(r, z) = (0.1, 2)$. To evaluate performance on higher density topologies, we also give ``High Density'' parameterizations, where $(r, z) = (0.2, 4)$. 

Given an edge density scenario, we simulate under two parameterizations of the edge weights. In the ``Strong Effects'' parameterization, $w \sim \Unif\left([-2.0, -0.5] \cup [0.5, 2.0]\right)$. We also give the ``Weak Effects'' parameterization, which edge weights are drawn from $w \sim \Unif\left([-1.0, 0.3] \cup [0.3, 1.0]\right)$ such that $|w| \leq 1$ is guaranteed. Here, there is no guarantee any two nodes will be varsortable (although they may be in practice, as a function of $\Omega_0$) \cite{reisach2021beware, kaiser2022unsuitability}. Table \ref{tab:simulation parameters} summarizes all four possible simulation parameterizations.

For each node $i$, we draw $\sigma_i \sim \Unif\left([0.5, 2.0]\right)$, and draw $n_{0}$ observations from $\interv{\boldepsilon_i}{0} \sim \Nor\left(0, \sigma_i^2\right)$. For each DAG we generate an instance of observational data, where $n_{0} = (p + 1) * n_k = 4100$, and an instance of interventional data, where $n_{k} = 100$ for all $k=0\dots p$ to match sample size. We set the distribution of $\interv{\epsilon_i}{k}$ according to Assumptions \ref{assumption:global a assumption} and \ref{assumption:variance without intervention} for $\alpha=4$.

\begin{table*}[ht!]
    \centering
    \begin{tabular}{r@{}l||r@{}l|c}
        \toprule
        \multicolumn{2}{c||}{\textbf{Simulation Type}} & \multicolumn{2}{c|}{\textbf{Topology}} & $w$\tabularnewline
        \hline
        Low Density&, Strong Effects & ER-0.1\,\,&OR SF-2 & $\Unif\left([-2.0, -0.5] \cup [0.5, 2.0]\right)$  \tabularnewline  
        Low Density&, Weak Effects & ER-0.1\,\,&OR SF-2 & $\Unif\left([-1.0, -0.3] \cup [0.3, 1.0]\right)$ \tabularnewline 
        High Density&, Strong Effects & ER-0.2\,\,&OR SF-4 & $\Unif\left([-2.0, -0.5] \cup [0.5, 2.0]\right)$ \tabularnewline 
        High Density&, Weak Effects & ER-0.2\,\,&OR SF-4 & $\Unif\left([-1.0, -0.3] \cup [0.3, 1.0]\right)$  \tabularnewline
    \end{tabular}
    \caption{For each simulation type for $p=40$ DAG simulations, we provide the parameterizations of the \erdosrenyi{} and Scale-Free topologies and the distribution on the drawn edge weights.}
    \label{tab:simulation parameters}
\end{table*}

For \dotears{}, NO TEARS, sortnregress, GOLEM-EV, GOLEM-NV, and DCDI-G 5-fold cross-validation was performed to select the regularization parameters. For each drawn DAG, a separate data instance of interventional and observational data was re-drawn from the same distribution specifically for cross-validation. After choosing a $\lambda$ (or for GOLEM-EV and GOLEM-NV, the set $(\lambda_1, \lambda_2)$) from the data for cross-validation, the methods were evaluated on the original simulated data. For \dotears{}, NO TEARS, sortnregress, and DCDI-G, 5-fold cross-validation was performed across the grid $\lambda \in \{.001, .01, .1, 1, 10, 100\}$. For GOLEM-EV and GOLEM-NV, 5-fold cross-validation was performed across the grid $\lambda_1 \times \lambda_2 \in \{.001, .01, .1, 1, 10, 100\} \times \{.05, .5, 5, 50\}$. We threshold at 0.2 for ``Weak Effects'' simulations, and at 0.3 for ``Strong Effects'' simulations. Methods are generally robust to thresholding choice. Results for precision and recall on thresholded edges are shown in Figures \ref{suppfig:large simulation precision} and \ref{suppfig:large simulation recall}, respectively.

\subsubsection{Benchmarking}
\label{suppsection:large p simulation benchmarking}
In Figure \ref{fig:benchmarking}, we benchmark wallclock time and memory usage for all methods in $p=40$ simulations \cite{koster2012snakemake} on the UCLA hoffman2 cluster. All continuous optimization methods (\dotears{}, NO TEARS, GOLEM-EV, and GOLEM-NV) have significantly higher average runtimes than other methods, which is partially explained by cross-validation procedures. \dotears{} has relatively light memory usage, outperformed only by NO TEARS and sortnregress.

DCDI-G has enormous memory requirements, especially relative to other methods. We note that for DCDI-G, the reported benchmarks do not include memory usage or runtime from cross-validation folds. Instead, we report only the ``main'' run of DCDI-G, and thus DCDI-G is not denoted with the * indicating cross-validation.  

\begin{figure*}[ht!]
    \centering
    \begin{subfigure}[t]{0.45\textwidth}
        \centering
        \includegraphics[width=\textwidth]{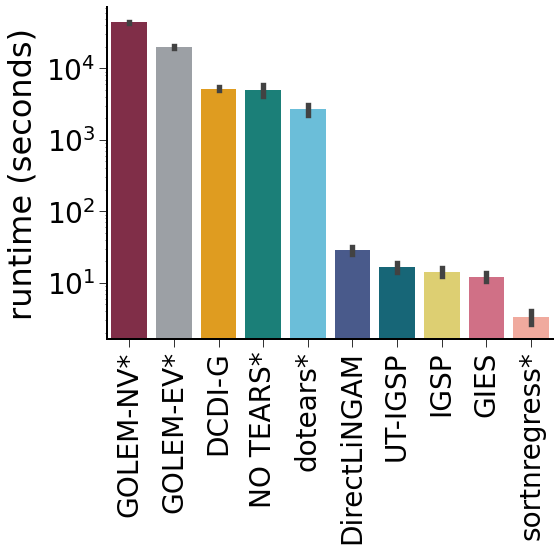}
        \caption{Wallclock time across all $p=40$ simulations in seconds. Y-axis is in log-scale.}    
        \label{fig:runtime}
    \end{subfigure}
    \hfill
    \begin{subfigure}[t]{0.45\textwidth}  
        \centering 
        \includegraphics[width=\textwidth]{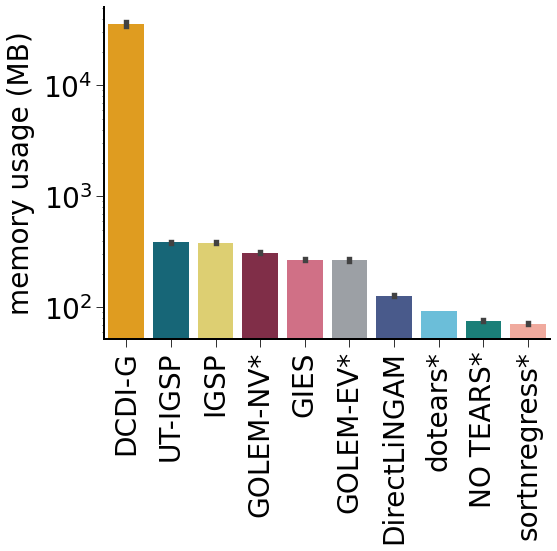}
        \caption{Memory usage across all $p=40$ simulations, defined as the maximum unique set size used by the process, or max\_uss (see psutil documentation).}        
        \label{fig:memory}
    \end{subfigure}
    \caption{Comparison of mean runtime (a) and mean memory usage (b) between methods across all $p=40$ simulations. Error bars are standard errors.} 
    \label{fig:benchmarking}
\end{figure*}
    
\subsubsection{Thresholding on large simulation results}
\label{suppsection:thresholding}
Edge thresholding for weighted adjacency matrices is necessary for accurate evaluation using SHD, but the choice of threshold can feel arbitrary. We find that methods are generally robust to thresholding choice, following similar results from \cite{reisach2021beware}.

Figure \ref{suppfig:thresholding all results} examines the effect on thresholding of small weights in $W$ in large random DAG simulations, for methods that infer weighted adjacency matrices (\dotears{}, NO TEARS, sortnregress, DirectLiNGAM, GOLEM-NV, and GOLEM-EV) \cite{zheng2018dags, reisach2021beware, ng2020role, shimizu2011directlingam}. For simplicity, we summarize simulations in terms of the generative edge distribution. Simulations with ``Weak Effects'', $w\sim \Unif\left([-1.0, -0.3] \cup [0.3, 1.0]\right)$, are shown in Figure \ref{suppfig:thresholding small edges}; simulations with ``Strong Effects'', $w\sim \Unif\left([-2.0, -0.5] \cup [0.5, 2.0]\right)$ are shown in Figure \ref{suppfig:thresholding large edges}. We compare the SHD between the ground truth adjacency matrix and the inferred adjacency matrix for each method at all thresholds between 0 and the absolute lower bound of the true edge weight distribution (0.3 for ``Weak Effects'', 0.5 for ``Strong Effects''). Any edge whose magnitude is below the chosen threshold is set to 0 for SHD evaluation.

\begin{figure*}[ht!]
    \centering
    \begin{subfigure}[t]{0.45\textwidth}
        \centering
        \includegraphics[width=\textwidth]{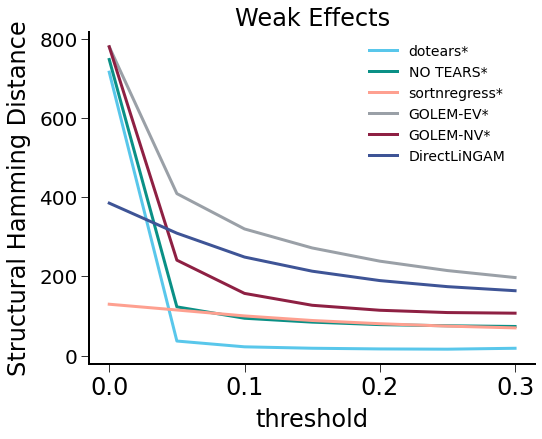}
        \caption{Structural Hamming Distance (lower is better) as a function of $\hat{w}$ threshold for ``Weak Effects'' scenarios, where $w\sim \Unif\left([-1.0, -0.3] \cup [0.3, 1.0]\right)$.}    
        \label{suppfig:thresholding small edges}
    \end{subfigure}
    \hfill
    \begin{subfigure}[t]{0.45\textwidth}  
        \centering 
        \includegraphics[width=\textwidth]{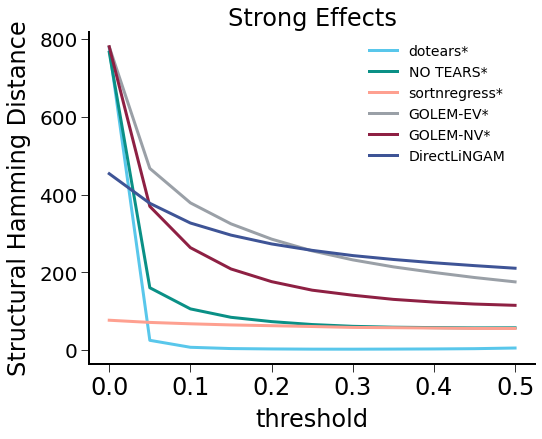}
        \caption{Structural Hamming Distance (lower is better) as a function of $\hat{w}$ threshold for ``Strong Effects'' scenarios, where $w\sim \Unif\left([-2.0, -0.5] \cup [0.5, 2.0]\right)$.} 
        \label{suppfig:thresholding large edges}
    \end{subfigure}
    \caption{Structural Hamming Distance (lower is better) as a function of thresholding choice for $p=40$ simulations detailed in \ref{suppsection:large p data generation}. Parameterizations with ``Weak Effects'', $w\sim \Unif\left([-1.0, -0.3] \cup [0.3, 1.0]\right)$, are consolidated in (a), while parameterizations with ``Strong Effects'', $w\sim \Unif\left([-2.0, -0.5] \cup [0.5, 2.0]\right)$, are consolidated in (b).} 
    \label{suppfig:thresholding all results}
\end{figure*}

Figure \ref{suppfig:thresholding all results} shows that thresholding is necessary for evaluation of SHD between weighted adjacency matrices, but also that methods are generally robust to thresholding choice. Without thresholding (equivalently, a threshold of 0), SHD results are inflated, but recover even at low thresholds. 

\begin{figure*}[ht!]
    \centering
    \includegraphics[width=\textwidth]{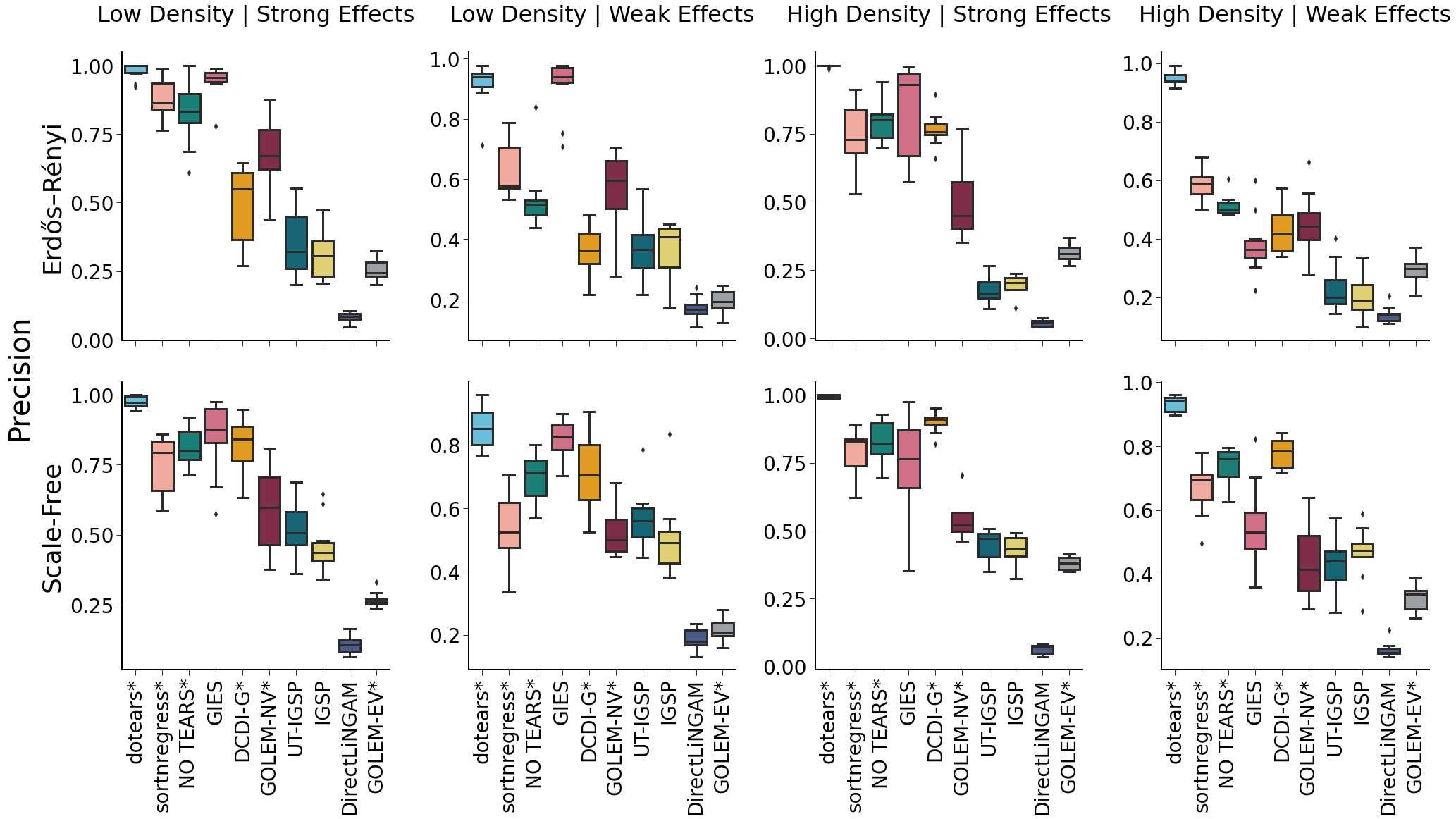}
    \caption{Methodological comparison of structure learning methods on large random graphs ($p=40$) using precision (higher is better). Each row represents simulations drawn with a topology in \{\erdosrenyi{}, Scale Free\}, respectively, and each column represents a different parameterization of edge density and edge weights ordered in terms of increasing difficulty. For details, see Table \ref{tab:simulation parameters}. 10 simulations were drawn for each set of parameters with sample size $(p+1) * 100 = 4100$, and * indicates methods performed with 5-fold cross-validation on a separate draw from the same distribution. Y-axis scales are not shared between figures. Methods are sorted, left to right, by strongest average performance across all simulations. GIES, IGSP, and UT-IGSP infer binary edges and are therefore omitted for fairness.}
    \label{suppfig:large simulation precision}
\end{figure*}

\begin{figure*}[ht!]
    \centering
    \includegraphics[width=\textwidth]{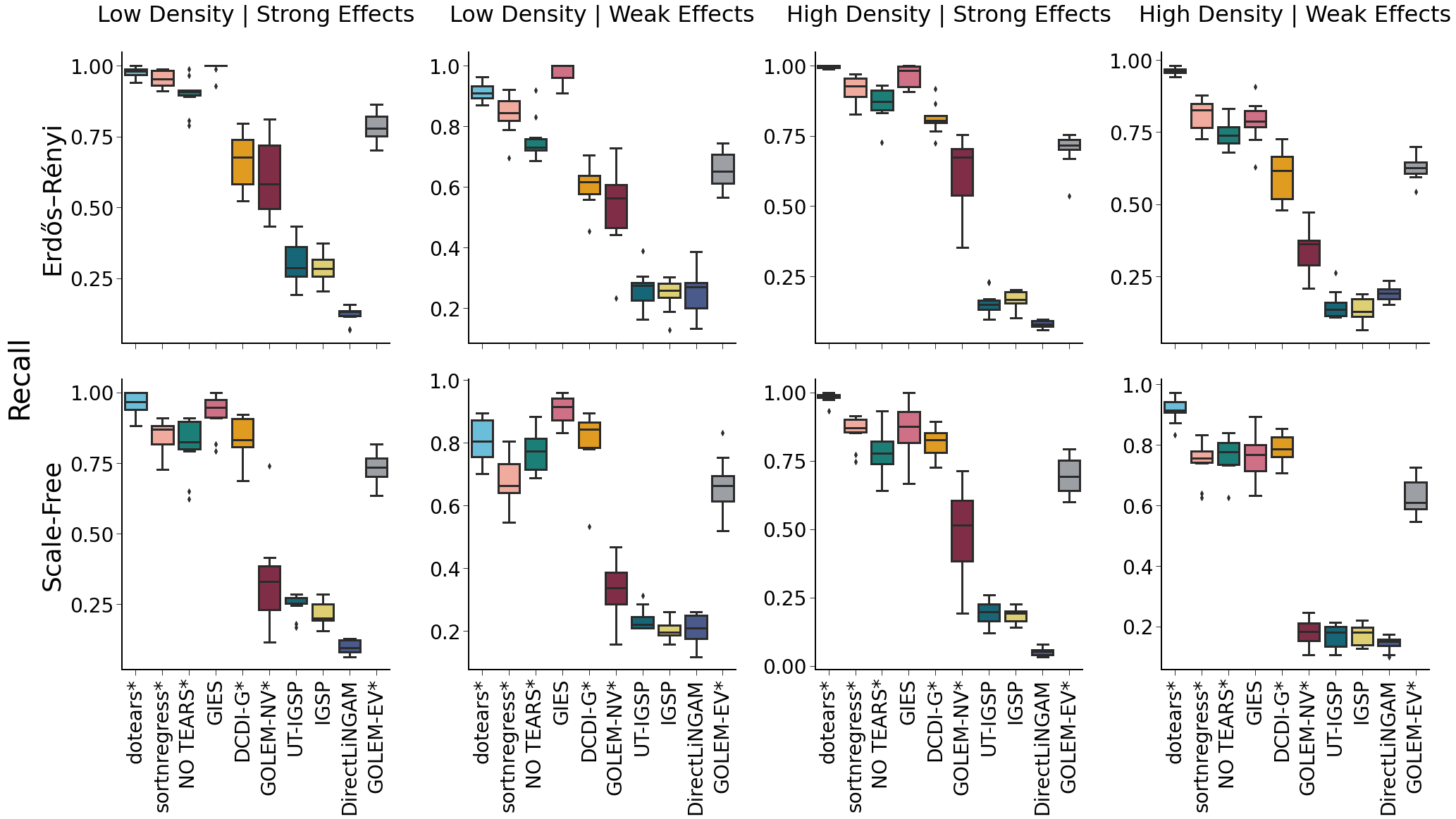}
    \caption{Methodological comparison of structure learning methods on large random graphs ($p=40$) using recall (higher is better). Each row represents simulations drawn with a topology in \{\erdosrenyi{}, Scale Free\}, respectively, and each column represents a different parameterization of edge density and edge weights ordered in terms of increasing difficulty. For details, see Table \ref{tab:simulation parameters}. 10 simulations were drawn for each set of parameters with sample size $(p+1) * 100 = 4100$, and * indicates methods performed with 5-fold cross-validation on a separate draw from the same distribution. Y-axis scales are not shared between figures. Methods are sorted, left to right, by strongest average performance across all simulations. GIES, IGSP, and UT-IGSP infer binary edges and are therefore omitted for fairness.}
    \label{suppfig:large simulation recall}
\end{figure*}

\subsubsection{Edge weight estimation for large random simulations}
\label{suppsection:shd distribution significance}
Accurate estimation of edge weights is important for understanding structure. Figure \ref{fig:large simulations SHD results} gives results on structural recovery through SHD, but does not inform edge weight recovery. To measure edge weight recovery we use $\ell_1$ distance, defined as the vector $\ell_1$ norm between the flattened true weighted adjacency matrix and the flattened inferred weighted adjacency matrix. $\ell_1$ distance gives information on both structure recovery and edge weight estimation simultaneously. In Figure \ref{suppfig:large simulation L1 results}, we benchmark the recovery of edge weights for methods that return a weighted adjacency matrix (\dotears{}, sortnregress, NO TEARS, GOLEM-EV, GOLEM-NV, DirectLiNGAM \cite{zheng2018dags, reisach2021beware, ng2020role, shimizu2011directlingam}) for simulations given in \ref{suppsection:large p data generation}. For fairness, we exclude methods that only return a binary adjacency matrix. Methods are thresholded in the same manner as in Figure \ref{fig:large simulations SHD results}.

\begin{figure*}[ht!]
    \centering
    \includegraphics[width=\textwidth]{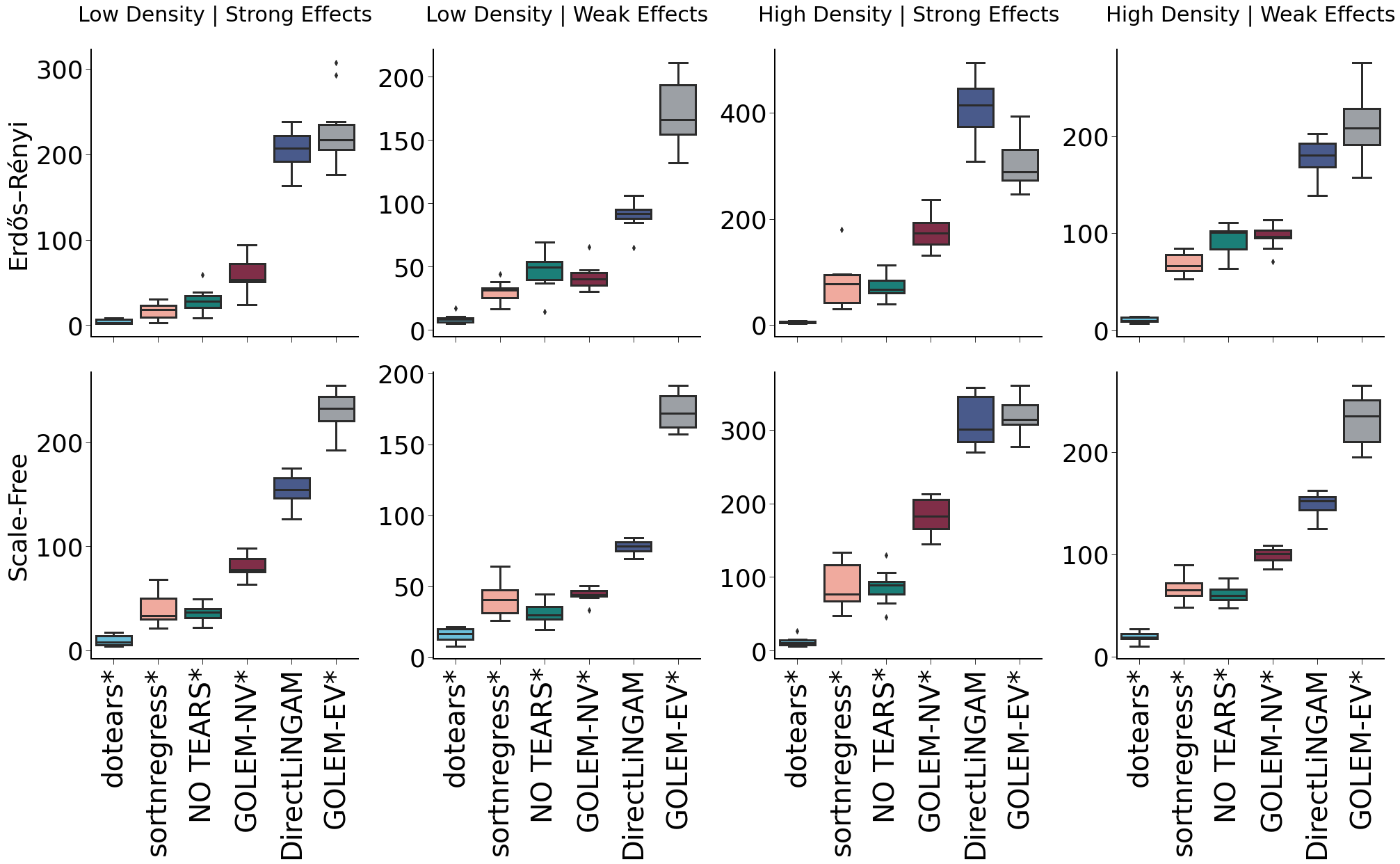}
    \caption{Methodological comparison of structure learning methods on large random graphs ($p=40$) using $\ell_1$ distance (lower is better). Each row represents simulations drawn with a topology in \{\erdosrenyi{}, Scale Free\}, respectively, and each column represents a different parameterization of edge density and edge weights ordered in terms of increasing difficulty. For details, see Table \ref{tab:simulation parameters}. 10 simulations were drawn for each set of parameters with sample size $(p+1) * 100 = 4100$, and * indicates methods performed with 5-fold cross-validation on a separate draw from the same distribution. Y-axis scales are not shared between figures. Methods are sorted, left to right, by strongest average performance across all simulations. GIES, IGSP, and UT-IGSP infer binary edges and are therefore omitted for fairness.}
    \label{suppfig:large simulation L1 results}
\end{figure*}

\dotears{} outperforms all other methods in terms of effect size recovery. In addition, the relative ordering of the methods stays consistent with Figure \ref{fig:large simulations SHD results}. 

\subsubsection{Statistical significance of SHD distribution difference}
For each method, and across all simulations in Section \ref{suppsection:large p data generation}, we compare the SHD distributions, marginalized across all simulation parameterizations, against that of \dotears{}. For each method, we perform a one-sided Mann-Whitney U test to test difference between the method's SHD distribution and the SHD distribution of \dotears{}. The resulting p-values are reported in Table \ref{supptab:mann whitney results}, and are significant for all methods.

\begin{table}[ht!]
    \centering
    \begin{tabular}{c|c}
        \textbf{Method} & \textbf{p-value} \\
        \hline
        GIES & 2.08e-7 \\
        sortnregress & 1.1e-21 \\
        NO TEARS & 9.27e-23 \\
        DCDI-G & 1.04e-24 \\
        GOLEM-NV & 8.45e-28\\
        GOLEM-EV & 4.47e-28 \\
        DirectLiNGAM & 4.47e-28 \\
        UT-IGSP & 4.467e-28 \\
        IGSP & 4.47e-28 \\
    \end{tabular}
    \caption{p-value from asymptotic one-sided Mann-Whitney U test comparing distribution of indicated method SHD against distribution of \dotears{} SHD.}
    \label{supptab:mann whitney results}
\end{table}

\subsubsection{Incorporation of interventional loss}
\label{section:interventional vs observational only}
We note that inferring $W$ from strictly observational data, given $\hat{\Omega}_0$ from interventional data, has convenient theoretical properties but ignores the majority of our data. Figure \ref{suppfig:interventional vs observational only} compares the performance, in Structural Hamming Distance, of \dotears{} using both observational and interventional data $\left(\mathcal{L}_{\hat{\Omega}_0} \left(\interv{W}{k}, \interv{\boldX}{k}\right) \text{ for } k=0\dots p\right)$ with \dotears{} using only observational data $\left(\mathcal{L}_{\hat{\Omega}_0} \left(\interv{W}{k}, \interv{\boldX}{k}\right) \text{ for } k=0\right)$. We use the simulations drawn in Section \ref{suppsection:large p data generation}, aggregated across all scenarios in Table \ref{tab:simulation parameters}. Note that $\hat{\Omega}_0$ is estimated identically for both.

\begin{figure*}[ht!]
    \centering
    \includegraphics[width=0.4\textwidth]{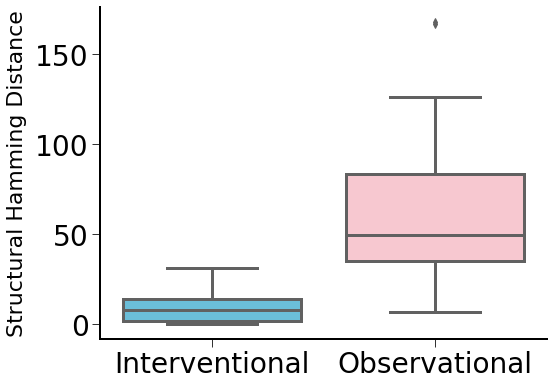}
    \caption{Comparison of \dotears{} performance when using only $\mathcal{L}_{\hat{\Omega}_0}\left(\interv{W}{k}, \interv{\boldX}{k}\right)$ for observational data ($k=0$), right, and when including interventional data ($k=0\dots p$), left. Methods are run on simulation data generated in Section \ref{suppsection:large p data generation}, averaged across parameterizations, and compared using Structural Hamming Distance.}
    \label{suppfig:interventional vs observational only}
\end{figure*}

As expected, restricting \dotears{}{} to inference only using the observational data (Figure \ref{suppfig:interventional vs observational only}, right) drastically decreases performance when compared to inference using both observational and interventional data (Figure \ref{suppfig:interventional vs observational only}, left). This motivates including $k=1\dots p$ into the loss and proving consistency for all $k$.

\subsection{Sensitivity Analysis}
\label{section:sensitivity}

\subsubsection{Sensitivity under linear SEM}
\label{section:linear}
\begin{figure*}[h]
    \centering
    \includegraphics[width=\textwidth]{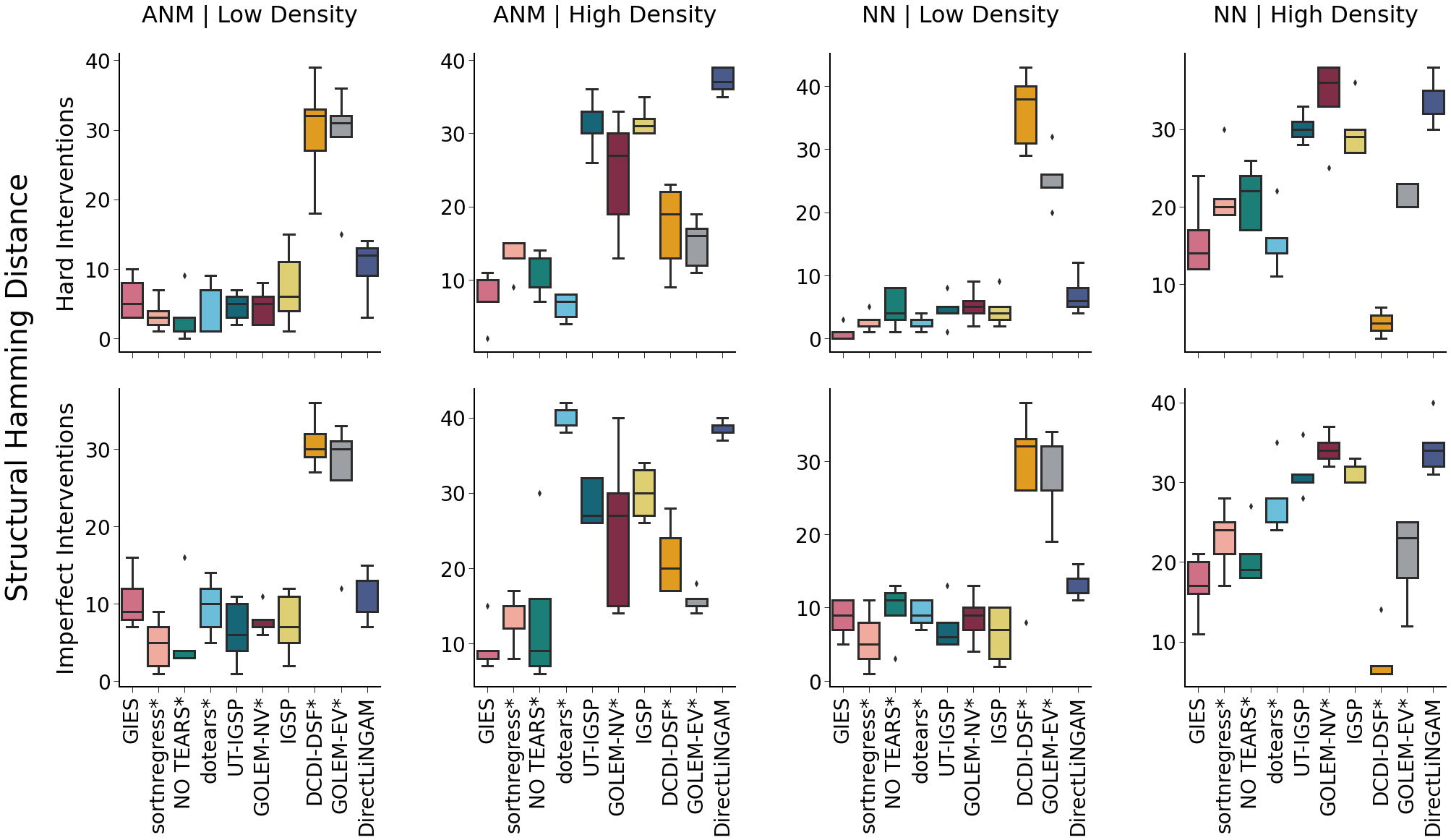}
    \caption{Method performance under nonlinear simulations in $p=10$ random graphs, in Structural Hamming Distance (lower is better). Data are generated by a neural network with Additive Noise Model (ANM) or a neural network with non-additive noise (NN) model. See \cite{brouillard2020differentiable} and Supplementary Material \ref{suppsection:nonlinear sensitivity} for details. Scenarios are sorted by increasing difficulty. Methods are sorted by average performance across scenarios.}
    \label{fig:nonlinear}
\end{figure*}

We performed sensitivity analyses under a linear SEM on \erdosrenyi{} graphs ($p=20$) with two edge density parameterizations crossed with two edge weight parameterizations (we note that the specific density and edge weight parameterizations differ from Section~\ref{section:large random graph simulations}). We report results averaged over $10$ replicates with a sample size of $n_k = 100$. We report SHD in the main text. For precision/recall for edge recovery, parameterization-specific results, details on data generation, and modelling specifics, see Supplementary Material \ref{suppsection:linear sensitivity}.

We assessed the sensitivity of our results to violation of Assumption \ref{assumption:hard intervention assumption} by allowing for $\pi$ percent of the observational parental variance into the marginal variance of the target . Here, $\pi = 0$ is a hard intervention. The accuracy of \dotears{} decreases as $\pi$ increases, or equivalently as the interventional signal decreases (Figure \ref{fig:hard interventions}). This effect is dependent on the magnitude of the edge weights and the density of the true DAG: in ``High Density'' and ``Strong Effects'' scenarios, the  decrease in accuracy is noticeable due to higher incoming parental variance (Supplementary Figure \ref{suppfig:hard interventions}). Despite this decrease in accuracy, \dotears{} remains the second best performing method across parameterizations, even when excluding the $\pi = 0$ scenario.
 
To assess sensitivity to violation of Assumption \ref{assumption:global a assumption}, we randomly perturb each $\alpha$ per node to allow non-uniform interventional effects on error variance across targets. \dotears{} is robust to violations of Assumption \ref{assumption:global a assumption} in these simulations (Figure~\ref{fig:alpha perturbation}). 

Finally, our model assumes that the distribution of the target $\interv{X_k}{k}$ depends on that of $\interv{\epsilon_k}{0}$. In Figure \ref{fig:fixed interventions}, we explore interventions with a fixed distribution, $\interv{X_k}{k} \iid \Nor(2, 1)$ following the model used in \cite{brouillard2020differentiable} and similar to that considered by \cite{hauser2012characterization}. In our model, this violates Assumption \ref{assumption:global a assumption}.

In Figure \ref{fig:fixed interventions}, \dotears{} is outperformed on average only by GIES in SHD, and outperforms DCDI under their simulation model. Supplementary Figure \ref{suppfig:linear sensitivity} shows that simulations with weak effects, where $0.3 \leq |w| \leq 0.5$ drive this performance dip; in simulations with strong effects, where $0.8 \leq |w| \leq 1.0$, \dotears{} again vastly outperforms other methods, including GIES. Sufficient signal-to-noise ratio can thus overcome severe misspecification of $\Omega_0$ under hard interventions.

\subsubsection{Sensitivity under nonlinear SEM}
\label{section:nonlinear}
To assess sensitivity to a nonlinear SEM, we generate nonlinear data under models that were used to evaluate DCDI \cite{brouillard2020differentiable}. We simulated $10$ \erdosrenyi{} DAGs with $p = 10$ and $n=10000$ total samples. For the non-linear models, we considered, in turn, a neural network additive noise model (ANM) \cite{buhlmann2014cam} and a neural network with non-additive noise model (NN) \cite{chickering2002optimal}. We test both hard interventions and ``imperfect'' interventions. Imperfect interventions add a random vector drawn from $\Nor(0, 1)$ to the last layer of the neural network. See \cite{brouillard2020differentiable} and Supplementary Material \ref{suppsection:nonlinear sensitivity} for details. Figure \ref{fig:nonlinear} shows results in SHD.

With hard interventions, \dotears{} maintains performance in DAG estimation across increasing levels of non-linearity, and on average outperforms all methods. \dotears{} performs worse in the ``imperfect'' intervention setting. In particular, \dotears{} has average performance in ``Low Density'' scenarios, but performs much worse in ``High Density'' simulations. This suggests that the imperfect intervention model, rather than non-linearity, drives the relatively worse performance of \dotears{}. As in Figure \ref{fig:hard interventions}, dense DAGs exacerbate difficulties with imperfect interventions. However, on average \dotears{} outperforms the neural network DCDI even under ``imperfect'' interventions.

\begin{figure*}[h]
     \centering
     \begin{subfigure}[t]{0.32\textwidth}
         \centering
         \includegraphics[width=\textwidth]{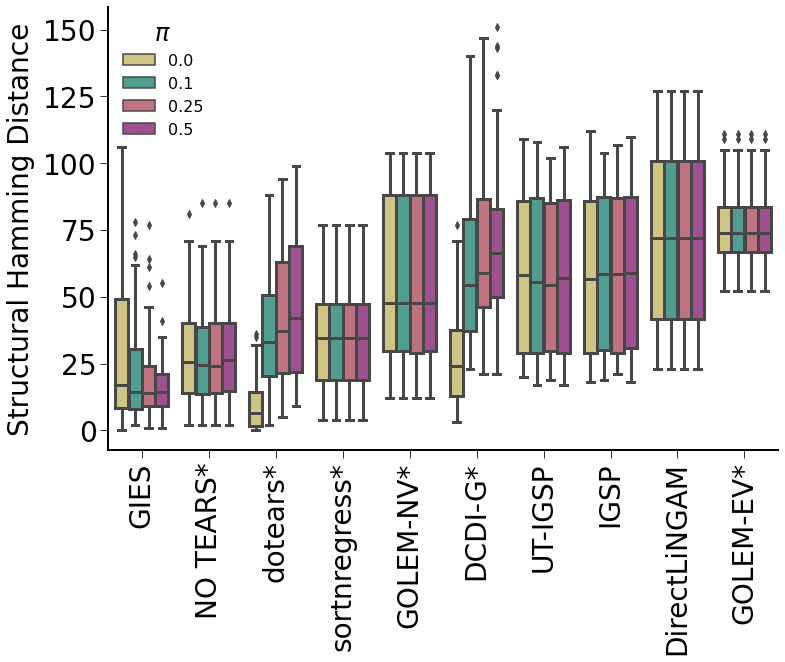}
         \caption{Performance under soft interventions, where $\pi$ is the proportion of allowed parental variance to the target.}
         \label{fig:hard interventions}
     \end{subfigure}
     \hfill
     \begin{subfigure}[t]{0.32\textwidth}
         \centering
         \includegraphics[width=\textwidth]{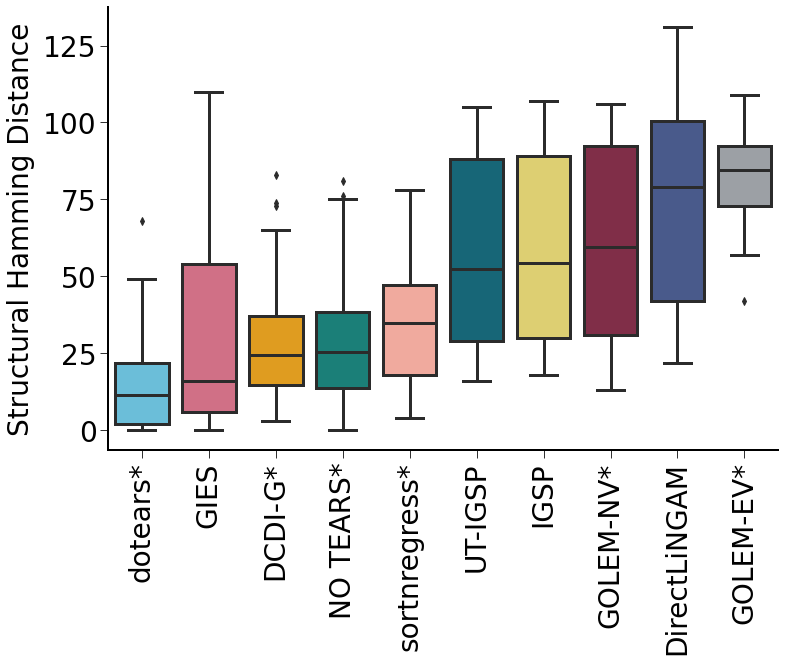}
         \caption{Performance under random per-node effects by interventions on target error variance.}
         \label{fig:alpha perturbation}
     \end{subfigure}
     \hfill
     \begin{subfigure}[t]{0.32\textwidth}
         \centering
         \includegraphics[width=\textwidth]{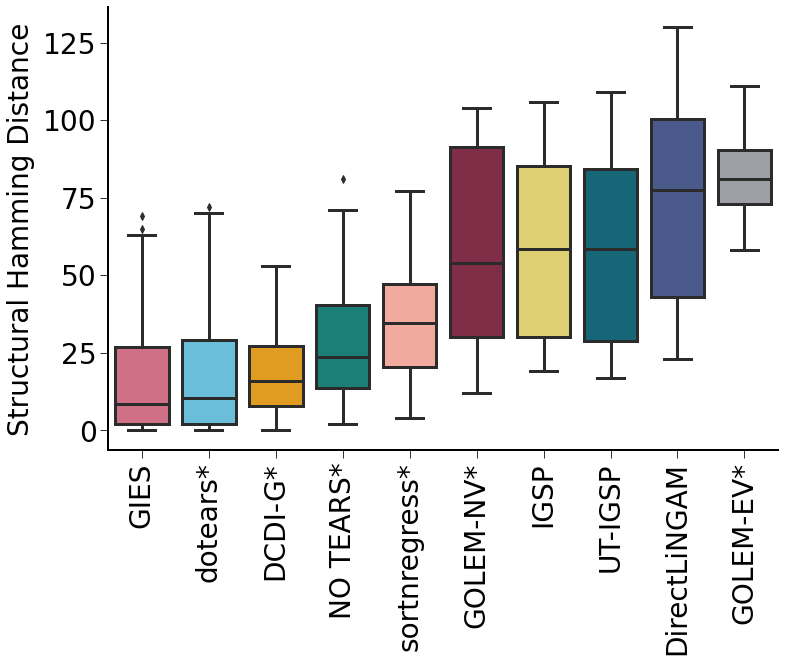}
         \caption{Performance under interventions which fix the distribution of the target to $\Nor(2, 1)$.}
         \label{fig:fixed interventions}
     \end{subfigure}
    \caption{Method performance on sensitivity analyses aggregated across parameterizations of $p=20$ random graphs, in Structural Hamming Distance (lower is better). See Supplementary Material \ref{suppsection:linear sensitivity} for details.}
    \label{fig:linear sensitivity analyses}
\end{figure*}

\subsection{Sensitivity under linear SEM}
\label{suppsection:linear sensitivity}
For evaluation of \dotears{} sensitivity to violations of the modeling assumptions, data is generated for DAGs with $p = 20$ nodes. Structures are drawn from both \erdosrenyi{}{} (ER) and Scale-Free (SF) DAGs \cite{erdHos1960evolution, barabasi1999emergence}. We simulate under two parameterizations $\{\text{Low Density, High Density}\}$ of the edge densities $(r, z)$. In the ``Low Density'' parameterization, give $(r, z) = (0.2, 2)$. To evaluate performance on higher density topologies, we also give ``High Density'' parameterizations, where $(r, z) = (0.5, 6)$. 

Given an edge density scenario, we simulate under two parameterizations of the edge weights. In the ``Strong Effects'' parameterization, $w \sim \Unif\left([-1.0, -0.8] \cup [0.8, 1.0]\right)$. We also give the ``Weak Effects'' parameterization, which edge weights are drawn from $w \sim \Unif\left([-0.3, -0.5] \cup [0.3, 0.5]\right)$. Table \ref{tab:simulation parameters} summarizes all four possible simulation parameterizations.

For each node $i$, we draw $\sigma_i \sim \Unif\left([0.5, 2.0]\right)$, and draw $n_{0}$ observations from $\interv{\boldepsilon_i}{0} \sim \Nor\left(0, \sigma_i^2\right)$. For each DAG we generate an instance of observational data, where $n_{0} = (p + 1) * n_k = 2100$, and an instance of interventional data, where $n_{k} = 100$ for all $k=0\dots p$ to match sample size. We set the distribution of $\interv{\epsilon_i}{k}$ according to Assumptions \ref{assumption:global a assumption} and \ref{assumption:variance without intervention} for $\alpha=4$. Cross-validation was performed as in Supplementary Material \ref{suppsection:large p data generation}.

\begin{table*}[ht!]
    \centering
    \begin{tabular}{r@{}l||r@{}l|c}
        \toprule
        \multicolumn{2}{c||}{\textbf{Simulation Type}} & \multicolumn{2}{c|}{\textbf{Topology}} & $w$\tabularnewline
        \hline
        Low Density&, Strong Effects & ER-0.2\,\,&OR SF-2 & $\Unif\left([-1.0, -0.8] \cup [0.8, 1.0]\right)$  \tabularnewline  
        Low Density&, Weak Effects & ER-0.2\,\,&OR SF-2 & $\Unif\left([-0.5, -0.3] \cup [0.3, 0.5]\right)$ \tabularnewline 
        High Density&, Strong Effects & ER-0.5\,\,&OR SF-6 & $\Unif\left([-1.0, -0.8] \cup [0.8, 1.0]\right)$ \tabularnewline 
        High Density&, Weak Effects & ER-0.5\,\,&OR SF-6 & $\Unif\left([-0.5, -0.3] \cup [0.3, 0.5]\right)$  \tabularnewline
    \end{tabular}
    \caption{For each simulation type for $p=20$ sensitivity analyses, we provide the parameterizations of the \erdosrenyi{} and Scale-Free topologies and the distribution on the drawn edge weights.}
    \label{tab:simulation parameters sensitivity}
\end{table*}

\subsubsection{Hard Interventions}
\begin{figure*}[h]
     \centering
     \includegraphics[height=0.5\textheight]{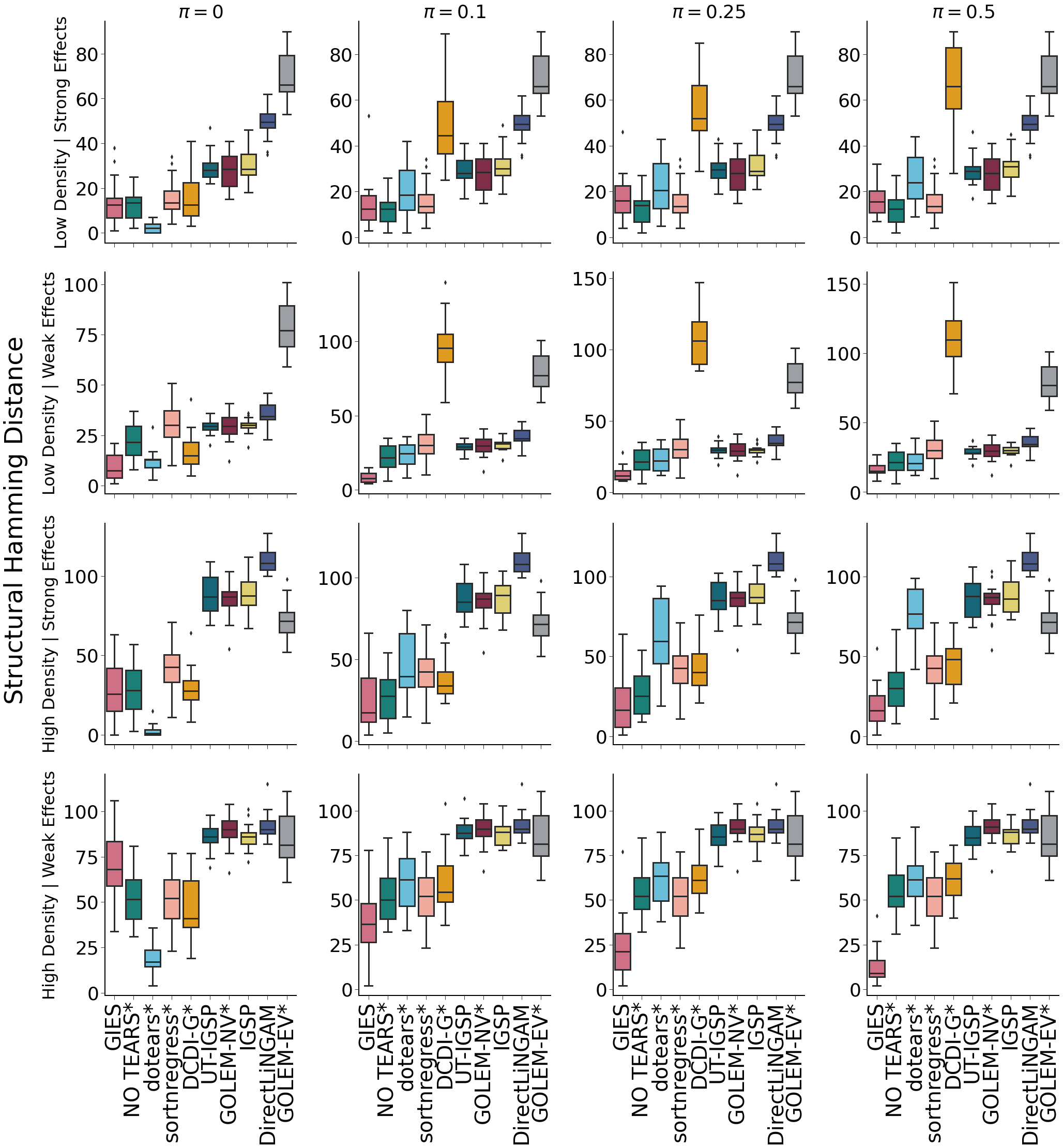}
    \caption{Method performance under soft or imperfect interventions, where $\pi$ is the proportion of allow parental variance to the target. Each row is a different parameterization of DAG density and edge weights; each column is a different value of $\pi$. Results are shown in Structural Hamming Distance (SHD), and methods are sorted by average performance in SHD.}
    \label{suppfig:hard interventions}
\end{figure*}

\begin{figure*}[h]
    \centering
    \includegraphics[height=0.5\textheight]{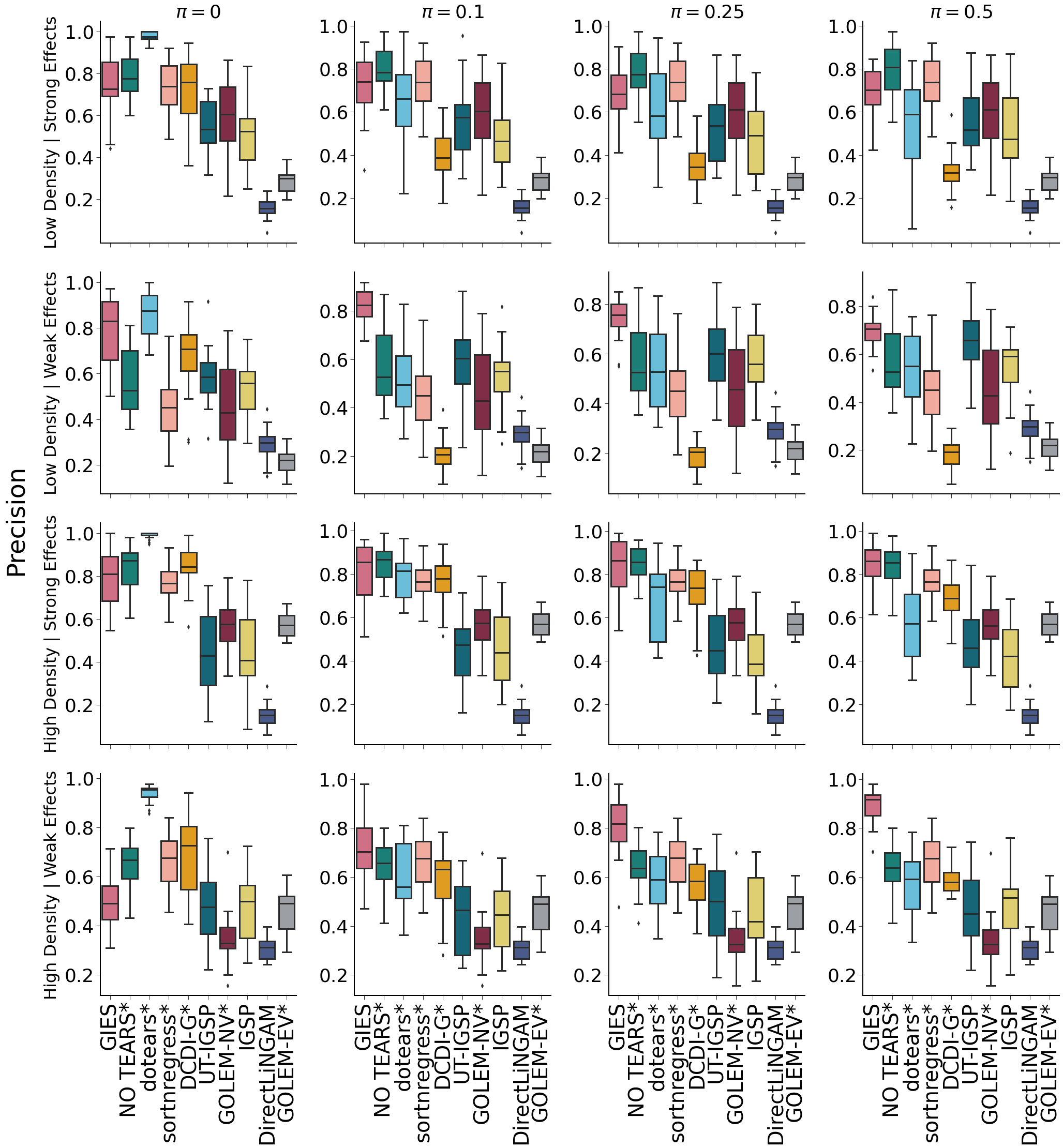}
    \caption{Method performance under soft or imperfect interventions, where $\pi$ is the proportion of allow parental variance to the target. Each row is a different parameterization of DAG density and edge weights; each column is a different value of $\pi$. Results are shown in precision of binary edge recovery, and methods are sorted by average performance in SHD.}
    \label{suppfig:hard precision}
\end{figure*}

\begin{figure*}[h]
    \centering
    \includegraphics[height=0.5\textheight]{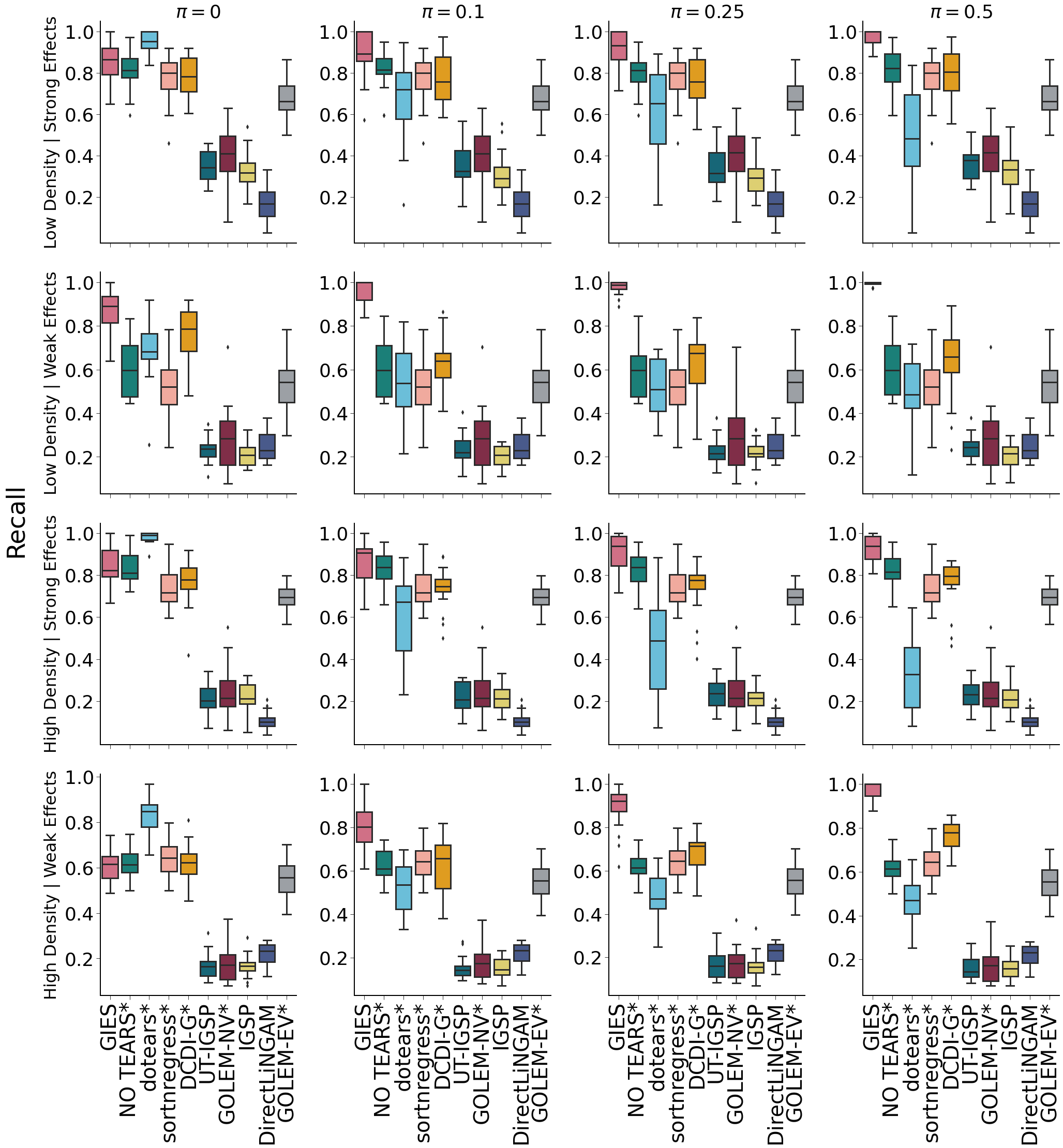}
    \caption{Method performance under soft or imperfect interventions, where $\pi$ is the proportion of allow parental variance to the target. Each row is a different parameterization of DAG density and edge weights; each column is a different value of $\pi$. Results are shown in recall of binary edge recovery, and methods are sorted by average performance in SHD.}
    \label{suppfig:hard recall}
\end{figure*}

Assumption \ref{assumption:hard intervention assumption} assumes that interventions remove all edges incoming to the target, and therefore that the marginal variance of the target has no incoming parental variance. We examine the model where 
\begin{equation}
    \interv{X_k}{k} = \sqrt{\pi} \sum_{i=0}^p w_{ij} \interv{X_i}{k} + \interv{\epsilon_k}{k},
        \nonumber
\end{equation}
where interventions downweight incoming parental edges by $\sqrt{\pi}$. Accordingly, this means that 
\begin{equation}
    \VV\left(\interv{X_k}{k}\right) = \pi \sum_{i=0}^p w_{ij}^2 \VV\left(\interv{X_i}{k}\right) + \frac{\sigma_k^2}{\alpha^2},
        \nonumber
\end{equation}
where for $\pi = 0$ we recover hard interventions. We call these interventions either soft or imperfect interventions, and examine method performance when $\pi = \{0, 0.1, 0.25, 0.5\}$. Note that other than the imperfect intervention, the model and generating processes are the same as in Supplementary Material \ref{suppsection:large p data generation}.

For each $\pi \in \{0, 0.1, 0.25, 0.5\}$ and each parameterization in Table \ref{tab:simulation parameters sensitivity}, ten \erdosrenyi{} DAGs and ten Scale-Free DAGs are drawn, with sample size matched as described above. Results in Structural Hamming Distance are shown in Supplementary Figure \ref{suppfig:hard interventions} per parameterization, and are shown averaged across parameterizations in Figure \ref{fig:hard interventions}. Results in binary precision and recall are given in Supplementary Figure \ref{suppfig:hard precision} and Supplementary Figure \ref{suppfig:hard recall}, respectively.

\subsubsection{$\alpha$ perturbation}
Assumption \ref{assumption:global a assumption} assumes that for each node under intervention, the error variance is reduced by a shared factor $\alpha$ across interventional targets, i.e.
\begin{equation}
    \VV\left(\interv{\epsilon_k}{k}\right) = \frac{\sigma_k^2}{\alpha^2}
        \nonumber
\end{equation}
for all $k$.

To test the sensitivity of \dotears{} to this assumption, we draw data for which
\begin{equation}
    \VV\left(\interv{\epsilon_k}{k}\right) = \frac{\sigma_k^2}{c_k^2\alpha^2}, \quad c_k \sim \text{Unif}([0.8, 1.2]).
        \nonumber
\end{equation}
$c_k$ is target-specific, and thus under this model $\Omega_0$ will be misspecified.

Ten \erdosrenyi{} DAGs and ten Scale-Free DAGs are drawn, with sample size matched as described above. Results in Structural Hamming Distance, averaged across parameterizations, are shown in Figure \ref{fig:alpha perturbation}. Results for each parameterization in Structural Hamming Distance are shown in Supplementary Figure \ref{suppfig:linear sensitivity} (top row). Results in binary edge precision are shown in Supplementary Figure \ref{suppfig:linear precision}. Results in binary edge recall are shown in Supplementary Figure \ref{suppfig:linear recall}.

\begin{figure*}[h]
     \centering
    \includegraphics[width=0.8\textwidth]{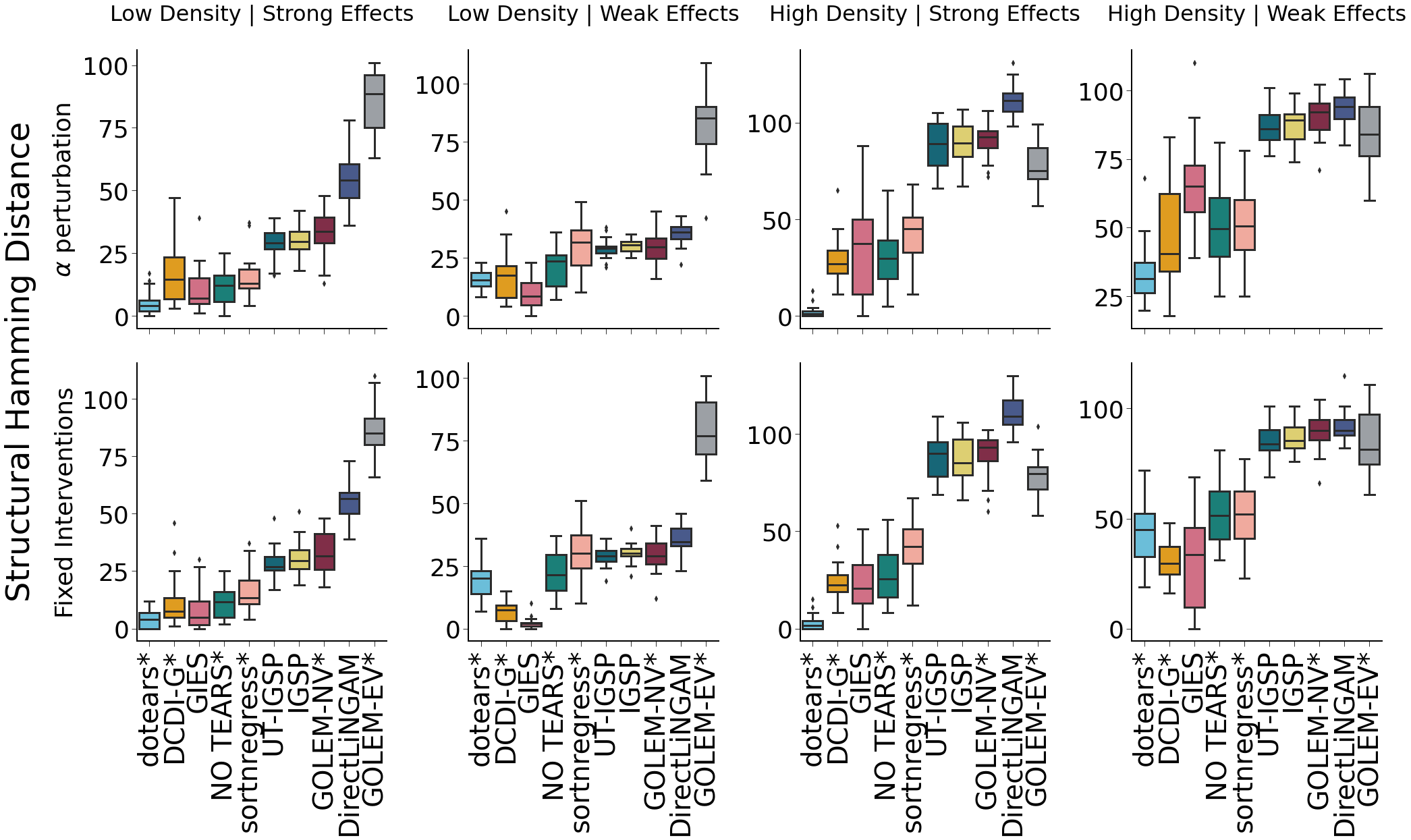}
    \caption{Method performance in linear sensitivity analyses. Each row is a different sensitivity analysis; each column is a parameterization of DAG density and edge weights. Results are shown in Structural Hamming Distance (SHD), and methods are sorted by average performance in SHD.}
    \label{suppfig:linear sensitivity}
\end{figure*}

\begin{figure*}[h]
    \centering
    \includegraphics[width=0.8\textwidth]{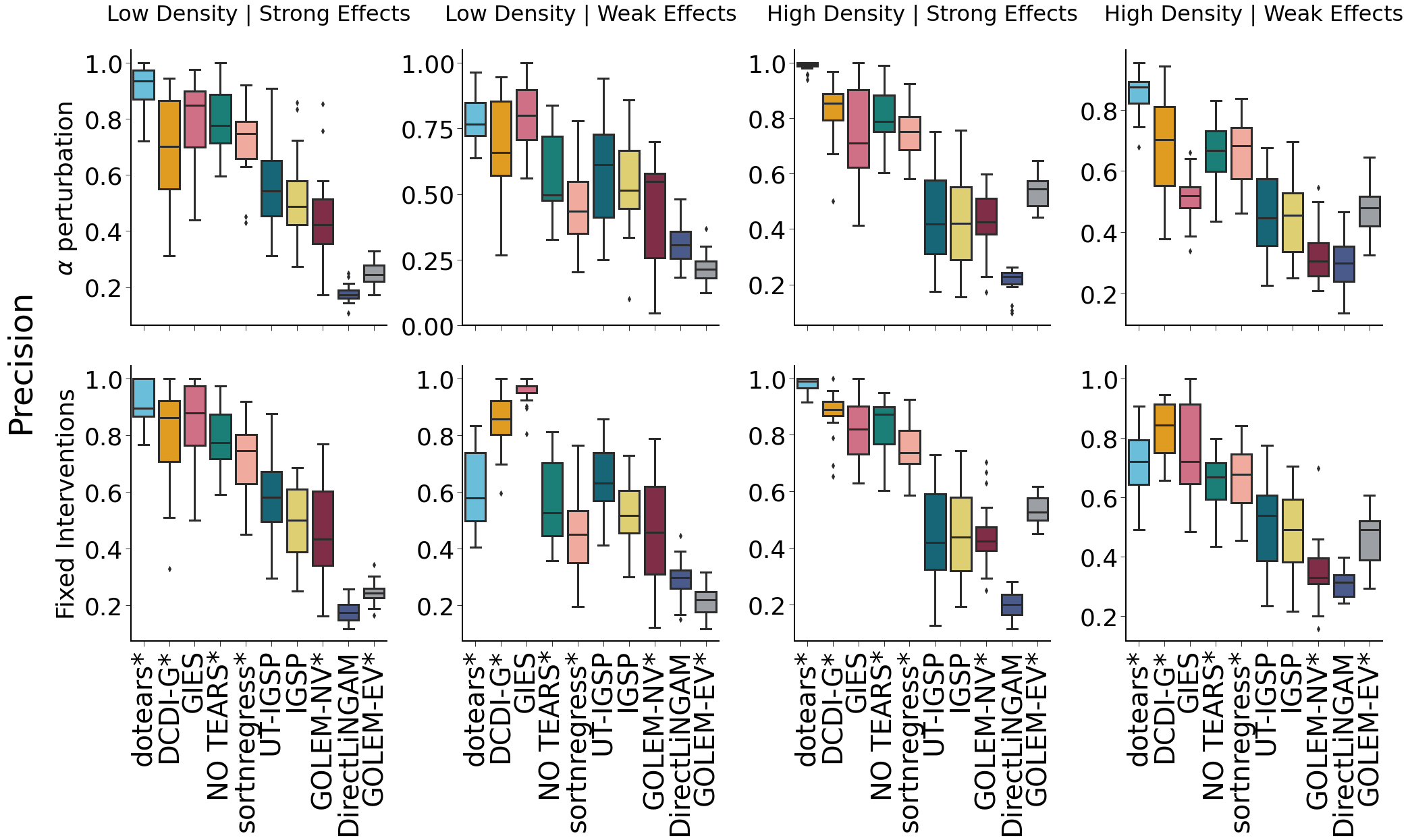}
    \caption{Method performance under linear sensitivity analyses. Each row is a different sensitivity analysis; each column is a parameterization of DAG density and edge weights. Results are shown in precision of binary edge recovery, and methods are sorted by average performance in SHD.}
    \label{suppfig:linear precision}
\end{figure*}

\begin{figure*}[h]
    \centering
    \includegraphics[width=0.8\textwidth]{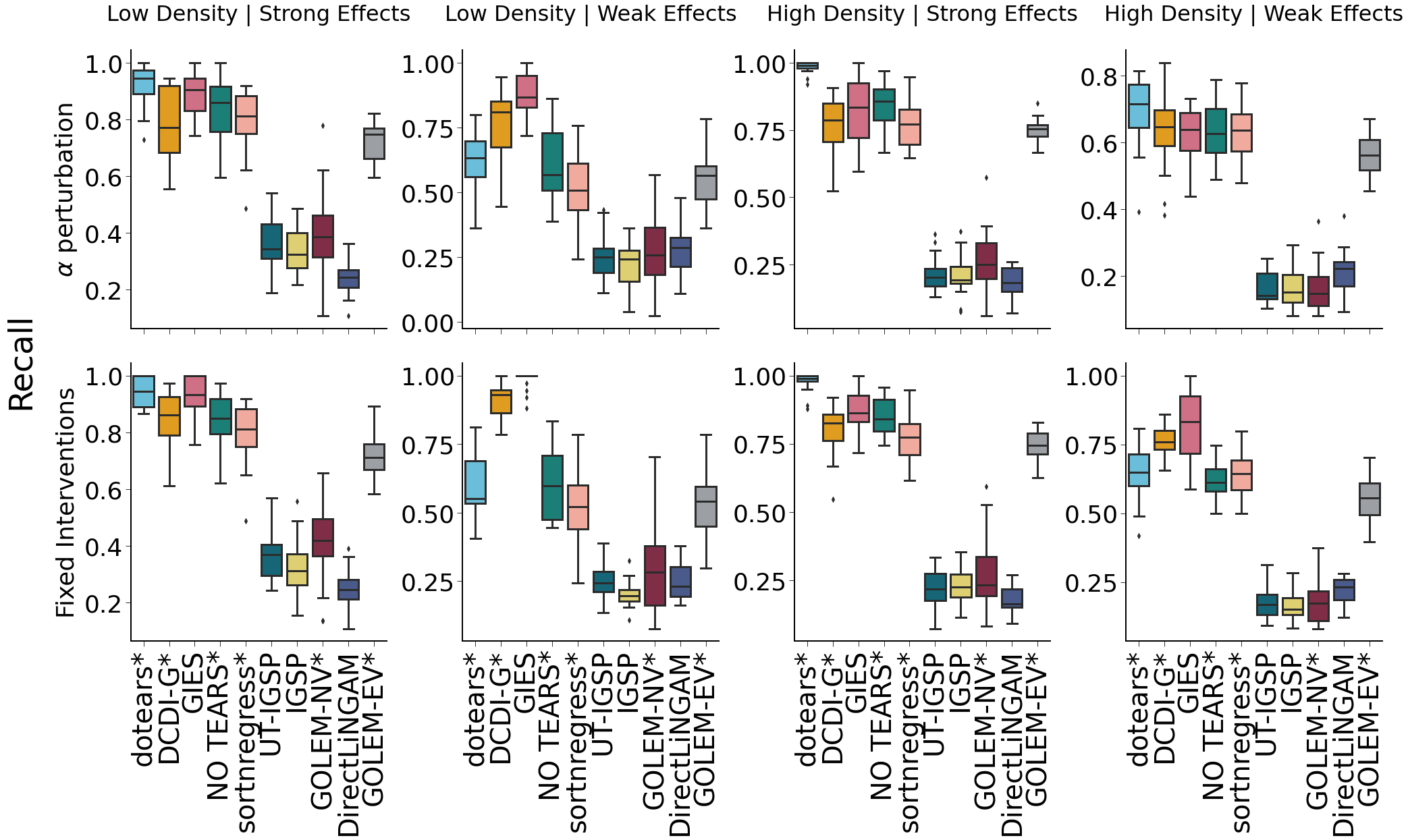}
    \caption{Method performance under linear sensitivity analyses. Each row is a different sensitivity analysis; each column is a parameterization of DAG density and edge weights. Results are shown in recall of binary edge recovery, and methods are sorted by average performance in SHD.}
    \label{suppfig:linear recall}
\end{figure*}

\subsubsection{Fixed Interventions}
Our model assumes that the distribution of node $k$ under intervention on $k$ is dependent on the distribution of $\interv{\epsilon_k}{0}$. We also examine an interventional model in which intervention on a node $k$ sets the distribution with a fixed mean shift and error variance, i.e.
\begin{equation}
    \interv{X_k}{k} \sim \Nor(2, 1),
        \nonumber
\end{equation}
or equivalently $\interv{\epsilon_k}{k}$. This model is used in linear simulations in \cite{brouillard2020differentiable} and is similar to that of \cite{hauser2012characterization}.

\subsection{Sensitivity to nonlinear SEM}
\label{suppsection:nonlinear sensitivity}
To evaluate the performance of \dotears{} on nonlinear data, we generate data from the generative models in \cite{brouillard2020differentiable}. Specifically, structures are \erdosrenyi{} random DAGs with $p=10$ with single-node interventions on every node. The total sample size is $n=10000$ for all draws and scenarios, evenly split across the interventional scenarios (10 interventions and 1 observational) for $n_k \approx 909$.

Data is generated either from the additive noise model (ANM) or the non-linear with non-additive noise model (NN). For a node $j$, let $N_j \sim \Nor(0, \sigma_j^2)$, where $\sigma_j^2 \sim \text{Unif}[1, 2]$. In the additive noise model, for a node $j$ we have 
\begin{equation}
    X_j \coloneqq f_j\left(\Pa(j)\right) + 0.4 \cdot N_j,
\end{equation}
where ``$f_j$ are fully connected neural networks with one hidden layer of 10 units and \textit{leaky ReLU} with a negative slope of 0.25 as nonlinearities'' \cite{brouillard2020differentiable}. Imperfect interventions add a random vector of $\Nor(0, 1)$ to the last layer; for details, see \cite{brouillard2020differentiable}.

In the nonlinear with non-additive noise model, for a node $j$ we have
\begin{equation}
    X_j \coloneqq f_j\left(\Pa(X_j), N_j\right),
\end{equation}
where $f_j$ are ``fully connected neural networks with one hidden layer of 20 units and \textit{tanh} as nonlinearities'' \cite{brouillard2020differentiable}. Imperfect interventions are obtained in the same fashion as the additive noise model; for details, see \cite{brouillard2020differentiable}.

We give two parameterizations for each setting: ``Low Density'', in which the expected number of edges is 10, and ``High Density'', in which the expected number of edges is 40. Results in SHD are given in Figure \ref{fig:nonlinear}. Results in precision and recall are given in Figure \ref{suppfig:nonlinear precision} and Figure \ref{suppfig:nonlinear recall}, respectively. 


\begin{figure*}[h]
    \centering
    \includegraphics[width=0.8\textwidth]{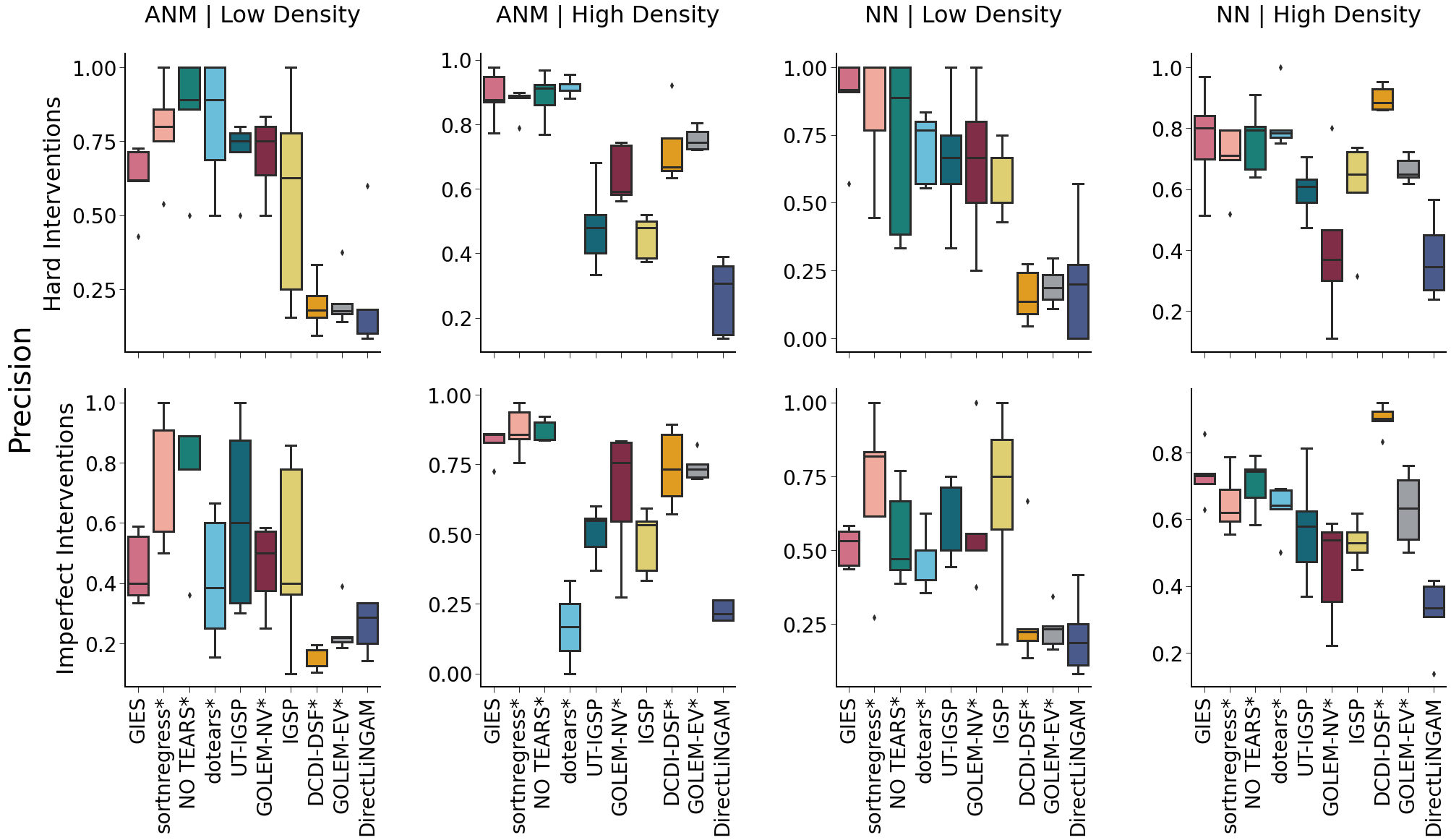}
    \caption{Method performance in nonlinear sensitivity analyses. Each row is a different sensitivity analysis; each column is a parameterization of DAG density and edge weights. Results are shown in precision of binary edge recovery, and methods are sorted by average performance in SHD across parameterizations.}
    \label{suppfig:nonlinear precision}
\end{figure*}

\begin{figure*}[h]
    \centering
    \includegraphics[width=0.8\textwidth]{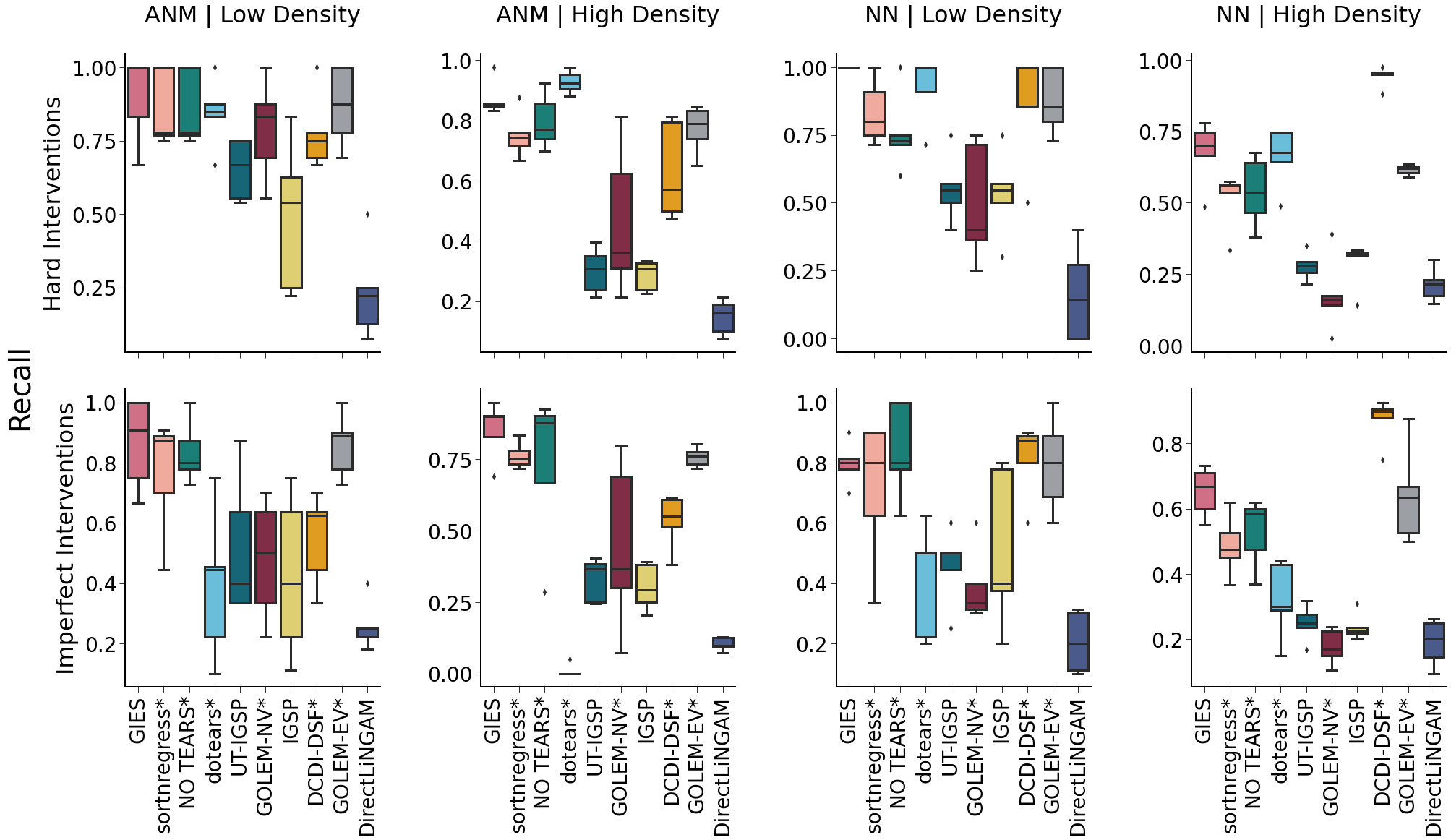}
    \caption{Method performance in nonlinear sensitivity analyses. Each row is a different sensitivity analysis; each column is a parameterization of DAG density and edge weights. Results are shown in recall of binary edge recovery, and methods are sorted by average performance in SHD across parameterizations.}
    \label{suppfig:nonlinear recall}
\end{figure*}

\subsection{Genome-Wide Perturb-Seq}
\label{suppsection:gwps}
We benchmark all methods on a single-cell Perturb-seq experiment from Replogle \etal{} \cite{replogle2022mapping}. Replogle \etal{} provide both raw count data as well as normalized data, where normalized here means a z-scoring relative to the mean and standard deviation of the control dataset, accounting for batch. Here, the control dataset indicates a set of pre-selected control cells, and not simply cells with non-targeting guides. For details, see \cite{replogle2022mapping}.

We perform feature selection in the raw data. We select the top 100 most variable genes in the raw observational data, excluding \textbf{1.)} genes coding for ribosomal proteins and \textbf{2.)} genes without knockdown data. Here, the observational data indicates cells incorporating non-targeting control guides. We then take the normalized expression data for these selected 100 genes in \textbf{1.)} the observational data, and \textbf{2.)} each of the 100 knockdowns. This forms our training set in Figures \ref{fig:gwps}a-d.

In Figure \ref{fig:gwps}e, we perform the same procedure as above, but exclude \textbf{1.)} the observational data. In other words, our training set is formed exclusively from expression data in the 100 knockdowns.

Cross-validation is not performed due to low sample sizes in some knockdowns; instead, the L1 penalty is arbitrarily set to 0.1 for all methods where appropriate. Otherwise, method settings are the same as simulations. 

For DCDI, we report intermediate results, since convergence was not obtained under 24 hours even on GPU training. CPU training was attempted with 30 cores and 180 GB memory, but was killed by the operating system. GPU training was attempted, but was also repeatedly killed by the operating system for memory constraints. DCDI is run as DCDI-G, since DCDI-DSF would not fit in memory, and is also run with imperfect interventions and known interventions. DCDI reported 2982 total edges, but many of these are both causal and anti-causal; in total, 2039 ``interactions'' were reported.

\subsubsection{Differential Expression Testing}
We use DESeq2 to test for differential expression \cite{love2014moderated, ahlmann2020glmgampoi}. We calculate size factors across the raw feature-selected data, and use these in all downstream tests. Next, for each knockdown $\text{ko}(i)$, we test for differential expression for all genes $j$ compared to the observational data. For single-cell data, we run DESeq with parameters test=LRT, fitType=glmGamPoi, useT=TRUE, minmu=1e-6, minReplicatesForReplace=Inf, reduced=~1. We repeat this for each knockdown independently in turn.

At the end of this procedure, we have $100^2$ unadjusted p-values. We perform Benjamini-Hochberg correction to an FDR level of 0.05. Subsequently, in $\text{ko}(i)$, if gene $j$ has an adjusted p-value less than .05, we call that a true edge in downstream calculations of precision and recall.

\subsubsection{Protein-Protein Interactions}
As a second validation edge set, we use protein-protein interactions from the STRING database \cite{szklarczyk2023string}. We pull protein-protein interactions for Homo Sapiens down, and use only physical evidence for interactions. Protein-protein interactions were only used if the STRING confidence score was over 70\%.

\subsubsection{Precision Recall Tables}
\label{suppsection:gwps tables}

In this section, we present precision and recall at three edge threshold levels ($|w| > 0.2$, $|w| < 0.1$, and $|w| < 0.05$) for all methods. We use both differential expression calls and protein-protein interactions as true sets, and separate results for both. Tables \ref{supptab:w0.2}, \ref{supptab:w0.1}, and \ref{supptab:w0.05} denote results for $|w| > 0.2$, $|w| < 0.1$, and $|w| < 0.05$, respectively. For differential expression calls we ignore causal direction, and instead for two given genes denote any differential expression in either direction as a ``true'' edge. This is because all methods struggle equally with inferring causal direction under differential expression, and call an equal number of edges in the ``correct'' causal direction as the ``incorrect'' anticausal direction (see Supplementary Table \ref{supptab:directionality}).
\begin{table}[ht!]
    \centering
    \begin{tabular}{c|ccccc}
        \textbf{Method} & \textbf{\# Called Edges} & \textbf{\% DE} & \textbf{\% PPI} & \textbf{\% DE | PPI} & \textbf{\% Recall} \\
        \hline
        GOLEM-EV & 18 & 44.4 & 77.8 & 83.3 & 0.813 \\
        NO TEARS & 20 & 0.5 & 0.8 & 0.85 & 0.921 \\
        GOLEM-NV & 21 & 42.9 & 76.2 & 81.0 & 0.921 \\
        DirectLiNGAM & 22 & 40.9 & 68.2 & 77.3 & 0.921 \\
        sortnregress & 31 & 35.5 & 64.5 & 67.7 & 1.12 \\
        \dotears{} & 32 & 46.9 & 65.6 & 75.0 & 1.14 \\
        DCDI-G & 2039 & 30.0 & 4.46 & 32.6 & 44.0 \\
        GIES & 3038 & 29.2 & 15.4 & 39.2 & 64.5 \\
        IGSP & 3064 & 29.0 & 15.2 & 38.9 & 64.6 \\
        UT-IGSP & 3075 & 29.8 & 15.1 & 39.6 & 66.0 \\
    \end{tabular}
    \caption{Validation results for all methods at a threshold of $|w| < 0.2$. Columns in order are: Method name (\textbf{Method}), number of edges called at threshold (\textbf{\# Called Edge}), percent of called edges showing differential expression in either the causal or anticausal direction (\textbf{\% DE}), percent of called edges with protein-protein interactions at any confidence level (\textbf{\% PPI}), percent of edges called either differentially expressed or have protein-protein interactions (\textbf{\% DE | PPI}), and total recall, counting both differential expression and protein-protein interactions as true positives (\textbf{\% Recall}).}
    \label{supptab:w0.2}
\end{table}

\begin{table}[ht!]
    \centering
    \begin{tabular}{c|ccccc}
        \textbf{Method} & \textbf{\# Called Edges} & \textbf{\% DE} & \textbf{\% PPI} & \textbf{\% DE | PPI} & \textbf{\% Recall}\\
        \hline
        NO TEARS & 69 & 40.6 & 50.7 & 66.7 & 2.49 \\
        GOLEM-EV & 71 & 39.4 & 50.7 & 66.2 & 2.55 \\
        GOLEM-NV & 72 & 38.9 & 50.0 & 63.9 & 2.49 \\
        \dotears{} & 92 & 37.0 & 45.7 & 60.9 & 3.03 \\
        DirectLiNGAM & 121 & 37.2 & 41.3 & 61.2 & 4.01 \\
        sortnregress & 131 & 36.6 & 41.2 & 58.8 & 4.17\\
        DCDI-G & 2039 & 30.0 & 4.46 & 32.6 & 44.0 \\
        GIES & 3038 & 29.2 & 15.4 & 39.2 & 64.5\\
        IGSP & 3064 & 29.0 & 15.2 & 38.9 & 64.6 \\
        UT-IGSP & 3075 & 29.8 & 15.1 & 39.6 & 66.0 \\

    \end{tabular}
    \caption{Validation results for all methods at a threshold of $|w| < 0.1$. Columns in order are: Method name (\textbf{Method}), number of edges called at threshold (\textbf{\# Called Edge}), percent of called edges showing differential expression in either the causal or anticausal direction (\textbf{\% DE}), percent of called edges with protein-protein interactions at any confidence level (\textbf{\% PPI}), percent of edges called either differentially expressed or have protein-protein interactions (\textbf{\% DE | PPI}), and total recall, counting both differential expression and protein-protein interactions as true positives (\textbf{\% Recall}).}
    \label{supptab:w0.1}
\end{table}

\begin{table}[ht!]
    \centering
    \begin{tabular}{c|ccccc}
        \textbf{Method} & \textbf{\# Called Edges} & \textbf{\% DE} & \textbf{\% PPI} & \textbf{\% DE | PPI} & \textbf{\% Recall}\\
        \hline
        NO TEARS & 149 & 38.3 & 40.9 & 61.7 & 4.98 \\
        GOLEM-EV & 173 & 39.9 & 40.9 & 63.1 & 5.74 \\
        GOLEM-NV & 168 & 42.2 & 39.9 & 63.1 & 5.74 \\
        \dotears{} & 299 & 34.4 & 28.4 & 52.2 & 8.45 \\
        DirectLiNGAM & 613 & 31.6 & 25.0 & 46.5 & 15.4 \\
        sortnregress & 572 & 34.3 & 26.4 & 50.5 & 15.7 \\
        DCDI-G & 2039 & 30.0 & 4.46 & 32.6 & 44.0 \\
        GIES & 3038 & 29.2 & 15.4 & 39.2 & 64.5\\
        IGSP & 3064 & 29.0 & 15.2 & 38.9 & 64.6 \\
        UT-IGSP & 3075 & 29.8 & 15.1 & 39.6 & 66.0 \\
    \end{tabular}
    \caption{Validation results for all methods at a threshold of $|w| < 0.05$. Columns in order are: Method name (\textbf{Method}), number of edges called at threshold (\textbf{\# Called Edge}), percent of called edges showing differential expression in either the causal or anticausal direction (\textbf{\% DE}), percent of called edges with protein-protein interactions at any confidence level (\textbf{\% PPI}), percent of edges called either differentially expressed or have protein-protein interactions (\textbf{\% DE | PPI}), and total recall, counting both differential expression and protein-protein interactions as true positives (\textbf{\% Recall}).}
    \label{supptab:w0.05}
\end{table}

\begin{table}[ht!]
    \centering
    \begin{tabular}{c|cccc}
        \textbf{Method} & \textbf{Threshold} & \textbf{\# Called Edges} & \textbf{\% Causal DE} & \textbf{\% Anticausal DE} \\
        \hline
        \dotears{} & .05 & 299 & 19.7 & 18.7 \\
        \dotears{} & .1 & 92 & 22.8 & 21.7 \\
        \dotears{} & .2 & 32 & 31.3 & 31.3 \\
        NO TEARS & .05 & 149 & 21.5 & 22.1 \\
        NO TEARS & .1 & 69.0 & 26.1 & 23.2 \\
        NO TEARS & .2 & 20.0 & 40.0 & 30.0 \\
        sortnregress & .05 & 572 & 19.6 & 18 \\ 
        sortnregress & 0.1 & 131 &22.9 & 20.6 \\
        sortnregress & 0.2 & 31 & 25.8 & 16.1 \\
        DirectLiNGAM & .05 & 613 & 17.9 & 16.8 \\
        DirectLiNGAM & .1 & 121 & 23.1 & 20.7 \\
        DirectLiNGAM & 0.2 & 22 & 22.7 & 22.3 \\
        GOLEM-EV & .05 & 173 & 20.8 & 23.7 \\
        GOLEM-EV & .1 & 71 & 22.5 & 25.4 \\
        GOLEM-EV & .2 & 18 & 27.8 & 33.3 \\
        GOLEM-NV & .05 & 168 & 19.0 & 30.0 \\
        GOLEM-NV & .1 & 72 & 20.8 & 26.4 \\
        GOLEM-NV & .2 & 21 & 28.6 & 28.6 \\
        GIES &  & 3038 & 15.3 & 15.6 \\
        IGSP &  & 3064 & 14.8 & 16.2 \\
        UT-IGSP &  & 3075 & 16.0 & 15.8 \\
        DCDI-G &  & 2982 & 18.7 & 13.7
    \end{tabular}
    \caption{All methods call an equal proportion of causal and anticausal edges. Columns in order are: Method name (\textbf{Method}), number of edges called at threshold (\textbf{\# Called Edge}), proportion of called edges that show differential expression in the correct causal direction (\textbf{\% Causal DE}), proportion of called edges that show differential expression in the anticausal direction (\textbf{\% Anticausal DE}).}
    \label{supptab:directionality}
\end{table}

\section*{Code Availability}
Code to reproduce the experiments in this paper, as well as a working implementation of \texttt{dotears}, is available at \url{https://github.com/asxue/dotears/}.

\section*{Acknowledgments}
AX was supported by the NIH Training Grant in Genomic Analysis and Interpretation T32HG002536. This work was partially funded by HHMI Hanna H Gray and Sloan fellows programs to HP. This work was partially funded by NIH grant (R35GM125055) and NSF grants (CAREER-1943497, IIS-2106908) to SS. We thank Nathan LaPierre for helpful conversations.

\end{document}